\documentclass[twoside,letterpaper]{article}

\usepackage[preprint]{aistats2027}

\usepackage{amssymb,amsthm,mathtools,fixmath,mleftright,bm}
\let\originalleft\mleft
\let\originalright\mright
\renewcommand{\mleft}{\mathopen{}\mathclose\bgroup\originalleft}
\renewcommand{\mright}{\aftergroup\egroup\originalright}

\usepackage[activate={true,nocompatibility},final, shrink=78]{microtype}
\usepackage[numbers,sort&compress]{natbib}

\usepackage{xurl,siunitx,dirtytalk,subcaption,booktabs,multirow,enumitem,soul}
\setlist[enumerate,itemize]{topsep=3pt,itemsep=2pt,parsep=0pt,partopsep=0pt}
\usepackage[table,dvipsnames]{xcolor}
\usepackage{graphicx}

\usepackage{tikz}
\usetikzlibrary{arrows.meta,backgrounds,calc,positioning}
\usepackage[hidelinks]{hyperref}
\hypersetup{
  pdftitle={The Expressive Power of Random-Feature MPNNs Across Graph Sizes},
  pdfauthor={Robert R. Nerem, Mitchell Black, Antonis Vasileiou, Christopher Morris, Yusu Wang}
}
\usepackage{aliascnt}
\usepackage[capitalise,noabbrev]{cleveref}
\definecolor{BetterBlue}{rgb}{0.38,0.44,1.0}
\definecolor{ForestGreen}{rgb}{0.13,0.55,0.13}

\newtheorem{definition}{Definition}
\newtheorem{theorem}{Theorem}
\newaliascnt{corollary}{theorem}
\newtheorem{corollary}[corollary]{Corollary}
\aliascntresetthe{corollary}
\newaliascnt{proposition}{theorem}
\newtheorem{proposition}[proposition]{Proposition}
\aliascntresetthe{proposition}
\newaliascnt{lemma}{theorem}
\newtheorem{lemma}[lemma]{Lemma}
\aliascntresetthe{lemma}
\newaliascnt{remark}{theorem}

\aliascntresetthe{remark}
\crefname{corollary}{Corollary}{Corollaries}
\Crefname{corollary}{Corollary}{Corollaries}
\crefname{proposition}{Proposition}{Propositions}
\Crefname{proposition}{Proposition}{Propositions}
\crefname{lemma}{Lemma}{Lemmas}
\Crefname{lemma}{Lemma}{Lemmas}
\crefname{remark}{Remark}{Remarks}
\Crefname{remark}{Remark}{Remarks}

\newcommand{\powerset}{\mathcal{P}}
\newcommand{\E}{\mathbb{E}}
\newcommand{\R}{\mathbb{R}}
\newcommand{\Law}{\operatorname{Law}}
\newcommand{\gap}{\operatorname{gap}}
\DeclareMathOperator{\im}{im}
\DeclareMathOperator{\True}{True}
\DeclareMathOperator{\False}{False}
\DeclareMathOperator{\relu}{ReLU}

\renewcommand{\vec}[1]{\bm{#1}}

\begin{document}

\runningtitle{Random-Feature MPNNs Across Graph Sizes}
\runningauthor{Nerem, Black, Vasileiou, Morris, Wang}

\twocolumn[
\aistatstitle{The Expressive Power of Random-Feature MPNNs Across Graph Sizes}
\aistatsauthor{Robert R.~Nerem\textsuperscript{1}\quad
Mitchell Black{\renewcommand{\thefootnote}{\fnsymbol{footnote}}\footnotemark[1]}\quad
Antonis Vasileiou\textsuperscript{2}\quad
Christopher Morris\textsuperscript{2}\quad Yusu Wang\textsuperscript{1}}
\aistatsaddress{\textsuperscript{1}University of California, San Diego\\
\textsuperscript{2}Rheinisch Westf\"alische Technische Hochschule Aachen}
]
\begingroup
\renewcommand{\thefootnote}{\fnsymbol{footnote}}
\footnotetext[1]{Work initiated while at University of California San Diego.}
\endgroup
\thispagestyle{plain}

\begin{abstract}

 Continuous random node features increase the expressivity of message-passing neural networks (MPNNs), but existing theoretical analyses are not fine-grained enough to capture their limitations, particularly in approximating functions across graph sizes or under a fixed number of random bits. We develop a population-level framework for size-uniform approximation. For MPNNs with a graph-size-independent Lipschitz bound in their initial features, including standard mean, max, and order-normalized sum architectures, we show that continuous random features can be quantized to finitely many colors with arbitrarily small uniform error in expected output. This reduction lets us transfer limitations on finite-color distinguishability to limitations on size-uniform approximation with continuous random features. We use this framework to prove that these MPNNs are not size-uniformly universal, analyze the expressivity gains from colors and independent trials, and establish incomparability between random-color MPNNs and higher-order Weisfeiler--Leman methods. We also give finite-color MPNN constructions for tree subgraph isomorphism via color coding and a local MST task.

\end{abstract}

\section{Introduction}

Message-passing graph neural networks (MPNNs)~\citep{Sca+2009,Gil+2017} are widely used to learn functions on graphs, but their expressive power is limited by symmetries that message passing cannot break~\citep{xu2019how,morris2019weisfeiler}. Random node features provide a powerful way to break these symmetries~\citep{dasoulas2020coloring,sato2021random,abboud2021surprising,isaev2026canonical}. Specifically, continuous random features almost surely give every vertex a unique identifier, enabling universality at any fixed graph size~\citep{abboud2021surprising}. However, these constructions require unbounded precision as graphs grow and do not guarantee that a fixed model can approximate a target uniformly across graph sizes. In practice, we want a trained model to handle larger graphs with a fixed architecture, parameters, and precision. This motivates studying size-uniform approximation, as well as how expressive power depends on the number of random bits per vertex.

Representing the same function across graph sizes is a prerequisite for size generalization, which aims to maintain performance on graphs larger than those seen during training~\citep{yehudai2021local,xu2021extrapolate,Ner+2025}. This is essential for learned graph algorithms, particularly in combinatorial optimization, where larger instances may be too expensive to label for training. Neural algorithmic reasoning pursues this goal through inductive biases that favor algorithmic behavior~\citep{velickovic2020neural,velickovic2022clrs}. Such biases, however, cannot yield size generalization if the model class cannot represent a computation across graph sizes. Recent work gives conditions for size-generalization guarantees~\citep{Ner+2025,wittig2026which}. Related work studies uniform expressivity for graph transformers and virtual-node MPNNs~\citep{rosenbluth2024distinguished}, universality for transformers over arbitrarily many context tokens~\citep{furuya2024transformers}, and any-dimensional invariant universality~\citep{Lev+2025,yao2026any}.

We ask which invariant graph functions a random-feature MPNN can approximate uniformly across graph sizes. To obtain deterministic predictions invariant under vertex relabeling, we consider the expected output over random initializations. Let $T\colon\mathcal G\to\mathbb R$ be an invariant target function, write $\overline F(G)\coloneq\E[\widehat F_s(G)]$ for the population predictor, and set $\mathcal G_N\coloneq\{G\in\mathcal G\colon |V(G)|=N\}$. \emph{Fixed-size approximation} allows an $N$-dependent predictor $F_N$ with $\sup_{G\in\mathcal G_N}|\overline F_N(G)-T(G)|\leq\varepsilon$ while \emph{size-uniform approximation} instead requires one predictor $F$, independent of $N$, with $\sup_{G\in\mathcal G}|\overline F(G)-T(G)|\leq\varepsilon$.

Previous work establishes universality for random-feature MPNNs~\citep{abboud2021surprising} and gives specific algorithmic constructions~\citep{sato2021random,kanatsoulis2024counting}, but does not characterize how expressive power depends on the number of random bits per node feature and the number of independent evaluations. We develop a framework that quantifies these dependencies and identifies graph pairs that cannot be reliably distinguished with limited resources. Building on expectation-based analyses of randomized GNNs~\citep{kanatsoulis2024graph,kanatsoulis2024counting,kanatsoulis2025learning}, we study the functions these models represent in expectation and estimate them by averaging independent evaluations. This lets us use the same notions of approximation and distinguishability as for deterministic architectures~\citep{xu2019how,morris2019weisfeiler}, while accounting for their statistical nature. In particular, we develop WL-based combinatorial bounds on distinguishability that translate into limitations on expected approximation.

We study size-uniform approximation for MPNNs with a size-independent Lipschitz constant in the initial random features. This covers mean, max, and order-normalized sum aggregation, as well as unnormalized sums on bounded-degree graphs, and matches the regime of the size-generalization results of \citet{wittig2026which}. It also prevents models from extracting increasingly detailed node identities from vanishing numerical differences.
\definecolor{lqPeriwinkle}{HTML}{6670C9}
\definecolor{lqGold}{HTML}{C58C2B}
\definecolor{lqGoldInk}{HTML}{91621E}
\definecolor{lqInk}{HTML}{243442}
\definecolor{lqMuted}{HTML}{7A8592}
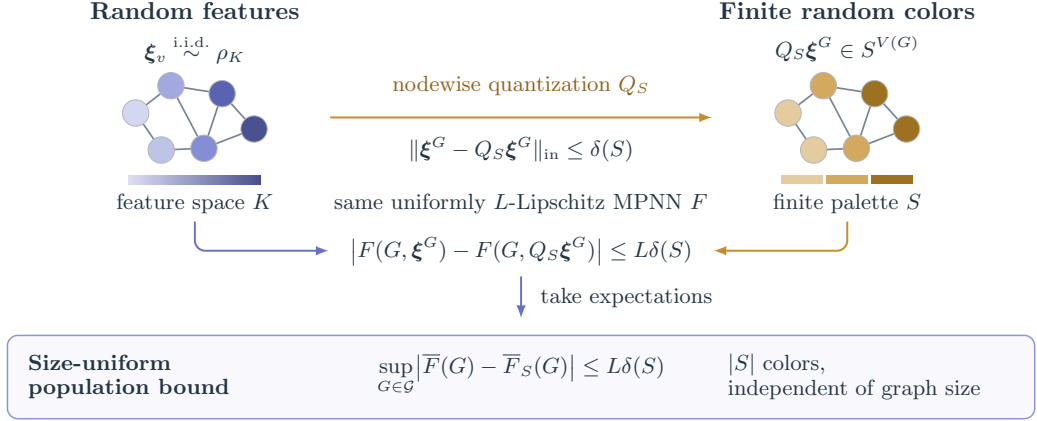
\begin{figure*}[t]
\centering
\resizebox{0.8\linewidth}{!}{
\begin{tikzpicture}[
  x=1cm,y=1cm,font=\small,text=lqInk,>=latex,
  vnode/.style={circle,draw=lqInk!35,line width=0.3pt,
    minimum size=3.8mm,inner sep=0pt},
  vedge/.style={draw=lqMuted,line width=0.7pt},
  process/.style={->,draw=lqPeriwinkle,line width=0.8pt},
  annotation/.style={font=\footnotesize,text=lqInk,
    align=center,inner sep=1pt}
]

\node[font=\bfseries] at (2.70,5.98) {Random features};
\node[font=\bfseries] at (12.10,5.98) {Finite random colors};
\node[font=\footnotesize] at (2.70,5.43)
  {$\vec{\xi}_v\overset{\mathrm{i.i.d.}}{\sim}\rho_K$};
\node[font=\footnotesize] at (12.10,5.43)
  {$Q_S\vec{\xi}^G\in S^{V(G)}$};

\foreach \dx/\prefix in {1.765/a,11.165/b} {
  \begin{scope}[shift={(\dx,3.93)}]
    \coordinate (\prefix1) at (0.082,0.574);
    \coordinate (\prefix2) at (0.590,0.968);
    \coordinate (\prefix3) at (1.328,0.853);
    \coordinate (\prefix4) at (1.788,0.344);
    \coordinate (\prefix5) at (1.066,0.066);
    \coordinate (\prefix6) at (0.451,0.041);
    \foreach \u/\v in {1/2,2/3,3/4,4/5,5/6,6/1,2/5,3/5}
      \draw[vedge] (\prefix\u)--(\prefix\v);
  \end{scope}
}
\node[vnode,fill=lqPeriwinkle!28!white] at (a1) {};
\node[vnode,fill=lqPeriwinkle!60!white] at (a2) {};
\node[vnode,fill=lqPeriwinkle!88!black] at (a3) {};
\node[vnode,fill=lqPeriwinkle!72!black] at (a4) {};
\node[vnode,fill=lqPeriwinkle!80!white] at (a5) {};
\node[vnode,fill=lqPeriwinkle!42!white] at (a6) {};
\node[vnode,fill=lqGold!45!white] at (b1) {};
\node[vnode,fill=lqGold!75!white] at (b2) {};
\node[vnode,fill=lqGold!80!black] at (b3) {};
\node[vnode,fill=lqGold!80!black] at (b4) {};
\node[vnode,fill=lqGold!75!white] at (b5) {};
\node[vnode,fill=lqGold!45!white] at (b6) {};

\shade[left color=lqPeriwinkle!22!white,
       right color=lqPeriwinkle!70!black]
  (1.75,3.48) rectangle (3.65,3.60);
\fill[lqGold!45!white] (11.15,3.48) rectangle (11.75,3.60);
\fill[lqGold!75!white] (11.80,3.48) rectangle (12.40,3.60);
\fill[lqGold!80!black] (12.45,3.48) rectangle (13.05,3.60);
\node[annotation] at (2.70,3.20) {feature space $K$};
\node[annotation] at (12.10,3.20) {finite palette $S$};

\draw[process,draw=lqGold] (4.65,4.44)--(10.15,4.44);
\node[annotation,text=lqGoldInk] at (7.40,4.90)
  {nodewise quantization $Q_S$};
\node[font=\footnotesize,inner sep=0pt] at (7.40,3.98)
  {$\|\vec{\xi}^G-Q_S\vec{\xi}^G\|_{\mathrm{in}}\leq\delta(S)$};

\node[annotation] at (7.40,3.20)
  {same uniformly $L$-Lipschitz MPNN $F$};
\node[inner sep=0pt] (pointwise) at (7.40,2.52)
  {$\displaystyle
      \bigl|F(G,\vec{\xi}^G)-F(G,Q_S\vec{\xi}^G)\bigr|
      \leq L\delta(S)$};
\draw[process,rounded corners=1.5mm,shorten >=3mm]
  (2.70,2.94)--(2.70,2.52)--(pointwise.west);
\draw[process,draw=lqGold,rounded corners=1.5mm,shorten >=3mm]
  (12.10,2.94)--(12.10,2.52)--(pointwise.east);

\draw[process] (7.40,2.15)--(7.40,1.55);
\node[annotation,anchor=west] at (7.64,1.85)
  {take expectations};
\path[rounded corners=1.5mm,draw=lqPeriwinkle!75,
  line width=0.6pt,fill=lqPeriwinkle!4]
  (0,0.10) rectangle (14.80,1.30);
\node[font=\footnotesize\bfseries,align=left] at (1.75,0.71)
  {Size-uniform\\population bound};
\node[inner sep=0pt] at (7.40,0.74)
  {$\displaystyle
      \sup_{G\in\mathcal G}
      \bigl|\overline F(G)-\overline F_S(G)\bigr|
      \leq L\delta(S)$};
\node[annotation,align=left] at (12.20,0.71)
  {$|S|$ colors,\\independent of graph size};
\end{tikzpicture}}
\caption{\textbf{Overview of our random-feature MPNN framework.} Size-uniform quantization with a finite palette $S\subset K$.
Here $(Q_S\vec{x})_v=q_S(\vec{x}_v)$,
$\delta(S)=\sup_{\vec{z}\in K}\|\vec{z}-q_S(\vec{z})\|$, and
\mbox{$\overline F_S(G)=\mathbb E[F(G,Q_S\vec{\xi}^G)]$}.
Both branches use the same MPNN and coupled initializations.}
\label{fig:lipschitz-quantization}
\end{figure*}
\paragraph{Our work}
Our central reduction shows that, for uniformly Lipschitz MPNNs, continuous random features can be replaced by random features taking finitely many values (which we call random colors drawn from a palette), with arbitrarily small uniform error in expected output. Hence, every invariant function approximable size-uniformly by a finite-dimensional continuous-random-feature MPNN is also size-uniformly
approximable by a random-color MPNN. The palette size may depend on the target, accuracy, and model, but not on graph size. Continuous precision is thereby
reduced to two discrete resources, the number of colors $q$ and the number of independent trials $s$; see~\cref{fig:lipschitz-quantization} for an overview.

Using CFI graph families~\citep{cai1992optimal}, we construct an invariant target that no uniformly Lipschitz random-feature MPNN can approximate uniformly across
sizes, even on graphs of bounded diameter. The corresponding $n$-vertex pairs are separated by \emph{(folklore) $2$-dimensional Weisfeiler--Leman} ($2$-WL). A random-color MPNN using $q$ colors requires at least $s$ independent trials per graph to distinguish them with success probability at least $2/3$, where
$q^{10}\log(3s)=\Omega(n)$. Thus, additional colors can be exponentially more effective than additional trials on these instances. In addition, random-color MPNNs and \emph{higher-order $k$-dimensional Weisfeiler--Leman} ($k$-WL) are also incomparable, i.e., for every fixed $k$, a two-color, one-trial MPNN separates
a pair that $k$-WL does not, while for every fixed palette and trial budget, $2$-WL separates a pair that the random-color model does not.

Finite colors nevertheless support useful size-uniform computations, including color coding for fixed-tree subgraph detection~\citep{alon1995color} and a local
MST edge task. Building on~\citet{wittig2026which}, a randomized finite-cover argument then links the represented function class to size-generalization guarantees.

\paragraph{Contributions}
In summary, our main contributions are as follows.
\begin{enumerate}
  \item \emph{Framework for randomized expressivity.} We develop a population-level framework for random-feature MPNNs, allowing their expressive power to be studied through invariant function approximation, graph distinguishability, and comparisons with the Weisfeiler--Leman hierarchy (Section~\ref{sec:formalism-draft}).
  \item \emph{Size-uniform characterization.} We reduce continuous random features to discrete random features (Theorem~\ref{thm:quantization-main} and Corollary~\ref{cor:finite-colors-main}) and prove that random-feature MPNNs are not size-uniformly universal (Corollary~\ref{cor:cfi-nonuniversality-main}).
  \item \emph{Expressivity results.} We establish strict color and trial hierarchies for statistical distinguishability by random-color WL (Theorem~\ref{thm:color-trial-hierarchies-main}), quantify their tradeoff using CFI graphs (Theorem~\ref{thm:cfi-trials-main}), and prove incomparability with higher-order Weisfeiler--Leman (Theorem~\ref{thm:kwl-incomparability-main}).
  \item \emph{Scalable computation and learning.} We give finite-color constructions for structural algorithms (Theorem~\ref{thm:tree-subgraph-main} and Proposition~\ref{prop:local-kruskal-main}), and give sufficient conditions for size-uniform generalization from a suitable finite training set (Theorem~\ref{thm:randomized_generalization}).
\end{enumerate}

\paragraph{Related work}

Loukas \citep{loukas2020what} showed that MPNNs equipped with unique node identifiers, sufficiently expressive message and update functions, and sufficient depth and width can simulate arbitrary distributed computations and, in particular, are Turing complete. Vignac et al.~\citep{vignac2020building} incorporated one-hot node identifiers into an equivariant structural message-passing architecture, making the resulting representation independent of the particular identifier assignment.

Randomization offers another route to individualization. Dasoulas et al. \citep{dasoulas2020coloring} showed that random node colorings can break symmetries and substantially strengthen MPNNs, while Sato et al.~\citep{sato2021random} demonstrated that random features enable MPNNs to simulate randomized local algorithms and improve approximation guarantees for combinatorial optimization problems. Rosenbluth et al.~\citep{Ros+2025} showed that a recurrent GNN with random vertex features can simulate any polynomial-time graph algorithm on connected graphs in polynomial time. GNNs with random node initialization are universal approximators for invariant graph functions on fixed-size graph domains~\citep{abboud2021surprising,gonon2026universality}. Pellizzoni et al. \citep{pellizzoni2024expressivity} subsequently analyzed both the expressivity and statistical sample complexity of node-individualized GNNs, establishing more favorable sample-complexity bounds for permutation-equivariant individualization schemes. \citet{kanatsoulis2024graph,kanatsoulis2024counting} proved that MPNNs with random features could compute graph properties that standard MPNNs could not, showing that random features can overcome the expressivity limitations of anonymous message passing \citep{xu2019how} by acting as probabilistic node identifiers. Their guarantees, however, primarily concern fixed-size expressivity or approximation on bounded graph domains and do not characterize when symmetry breaking yields a scalable algorithmic representation that transfers across graph sizes.

\section{Formal Framework}
\label{sec:formalism-draft}

Here, we outline our framework for analyzing MPNNs with random node features. We study the invariant function defined by a model’s expected output and ask whether a fixed model can approximate a target uniformly over graph sizes. To control the precision required of the random features, we impose a graph-size-independent Lipschitz bound on the output with respect to the initial features. We then connect population approximation to statistical graph distinguishability using finite color palettes and independent trials.

\subsection{Randomized MPNNs}

Let $\mathcal G$ be a class of finite graphs and let $K\subset\R^d$ be a compact node-feature space equipped with a probability distribution $\rho_K$.

A depth-$D$ MPNN initializes $\vec{h}_G^{(0)}(v) \coloneq \vec{x}_v \in \mathbb{R}^d$, $d>0$, and updates its
representations for $\ell \in \{0,\ldots,D-1\}$ according to
\begin{equation*}
\begin{aligned}
\vec{m}_G^{(\ell)}(v)
&\coloneq\operatorname{\mathsf{AGG}}_{\ell}\!\left(
\{\!\{\vec{h}_G^{(\ell)}(u) \colon u\in N_G(v)\}\!\}\right),\\
\vec{h}_G^{(\ell+1)}(v)
&\coloneq\sigma\!\left(
\vec{W}_{\mathrm{self}}^{(\ell)}\vec{h}_G^{(\ell)}(v)
\right.
\left.{}+\vec{W}_{\mathrm{agg}}^{(\ell)}\vec{m}_G^{(\ell)}(v)
+\vec b^{(\ell)}\right),
\end{aligned}
\end{equation*}
where $\sigma$ is coordinate-wise ReLU and $\vec b^{(\ell)}$ is a bias
shared by all vertices. We consider mean, max, min, and order-normalized
sum aggregation; the detailed hypothesis classes are specified in
\cref{subsec:hypothesis_classes}.
A graph-level MPNN applies a permutation-invariant readout function $\operatorname{\mathsf{READOUT}}$ to the final node representations to produce a scalar output $F(G,\vec{x})$, where
$\vec{x}=(\vec{x}_v)_{v\in V(G)}$ denotes the collection of initial node features.

For $G\in\mathcal G$, a random initialization assigns each vertex an
independent random feature drawn from $\rho_K$,
\begin{equation*}
    \vec{\xi}^G=(\vec{\xi}_v)_{v\in V(G)},
    \qquad
    \vec{\xi}_v\overset{\mathrm{i.i.d.}}{\sim}\rho_K.
\end{equation*}
The resulting random-feature MPNN has output $F(G,\vec{\xi}^G)$. When the support of $\rho_K$ comprises finitely many points, we call the corresponding random features \emph{random colors} to emphasize their discrete nature, and we refer to the support as the palette. An MPNN with such an initialization is called a
\textbf{random-color MPNN}. The color distribution need not be uniform. We consider graphs that are not given node attributes as input and are
instead assigned random attributes. It is straightforward to extend our
setting to graphs with input node attributes that are concatenated with the random
initialization.

\subsection{Size-uniform approximation}

A single initialization produces a random output. Taking its expectation
gives a deterministic graph invariant whose approximation power we can
analyze. For a graph-level MPNN $F$ with fixed parameters, we write its
population predictor as
\begin{equation*}
    \overline F(G)\coloneq\E[F(G,\vec{\xi}^G)].
\end{equation*}
Size-uniform approximation requires a single model to achieve the desired
accuracy throughout the graph domain, including on arbitrarily large
graphs when these belong to the domain.

\begin{definition}[Size-uniform approximation]
\label{def:size-uniform-draft}
For a graph class $\mathcal G$ and a graph invariant $T$, a class $\mathcal F$ of random-feature MPNNs
uniformly approximates $T$ if, for every $\varepsilon>0$, there exists a single $F \in \mathcal F$  with
\begin{equation*}
    \sup_{G\in\mathcal G}|\overline F(G)-T(G)|\leq\varepsilon.
\end{equation*}
In particular, if $\mathcal G$ contains graphs of arbitrarily large size, we call this approximation \emph{size-uniform}.
The class $\mathcal F$ is a \emph{size-uniform universal approximator for graph invariants} if $\mathcal G$ contains all finite graphs and it uniformly approximates every graph invariant
$T\colon\mathcal G\to\R$.
\end{definition}

To replace continuous features by finitely many colors, we need small
nodewise perturbations to have a controlled effect on the output even as
graphs grow. We therefore require a Lipschitz bound that is independent
of graph size. Specifically, fix a norm $\|\cdot\|$ on $\R^d$ and use
the max-over-nodes input norm
   $\|\vec{x}\|_{\mathrm{in}}\coloneq\max_{v\in V(G)}\|\vec{x}_v\|.
$
This norm measures the largest perturbation at any vertex.

\begin{definition}[Uniform Lipschitzness]
\label{def:uniform-lipschitzness-draft}
An MPNN $F$ is uniformly $L$-Lipschitz  with respect to $\| \cdot \|_{\mathrm{in}}$ if, for every graph $G$ and all
$\vec{x},\vec{x}'\in K^{V(G)}$,
\begin{equation*}
    |F(G,\vec{x})-F(G,\vec{x}')|
    \leq
    L\|\vec{x}-\vec{x}'\|_{\mathrm{in}}.
\end{equation*}
\end{definition}

The same constant $L$ applies to every graph. Consequently, perturbing
each initial feature by at most $\delta$ changes the output by at most
$L\delta$, independently of the number of vertices. This is the stability
property used in our quantization argument. The following lemma shows
that it holds for several standard MPNN architectures.

\begin{lemma}[Uniform Lipschitzness of standard MPNNs]
\label{lem:standard-mpnn-lipschitzness}
For fixed depth, widths, and parameters, ReLU MPNNs with mean, max, min,
or order-normalized sum aggregation and mean pooling are uniformly
Lipschitz. The same holds for unnormalized sum aggregation on graph
classes of bounded degree.
\end{lemma}
The proof is given in \cref{app:standard-mpnn-lipschitzness}.

\subsection{Colors, trials, and distinguishability}

The population predictor describes the function represented in expectation.
We next ask how reliably finitely many evaluations can estimate this
function and reveal differences between graphs.

We refer to a single evaluation with an independent random initialization as
a \emph{trial}. For an MPNN $F$, draw $s$ independent initializations
$\vec{\xi}_1^G,\ldots,\vec{\xi}_s^G$ and define
\begin{equation}
    \widehat F_s(G)\coloneq\frac1s\sum_{i=1}^s F(G,\vec{\xi}_i^G).
\end{equation}
Then $\E[\widehat F_s(G)]=\overline F(G)$ for every $s$, so the number of trials affects
estimation but not the represented population function.

Graph-pair distinguishability connects this estimation problem to
expressivity: approximating a target that takes different values on two
graphs requires separating their population outputs. We measure this
separation by the population gap
$\gap_F(G,H)\coloneq|\overline F(G)-\overline F(H)|$.
Write $D_K\coloneq\sup_{\vec{u},\vec{u}'\in K}\|\vec{u}-\vec{u}'\|$ for the diameter of
the feature space.

For inputs $A,B\in\{G,H\}$, compute empirical means
$\widehat F_s(A)$ and $\widehat F'_s(B)$ from two independent
batches of $s\in\mathbb N$ trials, even when $A=B$.
For a margin $\Delta>0$, the test of $H_0 \colon A=B$ against
$H_1 \colon A\neq B$ rejects $H_0$ when
    $|\widehat F_s(A)-\widehat F'_s(B)|>\Delta/2$.
The following proposition shows that a positive population gap can be
detected with a number of trials that depends on the gap and the model's
feature sensitivity, but not on graph size.

\begin{proposition}[Population separation yields finite-trial distinguishability]
\label{prop:finite-trial-detection-draft}
Let $F$ be uniformly $L$-Lipschitz, and fix $G,H\in\mathcal G$
with $\gap_F(G,H)\geq\Delta>0$.
For any $\alpha\in(0,1/2)$, the test above has probability of error
at most $\alpha$ for every $A,B\in\{G,H\}$ whenever the number of trials is greater than
\begin{equation*}
    \frac{8L^2D_K^2}{\Delta^2}\log\frac{2}{\alpha}.
\end{equation*}
\end{proposition}
The proof is given in \cref{app:framework-details}.

To prove limitations on finite-trial testing, we use
\emph{random-color Weisfeiler--Leman} (random WL).
Each trial independently colors the vertices and applies standard
Weisfeiler--Leman refinement. An arbitrary test then compares the resulting
patterns across trials. These patterns contain enough information to
determine the output of an MPNN on the same colored graph. Thus, for uniform
$q$-color initializations, limitations of random-color WL also apply to
the corresponding random-color MPNNs, including tests that use more than
their empirical means.

\begin{definition}[Random-color WL distinguishability]
\label{def:rcwl-main}
Let $q,s\in\mathbb N$ and $\alpha\in(0,1/2)$.
A non-isomorphic pair $G,H$ is
\emph{statistically $(q,s,\alpha)$-distinguishable by random-color WL}
if a randomized binary test $\psi$ satisfies
\[
    \max_{A,B\in\{G,H\}}
    \Pr\!\left(
        \psi(\vec Z_A,\vec Z'_B)\neq\mathbf{1}\{A\neq B\}
    \right)\leq\alpha,
\]
where $\vec Z_A,\vec Z'_B$ are independent batches of $s$ $1$-WL patterns from independent uniform $q$-colorings
of $A$ and $B$, respectively.
\end{definition}

Here, $q$ controls the number of colors available within each trial,
while $s$ controls the number of independent observations.
The full construction and statistical reductions are in
\cref{app:framework-details,def:rcwl-draft}.

These notions complete the connection to function approximation.
By \cref{prop:finite-trial-detection-draft}, a lower bound ruling out
successful testing at a sufficiently large trial budget also rules out
the corresponding population gap. Such a bound limits how accurately
the model can approximate a target that assigns different values to the
two graphs. We formalize the resulting obstruction to size-uniform
universality in \cref{cor:pair-nonuniversality-draft} in Appendix~\ref{app:graph-pair-reductions}.

\section{Expressivity Bounds}
\label{sec:expressivity}

We now analyze the expressive power of random-feature MPNNs across graph sizes. We first show that uniform Lipschitzness lets us replace continuous random features with finite random colors, and use this reduction to establish non-universality. We then prove strict color and trial hierarchies for statistical distinguishability and establish incomparability with higher-order Weisfeiler--Leman methods.

\subsection{From continuous features to random colors}

We begin by using uniform Lipschitzness to quantize random features with an error independent of graph size. Let $S\subset K$ be a finite palette and let $q_S\colon K\to S$ be a quantization map taking every point in $K$ to its nearest point in $S$, with
$\delta(S)\coloneq\sup_{\vec{z}\in K}\|\vec{z}-q_S(\vec{z})\|$.
Applying $q_S$ to each node gives $(Q_S\vec{x})_v=q_S(\vec{x}_v)$, and we consider the quantized population
predictor $\overline F_S(G)\coloneq\E[F(G,Q_S\vec{\xi}^G)]$.

\begin{theorem}[Size-uniform quantization]
\label{thm:quantization-main}
Let $F$ be a uniformly $L$-Lipschitz MPNN with random features supported
on compact $K$. For every finite palette $S\subset K$,
\par\makebox[\linewidth]{$\sup_{G\in\mathcal G}
  |\overline F_S(G)-\overline F(G)|
  \leq L\delta(S).$}
\end{theorem}
The result follows as
$\|\vec{x}-Q_S\vec{x}\|_{\mathrm{in}}\leq\delta(S)$ is independent of the number of nodes and
Lipschitzness bounds the output change by $L\delta(S)$.  
Compactness ensures the existence of finite palettes $S$ with arbitrarily small $\delta(S)$. For a concrete example, 
consider uniform features on $[0,1]^d$. An $m$-level grid per
coordinate gives quantization to $m^d$ equally likely colors and quantization error at most
$L\sqrt d/(2m)$. 

\begin{corollary}[Finite colors suffice]
\label{cor:finite-colors-main}
Every graph invariant size-uniformly approximable by uniformly
Lipschitz MPNNs with compactly supported random features is also
size-uniformly approximable by uniformly Lipschitz random-color MPNNs.
\end{corollary}
Intuitively, this corollary follows as a sufficiently fine palette preserves a model's output
up to arbitrarily small error. Uniform Lipschitzness then ensures that this
same palette remains sufficient as graphs grow. The palette may need
to become finer as we demand more accurate approximation, but it need
not grow with graph size.

In addition to palettes generated by quantization, we will also consider palettes that are one-hot encodings of colors. We show that these palettes have equivalent approximation power in \cref{sec:quantization-draft}.

\subsection{Non-universality and trial lower bounds}
We now use quantization to establish limitations on size-uniform
approximation. We construct a pair of non-isomorphic graphs using
Cai--F\"urer--Immerman gadgets~\citep{cai1992optimal} that require an exponential number of trials to distinguish.

\begin{theorem}[Exponential trial requirement]
\label{thm:cfi-trials-main}
There exist non-isomorphic connected graph pairs $(G_k,H_k)$, with
$n_k\to\infty$ vertices and diameter 2, such that
statistical $(q,s,1/3)$-distinguishability by random-color WL requires
\[
    s=\exp\!\left(\Omega\!\left(\frac{n_k}{q^{10}}\right)\right)
\]
for $q\geq2$. 
\end{theorem}
The pair $(G_k,H_k)$ is built from many copies of the
$1$-WL-indistinguishable pair $2C_3$ and $C_6$, connected by CFI gadgets
so that hiding the distinction within any one pair also hides the
distinction between $G_k$ and $H_k$. For a fixed palette, the $2C_3$
and $C_6$ blocks in each pair are both monochromatic with constant
probability, independently across pairs. When this happens, WL cannot distinguish the $2C_3$ and $C_6$ blocks.
The probability that a trial avoids all such events is then
exponentially small in graph size, so exponentially many trials
are required. These graph pairs can be used to construct a graph invariant taking value $1$ on $G_k$ and $0$ on $H_k$ that uniformly Lipschitz random-feature MPNNs cannot
approximate uniformly over graph size. This establishes
their non-universality.
\begin{corollary}[Failure of size-uniform universality]
\label{cor:cfi-nonuniversality-main}
Uniformly Lipschitz MPNNs with i.i.d.\ uniform features on $[0,1]^d$ are
not size-uniform universal approximators for bounded graph invariants, even for input graphs of bounded diameter or bounded degree. 
\end{corollary}

\subsection{Colors and trials}

The previous section gives an example where colors are exponentially more
useful than trials; we now show that increasing either resource strictly
increases statistical distinguishability.

\begin{theorem}[Strict color and trial hierarchies]
\label{thm:color-trial-hierarchies-main}
For statistical distinguishability by random-color WL
as in \cref{def:rcwl-main}, the following hold.
\begin{enumerate}[label=(\roman*)]
    \item \emph{Trials.}
    For every integer $s\geq2$, there are non-isomorphic finite
    graphs $G_s,H_s$ that are statistically
    $(2,s,1/3)$-distinguishable but are not statistically
    $(2,s-1,1/3)$-distinguishable.

    \item \emph{Colors.}
    For every pair of integers $q,s\geq1$, there are
    non-isomorphic finite graphs $G_{q,s},H_{q,s}$ that are
    statistically $(q+1,s,1/3)$-distinguishable but are not
    statistically $(r,s,1/3)$-distinguishable for any
    $1\leq r\leq q$.
\end{enumerate}
\end{theorem}

The trial hierarchy is already strict with two colors. Both statements follow from a common graph construction. That is, a parity gadget combines pairs of cycle blocks so that a trial
reveals the graph identity when every input pair has at least
one nonmonochromatic block. Otherwise, the conditional WL
observation laws agree for the two graphs. This yields an explicit revelation probability and an exact
expression for the optimal testing error. We then choose the graph
parameters so that this error crosses $1/3$ upon adding one
trial or one color. The full construction and proofs are given
in \cref{sec:random-color-hierarchies}.

\subsection{Comparison with Weisfeiler--Leman}

We now show that with fixed trials and palette size, random-color MPNNs
are incomparable with $k$-WL for $k\geq2$.\footnote{We use $k$-WL to
denote folklore WL ($k$-FWL), which is equivalent to oblivious
$(k+1)$-WL for $k\geq2$; see
\citet[Sections~3.1--3.2]{morris2023weisfeiler}.}

\begin{theorem}[Incomparability with higher-order WL]
\label{thm:kwl-incomparability-main}
For every $k\geq2$:
\begin{enumerate}[label=(\roman*)]
    \item There exists a $k$-WL-indistinguishable graph pair that is distinguished with success
    probability greater than $2/3$ by a uniformly Lipschitz MPNN with
    order-normalized sum aggregation, using two colors and one trial.
    \item For every fixed $q,s\geq1$, there exists a graph pair distinguished by $2$-WL,
    and hence by $k$-WL, that is not statistically
    $(q,s,1/3)$-distinguishable by random-color WL.
\end{enumerate}
\end{theorem}

For the first direction, start with an $n$-vertex graph pair that $k$-WL cannot distinguish. Attach a path of the same length to every vertex, preserving $k$-WL indistinguishability. The independent binary colors along each path, read in order from its attachment vertex, form an identifier for that original vertex. If the paths have length $O(\log n)$, then these identifiers are unique with high probability, and thus,  2 colors provide enough information to distinguish
every original vertex. An MPNN can collect these
identifiers and use them to recognize the graph. This example shows how random colors and graph structure can powerfully combine to distinguish vertices.

For the reverse direction, consider the pair $(G_k,H_k)$
from \cref{thm:cfi-trials-main}, which is constructed with CFI gadgets. The $2$-WL distinctions between the $2C_3$ and $C_6$ blocks propagate through the parity tree, revealing which root endpoint carries the marker. The exponential trial lower bound then rules out reliable
distinction with any fixed palette and trial count on sufficiently large
pairs.  The formal incomparability result appears in
\cref{app:kwl-incomparability}.

\section{Algorithmic Applications}
\label{sec:applications}

Finite random colors also support algorithmic computations across graph sizes. We give MPNN constructions for tree subgraph detection via color coding and a local MST edge task, with resource requirements independent of graph size.

\subsection{Tree subgraph detection}

We adapt the color-coding algorithm of \citet{alon1995color} to
random-color MPNNs. Consider a graph $G$ with $n$ vertices and a fixed rooted tree $T$ with $k$ vertices. Independently
assign each vertex in $G$ a uniform color from $[k]$. A subgraph of
$G$ isomorphic to $T$ is called colorful if its vertices have distinct colors; each fixed copy is colorful with probability $k!/k^k$. \citet{alon1995color} give a dynamic program that searches for colorful copies of $T$, using colors to distinguish their vertices. We show that an MPNN can implement this dynamic program with depth, width, and trial count independent of $n$.

\begin{theorem}[Rooted tree detection]
\label{thm:tree-subgraph-main}
Let $T$ be a rooted tree with $k\geq1$ vertices and let $\alpha\in(0,1)$.
There is a random-color MPNN using $k$ colors, $O(k)$ layers,
hidden dimension $\exp(O(k))$, max aggregation, and binary vertex outputs such that,
for every graph $G$ and vertex $v\in V(G)$,
\[
    s=\left\lceil\frac{k^k}{k!}\log\frac1\alpha\right\rceil
\]
independent trials suffice to detect a non-induced copy of $T$ rooted at $v$
with no false positives and false-negative probability at most $\alpha$.
The test at $v$ accepts when its average binary output is positive.
\end{theorem}
Taking the maximum over vertex outputs gives a graph-level detector for $T$ with the same resource bounds and error guarantees (\cref{cor:tree-subgraph-detection}).
The network and trial budgets depend only on $k$ and desired confidence, not on the input graph's size $n$. A subcase is detecting simple paths on $k$ vertices.
The construction and proof appear in \cref{thm:random-color-tree-subiso}. Note that without colors, MPNNs with identical initial vertex features {\bf cannot}
solve this task in general. For example, $C_6$ contains a simple four-vertex
path whereas $2C_3$ does not, but these graphs are indistinguishable
by $1$-WL and hence indistinguishable by any MPNN with identical initial vertex features
(\cref{prop:uncolored-tree-obstruction}).

\subsection{Local MST edge detection}

Next, we consider a task for detecting MST edges that requires only local information.  An edge belongs to every minimum spanning tree (MST) if and only if it is the unique minimum-weight edge across some cut. For a graph with uniform weights, this is exactly the set of bridge edges.  We analyze the set of such edges that unbalanced cuts can witness, so we need only local information to identify them. Specifically, let $G=(V,E)$ be a simple graph
with edge weights $w \colon E\to\Lambda$ in a fixed finite ordered set. Write
$A_k(G,w)$ for the edges that uniquely minimize a cut with fewer than $k$
vertices on one side of the cut.  We give an MPNN that detects such edges. As this is an edge prediction task, we compute the output for each edge
by applying a shared MLP, called the \emph{edge decoder}, to the sum of
its endpoint embeddings together with its edge weight.

\begin{proposition}[Size-independent local MST edge detection]
\label{prop:local-kruskal-main}
Fix an integer $k\geq2$, a finite weight set $\Lambda$, and
$\varepsilon,\alpha\in(0,1)$. There is a max-aggregation random-color MPNN
with an edge decoder such that, for every such graph $G$, its output
$\widehat A\subseteq E$ misses at most an $\varepsilon$
fraction of $A_k(G,w)$ and incorrectly includes at most an $\varepsilon$
fraction of $E\setminus A_k(G,w)$, with probability at least $1-\alpha$.
Its palette size, depth, and number of independent trials depend only on
$k,\Lambda,\varepsilon,\alpha$, not on graph size.
\end{proposition}

Our construction tests whether removing the candidate edge and all
heavier edges leaves its endpoints in different components, with at
least one component containing fewer than $k$ vertices. Each endpoint collects the
colors it can reach, allowing the network to recognize a small
component and compare the two endpoints' components. With sufficiently
many colors and trials, this controls both missed edges and false
detections. The construction and proof are given in
\cref{prop:local-kruskal}. In contrast, we also show that MPNNs with identical initial vertex features {\bf cannot} compute
$A_k$ in general
(\cref{thm:deterministic-kruskal-obstruction}).

\section{Learnability of Randomized GNNs}
\label{sec:Learnability}
We study the learnability of randomized MPNNs trained using averages
over finitely many random node initializations. We fix
the depth $D$, hidden dimensions, aggregation scheme, and initialization
law $\rho_K$ on a compact feature space $K$, subject to the weight and bias
bounds in \cref{subsec:hypothesis_classes}. We consider scalar
graph-level outputs with mean pooling and a linear readout.

The graph pseudometric $d_D^{\mathrm{rand}}$ compares the distributions
of depth-$D$ neighborhood representations under random node
initialization and uniform vertex sampling. It is the Wasserstein
distance between these distributions, using a recursive metric
reflecting the aggregation scheme; see
\cref{app:randomized_representations}. Distinct graphs may have
distance zero if their representation distributions coincide.
The graph domain $\mathcal G$ is totally bounded under this
pseudometric, and the mean predictor
$\overline F_\theta(G)\coloneq\mathbb E[F_\theta(G,\vec{\xi}^G)]$,
where $\vec{\xi}^G\sim\rho_K^{\otimes V(G)}$, satisfies
\[
|\overline F_\theta(G)-\overline F_\theta(G')|
\le B_\theta d_D^{\mathrm{rand}}(G,G').
\]
Here, $B_\theta$ is an explicit Lipschitz certificate computed
from the network parameters; see \cref{app:subsec_lipschitzness}.
These properties allow accuracy on a suitable finite training
set to control approximation error throughout $\mathcal G$. Let $f^\ast \colon \mathcal G\to\mathbb R$ be $B_\ast$-Lipschitz under
$d_D^{\mathrm{rand}}$. For training graphs $G_1,\ldots,G_n$
and $S$ independent initializations
$\vec{\xi}_{i,s}^{G_i}\sim\rho_K^{\otimes V(G_i)}$ of each graph, define
\[
\begin{gathered}
\widehat F_{\theta,S}(G_i)
\coloneq
\frac1S\sum_{s=1}^S F_\theta(G_i,\vec{\xi}_{i,s}^{G_i}),
\\
\begin{multlined}
\widehat{\mathcal L}_{\mathrm{reg}}(\theta)
\coloneq
\frac1n\sum_{i=1}^n
|\widehat F_{\theta,S}(G_i)-f^\ast(G_i)|
+\lambda(B_\theta-B_\ast)_+,
\end{multlined}
\end{gathered}
\]
where $\lambda>0$ and $(a)_+=\max\{a,0\}$.

\begin{theorem}[Uniform generalization from Monte Carlo training]
\label{thm:randomized_generalization}
For every $\varepsilon>0$ and $\delta\in(0,1)$, there exist a finite training set $\{G_1,\ldots,G_n\}\subseteq\mathcal G$ and a tolerance $\eta>0$ such that the following holds.
For any sufficiently large $S$, draw $S$ independent initializations of each training graph from
$\rho_K^{\otimes V(G_i)}$. Then, with probability at least $1-\delta$,
\[
\widehat{\mathcal L}_{\mathrm{reg}}(\theta)<\eta
\quad\Longrightarrow\quad
\sup_{G\in\mathcal G}
|\overline F_\theta(G)-f^\ast(G)|<\varepsilon
\]
simultaneously for all $\theta\in\Theta$.
\end{theorem}

Thus, sufficiently small sampled regularized error guarantees uniform approximation by the mean predictor, regardless of graph order. The training-set size is controlled by the graph domain's covering number at a radius determined by the desired accuracy and the target's Lipschitz constant. Uniform Monte Carlo concentration makes the guarantee hold even when parameters are selected using the same sampled initializations. For any fixed test graph, Monte Carlo concentration controls the error from averaging $S'$ fresh independent initializations, without retraining. The proof and an explicit sufficient sample size are given in \cref{app:learning_proofs,app:monte_carlo}.

\section{Experiments}
\label{sec:experimental-study}

We perform a range of experiments to explore the role of random colors in expressivity and size generalization. 
\Cref{app:experimental-details} gives additional details.

\paragraph{Bridge prediction across graph sizes}
\label{sec:er-bridges}
We train random-color MPNNs to identify bridge edges in unweighted
Erd\H{o}s--R\'enyi graphs with 16 or 32 vertices and test on 512-vertex
graphs. Our local MST edge-detection
result (\cref{prop:local-kruskal-main}) motivates examining the size $b$
of the smaller endpoint component after deleting a bridge.
We train separate models to detect bridges with exactly $b\in\{4,5,6\}$
vertices on the smaller side, treating all other edges as negatives.
A standard MPNN with constant initialization and max aggregation assigns the same score
to every edge, yielding exactly $50\%$ balanced accuracy for all three
target sizes.
With 16 trials and a fixed training budget, increasing the number of
one-hot colors increases mean balanced test accuracy (\cref{fig:er-bridge-generalization}).
In particular, larger palettes improve detection, while bridges with larger smaller-side components remain harder to identify. Furthermore, detection generalizes to larger graphs using a fixed finite palette. 

\begin{figure}[t]
\centering
\includegraphics[width=.8\linewidth]{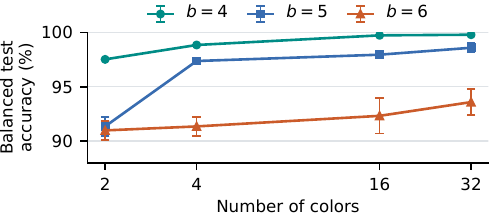}
\caption{Balanced test accuracy $1-\tfrac12(\operatorname{FNR}+\operatorname{FPR})$
for detecting bridges of smaller-side size $b$.
Error bars show SE over five training seeds.}
\label{fig:er-bridge-generalization}
\end{figure}

\paragraph{Tree substructure detection: colors and trials}
\label{sec:tree-detection}
Motivated by our color-coding construction
(\cref{thm:tree-subgraph-main}), we train models to detect the full
balanced degree-three tree of depth two---a root with three children,
each with two leaves---as a non-induced subgraph rooted at each vertex.
Training and test graphs are random 3-regular graphs on 24 vertices and
we vary the palette and trial count. On these graphs, the target is
equivalent to the vertex's two-hop neighborhood having the
maximum possible size of 10. We train using the maximum output across trials. At test time, we sum predictions across trials and apply a threshold chosen on validation graphs.

Increasing either the palette size or the number of trials increases mean test accuracy (\cref{fig:tree-detection-heatmap}). These results support our tree-detection construction (\cref{thm:tree-subgraph-main}): finite colors enable detection, whereas constant initial features give identical outputs at every vertex of these 3-regular graphs. The gains from both resources align with our color and trial hierarchies (\cref{thm:color-trial-hierarchies-main}).

\begin{figure}[t]
\centering
\includegraphics[width=0.8\linewidth]{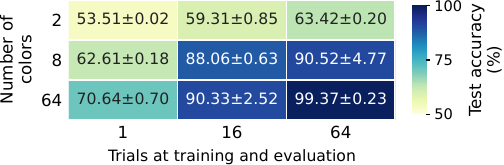}
\caption{Rooted tree detection with different numbers of colors and
trials. Cells show mean node test accuracy (\%) $\pm$ SE across five
training seeds.}
\label{fig:tree-detection-heatmap}
\end{figure}

\paragraph{Inference-time quantization}
\label{sec:exp-inference-quantization}
Motivated by our quantization result (\cref{thm:quantization-main}), we examine whether quantizing random features preserves predictive performance. We replicate GCN-RNI on EXP~\cite{abboud2021surprising} and rGIN on TRIANGLE~\cite{sato2021random}, then quantize their random features to two values per coordinate without retraining. EXP classifies graph-encoded Boolean formulas; we evaluate the 64-coordinate baseline and a scalar variant. On TRIANGLE, we predict triangle membership using 20-vertex 3-regular training graphs and 20- and 100-vertex test graphs. Both models retain most of their predictive performance (\cref{tab:exp-inference-quantization}), supporting our message that finite palettes can preserve the expressive benefits of continuous random features, including under size generalization.

\begin{table}[t]
\centering
\small
\setlength{\tabcolsep}{4pt}
\begin{tabular*}{\linewidth}{@{\extracolsep{\fill}}lrr@{}}
\toprule
\textbf{Setting} & \textbf{Continuous} & \textbf{2 colors/coordinate} \\
\midrule
\multicolumn{3}{@{}l}{\emph{EXP: graph accuracy (\%)}} \\
1 coordinate & $98.68\pm 0.17$ & $98.62\pm 0.19$ \\
64 coordinates & $97.90\pm 0.39$ & $97.83\pm 0.48$ \\
\midrule
\multicolumn{3}{@{}l}{\emph{TRIANGLE: node ROC-AUC (\%), $n=100$}} \\
1 coordinate & $93.98\pm 0.05$ & $93.35\pm 0.05$ \\
\bottomrule
\end{tabular*}
\caption{Inference-time performance under quantization. 
Mean $\pm$ SE across ten folds for EXP, and ten single-trial
feature draws of one trained model for TRIANGLE.}
\label{tab:exp-inference-quantization}
\vspace{-10pt}
\end{table}

\section{Conclusions}

We analyze random-feature MPNNs through the vertex distinguishability provided by finite random colors. This offers a natural intermediate between anonymous message passing and the use
of continuous random features as unique vertex identifiers. For uniformly Lipschitz MPNNs, finite colors retain the size-uniform
approximation power of continuous random features, while making
the capabilities and limitations of random-feature MPNNs explicit. Our framework also makes the roles of palette size and trial count clear, laying a theoretical foundation for understanding how to use random features in practice.

An important next step is to identify further algorithms that can be
implemented using graph-size-independent numbers of colors and trials.  Although our examples establish
limitations and incomparability with $k$-WL, a broader
characterization of which graph pairs these models can distinguish
remains open. Our statistical testing framework permits arbitrary
combinations of trials, but our approximation and learning analysis
centers on averaging. Studying other ways of combining trial information, beyond averaging and expectations, is another interesting direction, as aggregation across trials is an important architectural choice.

\section*{Acknowledgements}
We thank Sanjoy Dasgupta for insightful discussions during the early stages
of this work.

YW, RRN, and MB are supported by NSF grant CCF-2217058, and YW and RRN
are additionally supported by NSF grants CCF-2112665 and CCF-2310411.
CM is partially funded by a DFG Emmy Noether grant
(468502433), Israel ISF-DFG Grant (587125021), and a RWTH Junior
Principal Investigator Fellowship under Germany's Excellence Strategy.
AV is supported by the German Research Foundation (DFG) within
Research Training Group 2236/2 (UnRAVeL).

\section*{AI use statement}
We used generative AI tools to help refine and check our mathematical arguments, explore mathematical approaches, design and implement experiments, analyze results, and prepare figures. We also used these tools for literature searches and to draft and revise the manuscript. All material developed or modified with AI assistance was explicitly reviewed and verified by the authors.

\bibliographystyle{plainnat}
\bibliography{bib/main}

\clearpage
\appendix
\onecolumn

\section{Additional Framework Details}
\label{app:framework-details}

\subsection{Graphs with random node features and graph statistics}

The graph with random node features is
\begin{equation}
    \vec G^{\vec{\xi}}
    \coloneq
    \bigl(G,(\vec{\xi}_v)_{v\in V(G)}\bigr),
    \qquad
    P_G\coloneq\Law(\vec G^{\vec{\xi}}).
\end{equation}

Let $\mathcal G_K$ denote the resulting space of graphs with node features. A
\emph{graph invariant} is a function on graphs that is permutation invariant (unchanged by vertex
relabeling), and a \emph{graph statistic} is a scalar-valued graph invariant
whose input is a graph with (random) node features. For a graph statistic
$t$,\footnote{We assume that $t(\vec G^{\vec{\xi}})$ is integrable for every
$G\in\mathcal G$.} we refer to a single evaluation with
initialization $\vec{\xi}_v\overset{\mathrm{i.i.d.}}{\sim}\rho_K$, as a \emph{trial}, where the output and its law are
\begin{equation}
    Z_t(G)\coloneq t(\vec G^{\vec{\xi}}),
    \qquad
    P_G^t\coloneq\Law(Z_t(G)).
\end{equation}
\subsection{Message-passing neural networks}
\label{subsec:hypothesis_classes}

Let $\mathcal{G}$ be a class of graphs with node features and let $V_1(\mathcal{G})\coloneq\{(G,v) \colon G\in\mathcal{G},\,v\in V(G)\}$.
A depth-$D$ MPNN initializes $\vec{h}_G^{(0)}(v)=\vec{x}_v$, where $\vec{x}_v$ is the initial
feature of node $v$, and updates its representations for
$\ell=0,\ldots,D-1$ according to
\begin{align*}
\vec{m}_G^{(\ell)}(v)
&=
\operatorname{\mathsf{AGG}}_{\ell}\!\left(
\{\!\{\vec{h}_G^{(\ell)}(u) \colon u\in N_G(v)\}\!\}
\right),\\
\vec{h}_G^{(\ell+1)}(v)
&=
\operatorname{UPD}_{\ell}\!\left(
\vec{h}_G^{(\ell)}(v),\vec{m}_G^{(\ell)}(v)
\right).
\label{eq:MPNNS_gilmer}
\end{align*}

In this work, we specialize the update function to
\[
\vec{h}_G^{(\ell+1)}(v)
=
\sigma\!\left(
\vec{W}_{\mathrm{self}}^{(\ell)}\vec{h}_G^{(\ell)}(v)
+
\vec{W}_{\mathrm{agg}}^{(\ell)}\vec{m}_G^{(\ell)}(v)
+\vec b^{(\ell)}
\right),
\]
where $\sigma$ is the coordinate-wise ReLU activation function and
$\vec b^{(\ell)}\in\R^{d_{\ell+1}}$ is a bias shared by all vertices. We consider
the following aggregation schemes:
\[
\begin{aligned}
\vec{m}_{G,\overline{\mathrm{sum}}}^{(\ell)}(v)
&=\frac{1}{|V(G)|}\sum_{u\in N_G(v)}\vec{h}_G^{(\ell)}(u),\\
\vec{m}_{G,\mathrm{mean}}^{(\ell)}(v)
&=\frac{1}{|N_G(v)|}\sum_{u\in N_G(v)}\vec{h}_G^{(\ell)}(u),
\end{aligned}
\]
and 
\[
\begin{aligned}
\vec{m}_{G,\mathrm{max}}^{(\ell)}(v)&=
\max_{u\in N_G(v)}\vec{h}_G^{(\ell)}(u),\\
\vec{m}_{G,\mathrm{min}}^{(\ell)}(v)&=
\min_{u\in N_G(v)}\vec{h}_G^{(\ell)}(u),
\end{aligned}
\]
where the maximum and minimum are taken coordinate-wise.

Let $\theta\in\Theta$ collect the update and readout parameters, and
let $\mathcal{H}_{\overline{\mathrm{sum}}}^{(D)}$,
$\mathcal{H}_{\mathrm{mean}}^{(D)}$,
$\mathcal{H}_{\mathrm{max}}^{(D)}$, and $\mathcal{H}_{\mathrm{min}}^{(D)}$ denote the corresponding classes of maps
$\vec{h}_\theta^{(D)}\colon V_1(\mathcal{G})\to\mathbb{R}^{d_D}$.
The class $\mathcal{H}_{\mathrm{min}}^{(D)}$ is defined analogously. We assume
throughout that
\[
\left\|\vec{W}_{\mathrm{self}}^{(\ell)}\right\|_2,\,
\left\|\vec{W}_{\mathrm{agg}}^{(\ell)}\right\|_2\le B_W,
\qquad
\|\vec b^{(\ell)}\|_2\le B_b.
\]
The constants $B_W,B_b<\infty$ are shared across all
$\theta\in\Theta$ and layers and are independent of graph size.
When a linear readout $\vec{W}_{\mathrm{out}}$ is used, we also assume
$\|\vec{W}_{\mathrm{out}}\|_2\le B_{\mathrm{out}}$.

The graph-level output uses a permutation-invariant readout chosen as part
of the architecture:
\[
F_\theta(G,\vec{x})
=
\operatorname{\mathsf{READOUT}}_\theta\!\left(
\{\!\{\vec{h}_G^{(D)}(v) \colon v\in V(G)\}\!\}\right).
\]
We write $F(G,\vec{x})$ for the network output, where
$\vec{x}=(\vec{x}_v)_{v\in V(G)}$ denotes the initial node features.

\subsection{Population approximation}

For a graph statistic $t$, its \emph{$t$-mean} is the graph invariant defined by
\begin{equation}
    \mu_t(G)
    \coloneq
    \E_{\vec G^{\vec{\xi}}\sim P_G}[t(\vec G^{\vec{\xi}})].
\end{equation}
\subsection{Uniform Lipschitzness}
\label{app:standard-mpnn-lipschitzness}

\begin{proof}[Proof of \cref{lem:standard-mpnn-lipschitzness}]
Take the coordinate maximum norm on each hidden space and the induced
matrix norm $\|\cdot\|_{\infty\to\infty}$. For two initializations $\vec{x},\vec{x}'$
of the same nonempty graph, write
\begin{equation}
    e_\ell\coloneq\max_{v\in V(G)}
    \|\vec{h}_G^{(\ell)}(v;\vec{x})-\vec{h}_G^{(\ell)}(v;\vec{x}')\|_\infty.
\end{equation}
Mean and order-normalized sum aggregation are nonexpansive in this
max-over-nodes norm. Coordinatewise max and min are also nonexpansive,
since each coordinate changes by at most $e_\ell$. For unnormalized sums
on graphs of maximum degree at most $\Delta$, the corresponding bound is
$\Delta e_\ell$. Define empty neighborhood aggregates and the readout of
the empty graph to be zero.

Let $A_\ell=1$ for the normalized, max, and min aggregations, and
$A_\ell=\Delta$ for unnormalized sums. Since ReLU is nonexpansive and
the biases cancel in differences, each update satisfies
\begin{equation}
    e_{\ell+1}\leq
    \left(\|\vec{W}_{\mathrm{self}}^{(\ell)}\|_{\infty\to\infty}
    +A_\ell\|\vec{W}_{\mathrm{agg}}^{(\ell)}\|_{\infty\to\infty}\right)e_\ell.
\end{equation}
Mean pooling is nonexpansive. Composing these bounds with the final
linear map gives a feature-Lipschitz constant
\begin{equation}
    L_\infty=
    \|\vec{W}_{\mathrm{out}}\|_{\infty\to\infty}
    \prod_{\ell=0}^{D-1}
    \left(\|\vec{W}_{\mathrm{self}}^{(\ell)}\|_{\infty\to\infty}
    +A_\ell\|\vec{W}_{\mathrm{agg}}^{(\ell)}\|_{\infty\to\infty}\right).
\end{equation}
This constant is independent of graph size. For any other fixed norm
$\|\cdot\|$ on the input space, choose $c$ such that
$\|\vec{z}\|_\infty\leq c\|\vec{z}\|$. Then $L=cL_\infty$ gives uniform
Lipschitzness with respect to $\|\cdot\|_{\mathrm{in}}$.
\end{proof}

Write $D_K\coloneq\sup_{\vec{u},\vec{u}'\in K}\|\vec{u}-\vec{u}'\|$ for the diameter of the
feature space. The input norm and uniform Lipschitzness are defined in
\cref{def:uniform-lipschitzness-draft}.

Uniform Lipschitzness immediately
bounds the output diameter, independently of graph size.
\begin{lemma}[Uniform output diameter]
\label{lem:uniform-output-diameter-draft}
Let $F$ be a uniformly $L$-Lipschitz MPNN. Then, for every graph $G$,
$Z_F(G)$ is supported on an interval of length at most $LD_K$.
\end{lemma}

\begin{proof}
For any $\vec{x},\vec{x}'\in K^{V(G)}$, every nodewise difference is at most $D_K$,
so $\|\vec{x}-\vec{x}'\|_{\mathrm{in}}\leq D_K$. Uniform $L$-Lipschitzness then
gives $|F(G,\vec{x})-F(G,\vec{x}')|\leq LD_K$. Thus the image of
$\vec{x}\mapsto F(G,\vec{x})$ has diameter at most $LD_K$, which proves the claim.
\end{proof}

\subsection{Empirical averaging}
In addition to analyzing the expressivity of the mean of random-feature MPNNs, we also consider the complexity of finite-sample estimation, which is required in practice. To estimate $\mu_t(G)$, draw $s$ independently initialized trials
$X_1,\ldots,X_s\sim P_G$ and define
\begin{equation}
    \widehat\mu_{t,s}(G)
    \coloneq
    \frac1s\sum_{i=1}^s t(X_i).
\end{equation}
For an MPNN $F$, write $\widehat F_s=\widehat\mu_{F,s}$. Then
$\E[\widehat F_s(G)]=\overline F(G)$ for every $s$, so the number of trials affects
estimation but not the represented population function.

\begin{definition}[Size-uniform high-probability approximation]
\label{def:high-probability-approximation-draft}
For a graph invariant $T$, $\varepsilon\geq 0$, $\alpha\in(0,1)$, and
$s\in\mathbb N$, an MPNN $F$ size-uniformly approximates $T$ with error
$\varepsilon$ and failure probability at most $\alpha$ using $s$ trials if 
\begin{equation}
    \inf_{G\in\mathcal G}
    \Pr\left(
        |\widehat F_s(G)-T(G)|\leq\varepsilon
    \right)
    \geq 1-\alpha.
\end{equation}
\end{definition}

For the $L$-Lipschitz models we consider, population-level approximation (as in \Cref{def:size-uniform-draft}) immediately gives a finite-trial high-probability approximation. The key is that the required number of samples is constant for any fixed model, and in particular independent of graph size.
\begin{proposition}[Automatic finite-trial bound]
\label{prop:population-to-trials-draft}
Let $T$ be a graph invariant. For $\varepsilon>0$, suppose that $F$ is uniformly $L$-Lipschitz and
$\sup_{G\in\mathcal G}|\overline F(G)-T(G)|\leq\varepsilon/2$. Then, for any $\alpha\in(0,1)$, the invariant $F$ size-uniformly
approximates the invariant $T$ with error $\varepsilon$ and failure probability at most
$\alpha$ using any number of trials satisfying
\begin{equation}
    s\geq
    \frac{2L^2D_K^2}{\varepsilon^2}\log\frac{2}{\alpha}.
\end{equation}
\end{proposition}

\begin{proof}
By \cref{lem:uniform-output-diameter-draft}, each of the $s$ independent
trials averaged by $\widehat F_s(G)$ lies in an interval of length at most
$LD_K$. Hoeffding's inequality therefore yields
\begin{equation}
\begin{aligned}
    \Pr\left(
        |\widehat F_s(G)-\overline F(G)|>\frac{\varepsilon}{2}
    \right)&\\
    &\leq
    2\exp\left(-\frac{s\varepsilon^2}{2L^2D_K^2}\right)
    \leq \alpha.
\end{aligned}
\end{equation}
Thus, with probability at least $1-\alpha$,
$|\widehat F_s(G)-\overline F(G)|\leq\varepsilon/2$. In this case, the
triangle inequality and the population approximation hypothesis give
\begin{equation}
\begin{aligned}
    |\widehat F_s(G)-T(G)|
    &\leq
    |\widehat F_s(G)-\overline F(G)|
    +|\overline F(G)-T(G)|\\
    &\leq \frac{\varepsilon}{2}+\frac{\varepsilon}{2}
    =\varepsilon.
\end{aligned}
\end{equation}
This bound is independent of $G$ so taking the
infimum over $\mathcal G$ completes the proof. 
\end{proof}

\subsection{Graph-pair reductions and statistical distinguishability}
\label{app:graph-pair-reductions}

We also consider whether random-feature MPNNs distinguish graph pairs at
the population level. 
For an MPNN $F$, define its population gap on $G,H\in\mathcal G$ by
\begin{equation}
    \gap_F(G,H)
    \coloneq
    |\overline F(G)-\overline F(H)|.
\end{equation}
We say that $F$ mean-distinguishes the pair with margin $\Delta$ if
$\gap_F(G,H)\geq\Delta$.

We next connect the population gap to empirical isomorphism testing. For a graph statistic $t$, an $s$-trial test for $G,H$
receives two independent batches
\begin{equation}
\begin{aligned}
    \vec Z_A&=(Z_{t,1}(A),\ldots,Z_{t,s}(A)),\\
    \vec Z_B&=(Z_{t,1}(B),\ldots,Z_{t,s}(B)),
\end{aligned}
\end{equation}
where $A,B\in\{G,H\}$ and the entries of each batch are i.i.d.\ with laws
$P_A^t$ and $P_B^t$, respectively. The test returns one when it declares
that the batches came from different graphs.

\begin{definition}[Statistical distinguishability]
\label{def:statistical-distinguishability-draft}
For $\alpha\in(0,1/2)$, the pair $G,H$ is statistically
$(t,s,\alpha)$-distinguishable if there is a test $\psi$ such that
\begin{equation}
    \max_{A,B\in\{G,H\}}
    \Pr\left(
        \psi(\vec Z_A,\vec Z_B)
        \neq \vec 1\{A\neq B\}
    \right)
    \leq\alpha.
    \end{equation}
Thus the test must distinguish same from different in all four choices of
$(A,B)$.
\end{definition}

For observed batches $\vec z,\vec z'\in\R^s$, consider the
empirical-mean test
\begin{equation}
    \varphi_{t,\Delta,s}(\vec z,\vec z')
    \coloneq
    \vec 1\left\{
        \left|
            \frac1s\sum_{i=1}^s z_i
            -\frac1s\sum_{i=1}^s z_i'
        \right|>\frac{\Delta}{2}
    \right\}.
\end{equation}

\begin{proof}[Proof of \cref{prop:finite-trial-detection-draft}]
By \cref{lem:uniform-output-diameter-draft}, the outputs on each fixed graph
lie in an interval of length at most $LD_K$. Pairing the two batches, the
empirical difference is therefore an average of independent variables with
range length at most $2LD_K$. Its mean is zero
when $A=B$ and has absolute value at least $\Delta$ when $A\neq B$.
An error requires a deviation of at least $\Delta/2$, so Hoeffding's
inequality bounds every case by
\begin{equation}
    2\exp\left(-\frac{s\Delta^2}{8L^2D_K^2}\right).
\end{equation}
\end{proof}

Graph isomorphism testing will be particularly useful for bounding the power of random features. 
Specifically, by contraposition, a lower bound against all $s$-trial tests  implies $\gap_F(G,H)<\Delta$.
Finally, we can convert
bounded population gap to non-universality, in the function approximation framework. 

\begin{corollary}[Non-universality from graph pairs]
\label{cor:pair-nonuniversality-draft}
Let $\mathcal F$ be a class of random-feature MPNNs. Suppose that, for some
$\beta<1$, for every $F\in\mathcal F$ there exist non-isomorphic graphs
$G,H\in\mathcal G$ with
\begin{equation}
    \gap_F(G,H)\leq\beta.
\end{equation}
Then $\mathcal F$ is not size-uniformly universal.
\end{corollary}

\begin{proof}
First, enumerate the isomorphism classes in $\mathcal G$ and consider the graph invariant
$T_\star\colon\mathcal G\to\mathbb N$ that assigns each class its index. If there is some
$F$ that approximates $T_\star$ uniformly with error strictly less than
$(1-\beta)/2$, then 
$\gap_F(G,H)>\beta$ for $G$ and $H$ non-isomorphic, contradicting the
hypothesis. Hence $\mathcal F$ does not size-uniformly approximate
$T_\star$.
\end{proof}

\section{Quantization for Size-Uniform Function Approximation}
\label{sec:quantization-draft}

\subsection{Quantization setup}

For a finite nonempty palette $S\subset K$, fix a quantization map
$q_S\colon K\to S$ and define its radius by
\begin{equation}
    \delta(S)\coloneq\sup_{\vec{z}\in K}\|\vec{z}-q_S(\vec{z})\|.
\end{equation}
The corresponding nodewise quantization operator is
$(Q_S\vec{x})_v\coloneq q_S(\vec{x}_v)$. By construction,
\begin{equation}
    \|\vec{x}-Q_S\vec{x}\|_{\mathrm{in}}\leq\delta(S)
    \qquad
    \text{for every }\vec{x}\in K^{V(G)}.
\end{equation}
For $\delta>0$, let $N_K(\delta)$ be the covering number of $K$. Since
$K$ is compact, we may choose a palette $S_\delta\subset K$ with
$|S_\delta|=N_K(\delta)$ and take $q_{S_\delta}$ to be a nearest-point map,
so that $\delta(S_\delta)\leq\delta$. Thus $Q_{S_\delta}\vec{\xi}^G$ consists of
i.i.d.\ random colors drawn from a palette of size $N_K(\delta)$.

To make this concrete, take $K=[0,1]^d$, and consider rounding to a uniform grid. For
$m\in\mathbb N$, let
\begin{equation}
\begin{gathered}
    T_m\coloneq
    \left\{\frac{2k-1}{2m}\colon k =1,\ldots,m\right\}\subset[0,1],\\
    S_m\coloneq T_m^d.
\end{gathered}
\end{equation}
Let $q_{S_m}$ round each coordinate to the nearest point in $S_m$, and abbreviate
$Q_m\coloneq Q_{S_m}$. Let
\begin{equation}
    \kappa\coloneq\sup\{\|\vec{u}\|\colon\|\vec{u}\|_\infty\leq1\}.
\end{equation}
 Since rounding changes each coordinate by at
most $1/(2m)$, we have
\begin{equation}
    \|\vec{x}-Q_m\vec{x}\|_{\mathrm{in}}
    \leq\frac{\kappa}{2m}
    \qquad
    \text{for every }\vec{x}\in([0,1]^d)^{V(G)}.
\end{equation}
Consequently,
\begin{equation}
    N_K\left(\frac{\kappa}{2m}\right)\leq m^d.
\end{equation}
If $\rho_K$ is uniform, then $Q_m\vec{\xi}^G$ is uniform on the  palette
$S_m$. The random-color WL results below use uniformly distributed colors.

\subsection{Size-uniform quantization}

For an MPNN $F$ and a finite palette $S\subset K$, define the population
predictor obtained from the quantized features by
\begin{equation}
    \overline F_S(G)
    \coloneq
    \E[F(G,Q_S\vec{\xi}^G)].
\end{equation}

\begin{proposition}[Size-uniform quantization]
\label{prop:size-uniform-quantization-draft}
Let $F$ be a uniformly $L$-Lipschitz MPNN and let
$T\colon\mathcal G\to\R$ be a graph invariant. Then, for every finite
nonempty $S\subset K$,
\begin{equation}
    \sup_{G\in\mathcal G}|\overline F_S(G)-T(G)|
    \leq
    \sup_{G\in\mathcal G}|\overline F(G)-T(G)|
    +L\delta(S).
\end{equation}
\end{proposition}

\begin{proof}
For every graph
$G$,
\begin{equation}
\begin{aligned}
    |\overline F_S(G)-\overline F(G)|
    &\leq \E\bigl[|F(G,Q_S\vec{\xi}^G)-F(G,\vec{\xi}^G)|\bigr] \\
    &\leq L\delta(S).
\end{aligned}
\end{equation}
Taking the supremum over $\mathcal G$ and applying the triangle inequality
completes the proof.
\end{proof}

As a consequence, if $\overline F$ approximates $T$ with error at most
$\varepsilon/2$ and $L>0$, taking $S=S_{\varepsilon/(2L)}$ gives a predictor
that uses $N_K(\varepsilon/(2L))$ colors with error at most $\varepsilon$. For the
midpoint grid on $[0,1]^d$, taking $q=m^d$ gives the explicit rate
\begin{equation}
    \sup_{G\in\mathcal G}|\overline F_{S_m}(G)-T(G)|
    \leq
    \sup_{G\in\mathcal G}|\overline F(G)-T(G)|
    +\frac{L\kappa}{2q^{1/d}}.
\end{equation}
The palette size may depend on the target, accuracy, feature space, and
Lipschitz constant, but not on graph size.

Let $\mathcal A$ be the class of Lipschitz MPNNs with random node
features supported on $K$, and let $\mathcal A_{\mathrm{color}}$ be the
subclass of random-color MPNNs as defined in \cref{sec:formalism-draft}.
\begin{theorem}[Finite colors preserve the approximation class]
\label{thm:color-approximation-class-draft}
The classes $\mathcal A$ and $\mathcal A_{\mathrm{color}}$ uniformly
approximate the same graph invariants.
\end{theorem}

\begin{proof}
One direction follows from $\mathcal A_{\mathrm{color}}\subseteq\mathcal A$.
For the other direction, suppose that $\mathcal A$ uniformly approximates $T$, and fix
$\varepsilon>0$. Choose $F\in\mathcal A$ whose population predictor has error
at most $\varepsilon/2$, and let $L$ be its uniform Lipschitz constant.
Choose a finite palette $S$ such that $L\delta(S)\leq\varepsilon/2$.
Such a palette exists by compactness. By
\cref{prop:size-uniform-quantization-draft}, the quantized model belongs to
$\mathcal A_{\mathrm{color}}$ and approximates $T$ with error at most
$\varepsilon$.
\end{proof}

\subsection{One-hot and vector color encodings}

The quantization reduction uses the vectors in a finite palette $S\subset K$
as colors, while our later results use one-hot vectors for simplicity. We now
compare these two representations.

An encoding of $q$ colors in a feature space $K$ is an injection
$c\colon[q]\to K$. The one-hot encoding maps $i$ to the coordinate basis vector
$\vec{e}_i\in\R^q$.

Let $\widetilde c\colon[q]\to\widetilde K$ be another encoding, and let
$\phi\colon\widetilde K\to K$ be $M$-Lipschitz. Define its conversion error as
\begin{equation}
    \eta
    \coloneq
    \max_{i\in[q]}\|\phi(\widetilde c(i))-c(i)\|.
\end{equation}
For a coloring $\chi\colon V(G)\to[q]$, we can convert the encodings by
$c\circ\chi\coloneq(c(\chi_v))_{v\in V(G)}$.

If $F$ is uniformly $L$-Lipschitz, then applying $\phi$ nodewise before $F$
gives a uniformly $LM$-Lipschitz model. Moreover, for i.i.d.\ colors drawn
from any distribution on $[q]$, uniform Lipschitzness gives
\begin{equation}
    \sup_{G\in\mathcal G}
    \bigl|\E F(G,\phi\circ\widetilde c\circ\chi)
    -\E F(G,c\circ\chi)\bigr|
    \leq L\eta.
\end{equation}

\begin{proposition}[Encoding invariance]
\label{prop:encoding-invariance-draft}
Fix $q$ and any distribution on $[q]$. Uniformly Lipschitz MPNNs using
one-hot colors and those using any injective vector encoding uniformly
approximate the same graph invariants.
\end{proposition}

\begin{proof}
For a fixed finite palette, a biased ReLU network can interpolate each of
the assignments $\vec{e}_i\mapsto c(i)$ and $c(i)\mapsto \vec{e}_i$ exactly. We implement
the conversion using finitely many initial nodewise layers, setting
$\vec{W}_{\mathrm{agg}}^{(\ell)}=0$ in these layers.\footnote{For signed target
encodings, use a nonnegative shift and a constant coordinate, and absorb
the shift into the first message-passing layer.}
Their depth, width, and weights may depend on the encodings, but not on
graph size. The conversion therefore has a graph-size-independent
Lipschitz constant. Composing it with the original MPNN gives the same
output for every coloring. The resulting population predictors are
therefore identical, although their Lipschitz constants may differ by a
factor that is independent of graph size.
\end{proof}

Taking $S=S_\delta$ and $c$ to be a bijection from $[N_K(\delta)]$ onto
$S_\delta$ shows that the vector colors produced by quantization and their
one-hot encodings have the same size-uniform approximation power.

\section{Indistinguishability with Random Features Using CFI}
\label{sec:cfi-draft}
In this section we introduce a family of graph pairs that requires an
exponential number of trials in the graph size to distinguish with a fixed
number of random colors. Our
construction applies beyond MPNNs to any test based on the WL patterns of
these graphs. The results of
\cref{sec:formalism-draft,sec:quantization-draft} convert this
trial lower bound into non-universality of uniformly Lipschitz MPNNs with
continuous random features for bounded graph invariants. The construction
gives this obstruction on connected graphs of maximum degree eight. Adding
a virtual node gives a second family of diameter two with the same trial
lower bound, so non-universality holds under either restriction separately.

\subsection{Random-color WL}
\label{subsec:cfi-random-color-wl-draft}
We begin by defining random-color WL, in which graph isomorphism testing is
performed by comparing the WL patterns of two graphs over independent random
coloring trials. For $q\in\mathbb N$, write $[q]=\{1,\ldots,q\}$.
For a coloring $\chi\colon V(G)\to[q]$, define the
$1$-WL pattern of a vertex recursively. Set
\begin{equation}
        \pi_0^{G,\chi}(v)\coloneq\chi(v),
\end{equation}
and, for $t\geq0$, let
\begin{equation}
        \pi_{t+1}^{G,\chi}(v)
        \coloneq
        \left(
        \pi_t^{G,\chi}(v),
        \{\!\{\pi_t^{G,\chi}(u)\colon u\sim v\}\!\}
        \right).
\end{equation}
Here $u\sim v$ denotes adjacency and $\{\!\{\cdot\}\!\}$ denotes a multiset.
The complete $1$-WL pattern of $v$ is then
$\pi^{G,\chi}(v)\coloneq(\pi_t^{G,\chi}(v))_{t\geq0}$. The graph-level
observation is the unlabeled multiset
\begin{equation}
        \Pi(G,\chi)\coloneq 
        \{\!\{\pi^{G,\chi}(v)\colon v\in V(G)\}\!\}.
\end{equation}
Sample a random coloring $\chi^G\colon V(G)\to[q]$ with
$\chi^G(v)\overset{\mathrm{i.i.d.}}{\sim}\operatorname{Unif}([q])$ and define
\begin{equation}
        Z_{\mathrm{WL},q}(G)\coloneq \Pi(G,\chi^G).
\end{equation}
We write $P_G^{\mathrm{WL},q}$ for the law of $Z_{\mathrm{WL},q}(G)$.

We use the same setting as
\cref{def:statistical-distinguishability-draft}, with WL observations
in place of scalar outputs.

\begin{definition}[Random-color WL distinguishability]
\label{def:rcwl-draft}
For $q,s\in\mathbb N$ and $\alpha\in(0,1/2)$, a non-isomorphic pair $G,H$ is
\emph{statistically $(q,s,\alpha)$-distinguishable by random-color WL} if
there is a binary-valued, possibly randomized test $\psi$ satisfying the following. For every
$A,B\in\{G,H\}$, let
\begin{equation}
\begin{aligned}
        \vec Z_A&=(Z_1(A),\ldots,Z_s(A)),\\
        \vec Z'_B&=(Z'_1(B),\ldots,Z'_s(B))
\end{aligned}
\end{equation}
be independent batches with i.i.d.\ entries of laws $P_A^{\mathrm{WL},q}$ and
$P_B^{\mathrm{WL},q}$, respectively. Then
\begin{equation}
        \max_{A,B\in\{G,H\}}
        \Pr\left(
        \psi(\vec Z_A,\vec Z'_B)\neq\vec 1\{A\neq B\}
        \right)
        \leq\alpha.
\end{equation}
\end{definition}

The test may post-process the complete WL observations arbitrarily and is not
restricted to an empirical mean. For a scalar-output MPNN $F$ and a fixed
color encoding $c\colon[q]\to K$, the observation $Z_{\mathrm{WL},q}(G)$
determines the network output $F(G,c\circ\chi^G)$ by the standard $1$-WL
upper bound \citep{xu2019how}.
Thus random-color WL tests include all tests based on random-color MPNN outputs.

We now construct graph pairs whose random-color WL distributions are
exponentially close in total variation for any fixed number of colors.
The construction uses the
parity gadget from
\citet{cai1992optimal}, which we refer to as the CFI gadget. Let $q\geq2$ denote the number of colors.

\subsection{The parity gadget}

The CFI gadget is designed to construct graphs whose isomorphism
class depends on the parity of choices made across many local structures.

The core gadget has three distinguished unordered vertex pairs
\begin{equation}
        X=(x^0,x^1), \qquad
        Y=(y^0,y^1), \qquad
        Z=(z^0,z^1),
\end{equation}
and four internal vertices
\begin{equation}
        g_{abc},
        \qquad
        (a,b,c)\in\{0,1\}^3,
        \qquad
        a\oplus b\oplus c=0.
\end{equation}
Here $\oplus$ denotes addition modulo two. Each internal vertex $g_{abc}$ is
adjacent exactly to the three distinguished
vertices
\begin{equation}
        x^a,\quad y^b,\quad z^c.
\end{equation}

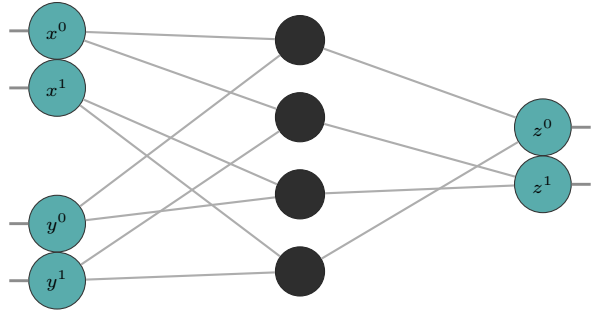
\begin{figure*}[t]
\centering
\begin{tikzpicture}[
    scale=1.02,
    edge/.style={line width=0.8pt, draw=black!32},
    stub/.style={line width=1.1pt, draw=black!45},
    outer/.style={circle, draw=black!80, fill=teal!65, minimum size=7.5mm,
        inner sep=0pt},
    inner/.style={circle, draw=black!85, fill=black!82, minimum size=6.5mm,
        inner sep=0pt},
    label distance=2pt
]
    \node[outer] (x0) at (-3.15, 1.62) {$\scriptstyle x^0$};
    \node[outer] (x1) at (-3.15, 0.88) {$\scriptstyle x^1$};
    \node[outer] (y0) at (-3.15,-0.88) {$\scriptstyle y^0$};
    \node[outer] (y1) at (-3.15,-1.62) {$\scriptstyle y^1$};
    \node[outer] (z0) at ( 3.15, 0.37) {$\scriptstyle z^0$};
    \node[outer] (z1) at ( 3.15,-0.37) {$\scriptstyle z^1$};

    \node[inner] (e0) at (0, 1.50) {};
    \node[inner] (e1) at (0, 0.50) {};
    \node[inner] (e2) at (0,-0.50) {};
    \node[inner] (e3) at (0,-1.50) {};

    \draw[stub] (x0) -- ++(-0.62,0);
    \draw[stub] (x1) -- ++(-0.62,0);
    \draw[stub] (y0) -- ++(-0.62,0);
    \draw[stub] (y1) -- ++(-0.62,0);
    \draw[stub] (z0) -- ++(0.62,0);
    \draw[stub] (z1) -- ++(0.62,0);

    \draw[edge] (e0) -- (x0);
    \draw[edge] (e0) -- (y0);
    \draw[edge] (e0) -- (z0);

    \draw[edge] (e1) -- (x0);
    \draw[edge] (e1) -- (y1);
    \draw[edge] (e1) -- (z1);

    \draw[edge] (e2) -- (x1);
    \draw[edge] (e2) -- (y0);
    \draw[edge] (e2) -- (z1);

    \draw[edge] (e3) -- (x1);
    \draw[edge] (e3) -- (y1);
    \draw[edge] (e3) -- (z0);
\end{tikzpicture}
\caption{The parity/CFI gadget.}
\label{fig:cfi-parity-gadget-draft}
\end{figure*}

The central property is that endpoint swaps extend to an automorphism over the entire
gadget precisely when they have even parity. Interpreting each swap as a bit
flip, the third swap is therefore the XOR of the first two.
This is the $k=3$ case of Lemma~6.1 of \citet{cai1992optimal}.

\begin{lemma}[Extension property of the parity gadget]
\label{lem:extension_property_xor_gate-draft}
Let $s_X,s_Y,s_Z\in\{0,1\}$, and let $\tau$ be the permutation of the distinguished
vertices that swaps the endpoints of pair $X$ if $s_X=1$, swaps the endpoints
of pair $Y$ if $s_Y=1$, and swaps the endpoints of pair $Z$ if $s_Z=1$.

Then $\tau$ extends to an automorphism of the parity gadget if and
only if
\begin{equation}
        s_X\oplus s_Y\oplus s_Z=0.
\end{equation}
\end{lemma}

\subsection{The graph pair family}

We now compose the parity gadgets so that a parity choice made at any input
pair can be transmitted to a single root pair.

Fix $k\geq2$ input vertex pairs
\begin{equation}
        X_i=(x_i^0,x_i^1),
        \qquad i=1,\ldots,k.
\end{equation}
For each $i$, attach two six-vertex blocks to the input pair as follows. Attach
$x_i^0$ to all six vertices of a copy of $2C_3$, and attach $x_i^1$ to all six
vertices of a copy of $C_6$. Here $2C_3$ is the disjoint union of two triangles
and $C_6$ is a six-cycle. Thus the two sides of the input pair are
non-isomorphic as rooted attached blocks.

Next, build a full binary tree of parity gadgets whose leaves are the input pairs
$X_1,\ldots,X_k$. Each internal node of the tree is one copy of the parity
gadget above. We identify each child's output pair with its parent's input
pair, matching superscripts. The output of the tree is a final vertex pair
\begin{equation}
        R=(r^0,r^1).
\end{equation}
Finally, attach a marker leaf to $r^0$ in $G_0$ and to $r^1$ in $G_1$.
We write $G_{0,k},G_{1,k}$ when the dependence on $k$ matters. The $2k-1$
vertex pairs, $k-1$ gadgets, attached blocks, and marker give
\begin{equation}
        n_k\coloneq |V(G_{0,k})|=|V(G_{1,k})|=20k-5.
\end{equation}
Each input endpoint has degree eight, each intermediate endpoint has degree
four, and every other vertex has degree at most three. Thus both graphs are
connected and have maximum degree eight.
The case $k=4$ is shown in \cref{fig:cfi-graph-pair-draft}.

For $j\in\{0,1\}$, let $\widehat G_{j,k}$ be obtained from $G_{j,k}$ by
adding a \emph{virtual node} $\nu_j$ adjacent to every original vertex,
including the marker. We omit $k$ when it is clear from context. Then
\begin{equation}
        \widehat n_k\coloneq |V(\widehat G_{j,k})|=n_k+1=20k-4.
\end{equation}
The graphs $\widehat G_{j,k}$ have diameter two. Their original vertices
have degree at most nine, while $\nu_j$ has degree $n_k$, so the virtual
node is the unique vertex of degree greater than nine. In random-color WL,
it receives an independent uniform color from the same palette as every
other vertex.

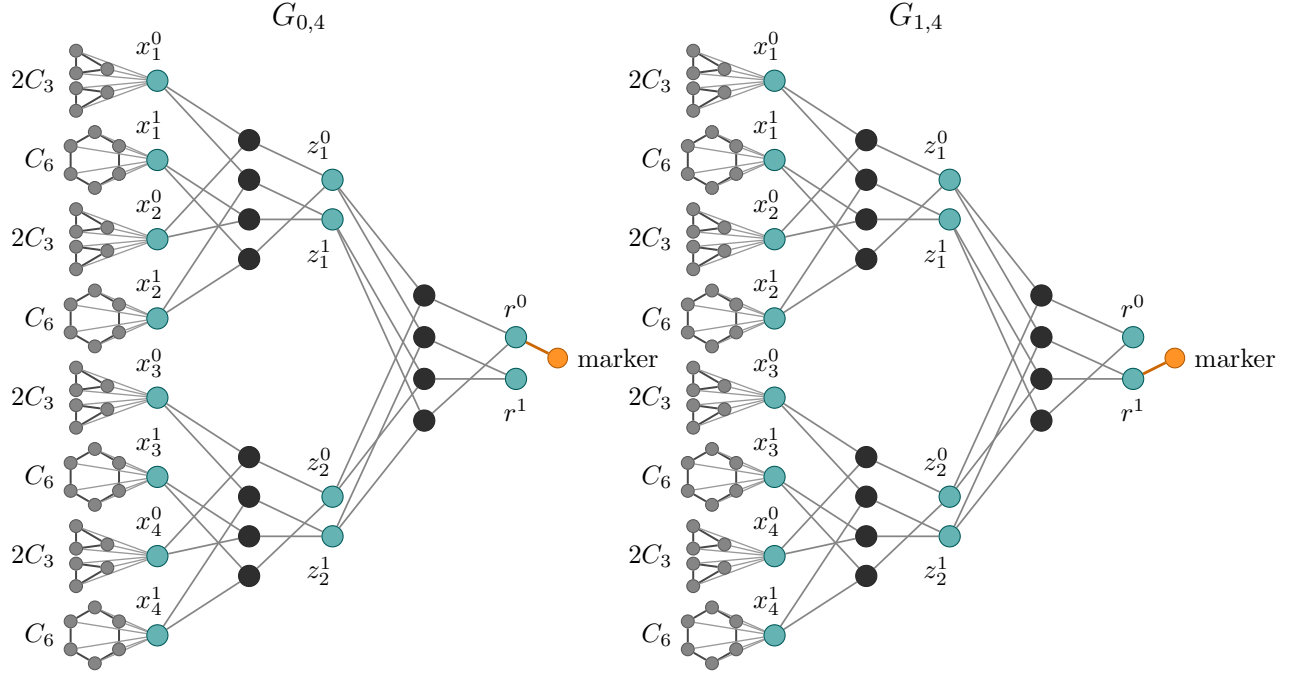
\begin{figure*}[t]
\centering
\begingroup
\pgfmathsetlengthmacro{\cfiunit}{(\linewidth-70pt)/26.62}
\begin{tikzpicture}[
    x=\cfiunit,y=\cfiunit,
    font=\fontsize{10}{12}\selectfont,
    endpoint/.style={circle,draw=teal!70!black,fill=teal!60,
        minimum size=2.8mm,inner sep=0pt},
    internal/.style={circle,draw=black!85,fill=black!82,
        minimum size=2.8mm,inner sep=0pt},
    block vertex/.style={circle,draw=black!65,fill=black!48,
        minimum size=1.7mm,inner sep=0pt},
    marker/.style={circle,draw=orange!70!black,fill=orange!85,
        minimum size=2.6mm,inner sep=0pt},
    gadget edge/.style={draw=black!48,line width=.65pt},
    attachment edge/.style={draw=black!40,line width=.5pt},
    block edge/.style={draw=black!72,line width=.75pt},
    marker edge/.style={draw=orange!80!black,line width=1.05pt}
]
\foreach \which/\shift in {0/0,1/14.8} {
    \begin{scope}[shift={(\shift,0)}]
        \node[font=\fontsize{12}{14}\selectfont] at (3.4,8.15) {$G_{\which,4}$};

        \foreach \i/\bit/\height in {
            1/0/6.65,1/1/4.75,2/0/2.85,2/1/.95,
            3/0/-.95,3/1/-2.85,4/0/-4.75,4/1/-6.65} {
            \node[endpoint,label={[xshift=-3pt]above:{$x_{\i}^{\bit}$}}]
                (x\i\bit) at (0,\height) {};
            \begin{scope}[shift={(0,\height)}]
                \ifnum\bit=0
                    \foreach \j/\xx/\yy in {
                        1/-1.95/.72,2/-1.95/.18,3/-1.20/.28,
                        4/-1.95/-.18,5/-1.95/-.72,6/-1.20/-.28} {
                        \node[block vertex] (a\i\bit-\j) at (\xx,\yy) {};
                    }
                    \node[anchor=east] at (-2.22,0) {$2C_3$};
                    \begin{scope}[on background layer]
                        \draw[block edge]
                            (a\i\bit-1) -- (a\i\bit-2) -- (a\i\bit-3) -- (a\i\bit-1);
                        \draw[block edge]
                            (a\i\bit-4) -- (a\i\bit-5) -- (a\i\bit-6) -- (a\i\bit-4);
                    \end{scope}
                \else
                    \foreach \j/\angle in {1/90,2/150,3/210,4/270,5/330,6/30} {
                        \node[block vertex] (a\i\bit-\j)
                            at ({-1.50+.67*cos(\angle)},{.67*sin(\angle)}) {};
                    }
                    \node[anchor=east] at (-2.22,0) {$C_6$};
                    \begin{scope}[on background layer]
                        \draw[block edge]
                            (a\i\bit-1) -- (a\i\bit-2) -- (a\i\bit-3)
                            -- (a\i\bit-4) -- (a\i\bit-5) -- (a\i\bit-6) -- (a\i\bit-1);
                    \end{scope}
                \fi
                \begin{scope}[on background layer]
                    \foreach \j in {1,...,6} {
                        \draw[attachment edge] (x\i\bit) -- (a\i\bit-\j);
                    }
                \end{scope}
            \end{scope}
        }

        \foreach \branch/\leftinput/\rightinput/\center in {1/1/2/3.8,2/3/4/-3.8} {
            \foreach \triple/\height in {000/1.425,011/.475,101/-.475,110/-1.425} {
                \node[internal] (g\branch-\triple) at (2.2,{\center+\height}) {};
            }
            \node[endpoint,label={[xshift=-5pt]above:{$z_{\branch}^0$}}]
                (z\branch0) at (4.2,{\center+.475}) {};
            \node[endpoint,label={[xshift=-5pt]below:{$z_{\branch}^1$}}]
                (z\branch1) at (4.2,{\center-.475}) {};
            \begin{scope}[on background layer]
                \foreach \triple/\a/\b/\c in {000/0/0/0,011/0/1/1,101/1/0/1,110/1/1/0} {
                    \draw[gadget edge] (g\branch-\triple) -- (x\leftinput\a);
                    \draw[gadget edge] (g\branch-\triple) -- (x\rightinput\b);
                    \draw[gadget edge] (g\branch-\triple) -- (z\branch\c);
                }
            \end{scope}
        }

        \foreach \triple/\height in {000/1.5,011/.5,101/-.5,110/-1.5} {
            \node[internal] (groot-\triple) at (6.4,\height) {};
        }
        \node[endpoint,label={above:{$r^0$}}] (r0) at (8.6,.5) {};
        \node[endpoint,label={below:{$r^1$}}] (r1) at (8.6,-.5) {};
        \node[marker,label={[font=\fontsize{10}{12}\selectfont]right:marker}]
            (marker) at (9.6,0) {};
        \begin{scope}[on background layer]
            \foreach \triple/\a/\b/\c in {000/0/0/0,011/0/1/1,101/1/0/1,110/1/1/0} {
                \draw[gadget edge] (groot-\triple) -- (z1\a);
                \draw[gadget edge] (groot-\triple) -- (z2\b);
                \draw[gadget edge] (groot-\triple) -- (r\c);
            }
            \draw[marker edge] (r\which) -- (marker);
        \end{scope}
    \end{scope}
}
\end{tikzpicture}
\endgroup\space
\caption{The CFI graph pair with four input pairs ($k=4$).
The intermediate pairs $(z_1^0,z_1^1)$ and $(z_2^0,z_2^1)$ are each shared
by a child gadget and the root gadget. Shading indicates vertex roles,
not random colors. The diameter-two variants are obtained by adding one
virtual node adjacent to every vertex shown in each graph.}
\label{fig:cfi-graph-pair-draft}
\end{figure*}

\subsection{Non-isomorphism}
We next verify that the graphs $G_0$ and $G_1$,
obtained by attaching the marker leaf to $r^0$ and $r^1$, respectively, are
non-isomorphic.

\begin{lemma}
\label{lem:cfi-nonisomorphic-draft}
The graphs $G_0$ and $G_1$ are not isomorphic. The graphs $\widehat G_0$ and
$\widehat G_1$ are also not isomorphic.
\end{lemma}

\begin{proof}
Suppose that $\gamma\colon V(G_0)\to V(G_1)$ is a graph isomorphism, and let $m_j$
be the marker leaf of $G_j$. Since the marker and the unmarked root endpoint
are the unique vertices of degrees one and two, respectively,
$\gamma(m_0)=m_1$, $\gamma(r^0)=r^1$, and $\gamma(r^1)=r^0$.
The four internal vertices of the root gadget are precisely the neighbors
of $r^0$ or $r^1$ other than the marker. Since $\gamma$ preserves adjacency
and maps the root pair and marker to their counterparts, it maps this
four-vertex set in $G_0$ onto the corresponding set in $G_1$.
The individual internal vertices may be permuted.

For a gadget with output pair $W=(w^0,w^1)$, let $A$ be its set of four
internal vertices. Two distinct input vertices $u,v$ belong to the same pair
if and only if $N(u)\cap N(v)\cap A=\varnothing$, where $N(v)$ denotes
the neighborhood of $v$. Since $\gamma$ preserves
this condition, it maps input pairs to input pairs, possibly interchanging
the two. Such a pair is a leaf of the gadget tree exactly when its endpoints
have degree eight. Otherwise, the four internal vertices of its child gadget
are the neighbors of its endpoints outside $A$. Induction from the root
therefore shows that $\gamma$ preserves the pair--gadget tree, up to
permutation of sibling subtrees.

For each pair $W=(w^0,w^1)$, let $W'=(u^0,u^1)$ be its image pair and
define $s_W\in\{0,1\}$ by
\begin{equation}
        \gamma(w^b)=u^{b\oplus s_W},\qquad b\in\{0,1\}.
\end{equation}
Since $\gamma$ exchanges the root endpoints, $s_R=1$. At every input pair $X_i$, the
non-isomorphic $2C_3$ and $C_6$ attachments force $s_{X_i}=0$.
By \cref{lem:extension_property_xor_gate-draft}, every gadget with input pairs
$X,Y$ and output pair $Z$ satisfies
\begin{equation}
        s_Z=s_X\oplus s_Y.
\end{equation}
This relation is unchanged when we interchange the two input pairs.
We prove $s_W=0$ for every pair $W$ by induction on the height of the
subtree rooted at $W$. At height zero, $W$ is an input pair $X_i$, for
which $s_{X_i}=0$. For the induction step, let $Z$ be the output pair of
a gadget with input pairs $X,Y$. Their subtrees have smaller height, so
the induction hypothesis gives $s_X=s_Y=0$. Hence
$s_Z=s_X\oplus s_Y=0$. Taking $W=R$ gives $s_R=0$, contradicting $s_R=1$.

Any isomorphism between $\widehat G_0$ and $\widehat G_1$ must map $\nu_0$
to $\nu_1$, since these are their unique vertices of degree greater than
nine. Deleting the virtual nodes would then give an isomorphism between
$G_0$ and $G_1$, which we have ruled out.
\end{proof}

\subsection{Block erasure}

We next identify an event under which the random coloring conceals the local
difference between the $C_6$ and $2C_3$ attachments at an input pair. For each
input pair $X_i$, let $B_i$ be the event that both attached six-vertex blocks
are monochromatic. The two blocks need not have the same color.

Since a six-vertex block is monochromatic with probability $q^{-5}$, we have
\begin{equation}
        \Pr(B_i)=q^{-10}.
\end{equation}
\begin{lemma}[Block erasure]
\label{lem:cfi-block-erasure-draft}
Let $G\in\{G_0,G_1\}$ and let $\chi$ be a coloring for which $B_i$ holds.
Let $H$ be obtained from $G$ by replacing the $2C_3$ attachment at $x_i^0$
with $C_6$ and the $C_6$ attachment at $x_i^1$ with $2C_3$, on the same
vertex sets and with all other edges unchanged. Then
\begin{equation}
        \Pi(G,\chi)=\Pi(H,\chi).
\end{equation}
\end{lemma}

\begin{proof}
At round zero, corresponding vertices have the same color in $G$ and $H$,
and each replaced block is monochromatic. Suppose that corresponding
vertices have the same patterns at round $t$ and that all vertices within
each replaced block share a pattern. In both graphs, each block vertex
has two neighbors within its block and one at the unchanged attachment
endpoint. Thus, at round $t+1$, corresponding block vertices have the same
pattern, and all vertices within each block share a pattern. All other
vertices have unchanged neighborhoods, so their patterns also agree.
By induction, corresponding vertices have the same patterns at every round,
and the graph-level observations are equal.
\end{proof}

\subsection{One erased input hides the root}

The parity tree is designed so that concealing one input pair also conceals
which root endpoint carries the marker leaf.

Let $B\coloneq\bigcup_{i=1}^k B_i$ be the event that at least one input
block pair is erased.

\begin{lemma}
\label{lem:cfi-root-hiding-draft}
Conditioned on $B$, the random-color WL sample distributions of $G_0$ and $G_1$ are
identical,
\begin{equation}
        \Law(Z_{\mathrm{WL},q}(G_0)\mid B)
        =\Law(Z_{\mathrm{WL},q}(G_1)\mid B).
\end{equation}
The same equality holds with $\widehat G_0$ and $\widehat G_1$ in place of
$G_0$ and $G_1$.
\end{lemma}

\begin{proof}
On $B$, choose an erased input pair $X_i$. Let $G_{\mathrm{core}}$ be the graph formed by
the parity gadgets, without the attached blocks or marker. We define an
automorphism $\tau_i\colon V(G_{\mathrm{core}})\to V(G_{\mathrm{core}})$ along the unique path from
$X_i$ to $R$ in the parity-gadget tree. For each pair $W=(w^0,w^1)$,
set $s_W=1$ if $W$ lies on this path and $s_W=0$ otherwise, and define
\begin{equation}
        \tau_i(w^b)=w^{b\oplus s_W},\qquad b\in\{0,1\}.
\end{equation}
Each gadget has either two incident pairs swapped or none, so its swap
bits satisfy $s_X\oplus s_Y\oplus s_Z=0$.
By \cref{lem:extension_property_xor_gate-draft}, these swaps extend to a gadget
automorphism. Explicitly, on its internal vertices, set
\begin{equation}
        \tau_i(g_{abc})=g_{a\oplus s_X,\,b\oplus s_Y,\,c\oplus s_Z}.
\end{equation}
Since the gadget maps agree on shared pairs and the gadgets have disjoint
internal vertex sets, these maps define an automorphism $\tau_i$ of $G_{\mathrm{core}}$.

To extend $\tau_i$ to a bijection $V(G_0)\to V(G_1)$, map the marker
of $G_0$ to that of $G_1$ and fix all attached blocks except those at $X_i$.
Exchange the two blocks at $X_i$ using a fixed bijection in one direction
and its inverse in the other. Since $\tau_i(r^b)=r^{b\oplus1}$, this
extension preserves the marker edge.
The only edges it fails to preserve are within the exchanged $2C_3$ and
$C_6$ blocks. For a coloring $\chi$ satisfying $B_i$, both blocks are
monochromatic. Thus \cref{lem:cfi-block-erasure-draft} gives
\begin{equation}
        \Pi(G_0,\chi)=\Pi(G_1,\chi\circ\tau_i^{-1}).
\end{equation}
Here $\chi\circ\tau_i^{-1}$ assigns to each image vertex the color of
its preimage.

The coloring map $\chi\mapsto\chi\circ\tau_i^{-1}$ preserves the set of
erased input pairs. If we use the same choice of $X_i$ for each such set,
this map is reversible and gives a bijection between the colorings
satisfying $B$. This bijection preserves their uniform conditional law.
Since matched colorings have identical WL observations, their laws agree
conditioned on $B$.

For the virtual-node graphs, extend each bijection by setting
$\tau_i(\nu_0)=\nu_1$ and carrying the color of $\nu_0$ to $\nu_1$.
This extension preserves all edges incident to the virtual node. The
only remaining discrepancies are within the two exchanged monochromatic
blocks. The induction in \cref{lem:cfi-block-erasure-draft} still applies:
each block vertex receives the additional pattern of the virtual node,
and the two virtual nodes receive equal multisets of patterns from all
original vertices. Thus corresponding vertices have identical patterns at
every round. Since the virtual-node color is independent of $B$, the
extended coloring bijection also preserves the uniform conditional law.
\end{proof}

\subsection{Total variation and sample complexity}

We now convert the root-hiding event into an upper bound on total variation
and, subsequently, into a lower bound on the required number of trials.

\begin{proposition}[Upper bound on CFI total variation]
\label{prop:cfi-tv-upper-draft}
The random-color WL distributions of $G_0$ and $G_1$ satisfy
\begin{equation}
    \operatorname{TV}(P_{G_0}^{\mathrm{WL},q},P_{G_1}^{\mathrm{WL},q})
    \leq (1-q^{-10})^k.
\end{equation}
For fixed $q$, this bound is $\exp(-\Omega(k))$.
The same bound holds for $\widehat G_0$ and $\widehat G_1$.
\end{proposition}

\begin{proof}
By \cref{lem:cfi-root-hiding-draft}, the conditional laws on $B$ agree.
Since $B$ has the same probability in both graphs,
\begin{equation}
    \operatorname{TV}(P_{G_0}^{\mathrm{WL},q},P_{G_1}^{\mathrm{WL},q})
    \leq \Pr(B^c).
\end{equation}
Here $B^c$ denotes the complement of $B$, namely the event that no input
block pair is erased.
The events $B_i$ are independent because they depend on disjoint attached
blocks. Therefore
\begin{equation}
    \Pr(B^c)
    =\prod_{i=1}^k\Pr(B_i^c)
    =(1-q^{-10})^k.
\end{equation}
The conditional-law equality in \cref{lem:cfi-root-hiding-draft} and the
probabilities of the events $B_i$ are unchanged by adding the virtual
node, so the same argument proves the bound for $\widehat G_0,\widehat G_1$.
\end{proof}

\begin{corollary}[Lower bound on the number of random-color WL trials]
\label{cor:cfi-trial-lower-draft}
Fix $q\geq2$ and $\alpha\in(0,1/2)$. If the pair $(G_0,G_1)$ is statistically
$(q,s,\alpha)$-distinguishable by random-color WL, then
\begin{equation}
    s\geq (1-2\alpha)(1-q^{-10})^{-k}.
\end{equation}
In particular, for fixed $q$ and $\alpha$, the required number of trials is
$\exp(\Omega(n_k))$.
The same lower bound holds for $(\widehat G_0,\widehat G_1)$, with
$\widehat n_k$ in place of $n_k$ in the asymptotic statement.
\end{corollary}

\begin{proof}
By \cref{prop:cfi-tv-upper-draft} and subadditivity of total variation under
products, $s$ independent trials satisfy
\begin{equation}
    \operatorname{TV}\!\left(
        (P_{G_0}^{\mathrm{WL},q})^{\otimes s},
        (P_{G_1}^{\mathrm{WL},q})^{\otimes s}
    \right)
    \leq s(1-q^{-10})^k.
\end{equation}
Here $P\otimes Q$ denotes the joint law of independent draws from $P$ and
$Q$, and $P^{\otimes s}$ denotes the joint law of $s$ independent draws
from $P$.
Let $P_s\coloneq(P_{G_0}^{\mathrm{WL},q})^{\otimes s}$ and
$Q_s\coloneq(P_{G_1}^{\mathrm{WL},q})^{\otimes s}$. A test with failure
probability at most $\alpha$ has expected output at most $\alpha$ under
$P_s\otimes P_s$ and at least $1-\alpha$ under $P_s\otimes Q_s$. Consequently,
\begin{equation}
\begin{aligned}
    1-2\alpha
    &\leq \operatorname{TV}(P_s\otimes P_s,P_s\otimes Q_s)\\
    &=\operatorname{TV}(P_s,Q_s)\\
    &\leq s(1-q^{-10})^k.
\end{aligned}
\end{equation}
Rearranging gives the first claim. The second follows from
$n_k=\Theta(k)$ when $q$ and $\alpha$ are fixed.
For the virtual-node pair, the same total-variation bound applies and
$\widehat n_k=\Theta(k)$, giving the remaining claim.
\end{proof}

\subsection{Non-universality}

We now use the trial lower bound to prove non-universality for bounded
graph invariants, separately on bounded-degree graphs and bounded-diameter
graphs.

\begin{corollary}[Non-universality from CFI pairs]
\label{cor:cfi-size-uniform-approximation-draft}
The class of scalar-output uniformly Lipschitz MPNNs with uniform
initialization on $[0,1]^d$ is not a size-uniform universal approximator for
bounded graph invariants on either of the following graph classes:
\begin{enumerate}[label=(\roman*)]
    \item connected graphs of maximum degree at most eight;
    \item graphs of diameter two.
\end{enumerate}
\end{corollary}

\begin{proof}
Fix a uniformly Lipschitz model $F$. By
\cref{prop:size-uniform-quantization-draft}, choose a midpoint grid
$S_m$, with $m\geq2$, such that quantization changes each population mean
by at most $1/8$, uniformly over graphs. The quantized features are uniform
on the $q=m^d$ grid points. The restriction of $F$ to $S_m$ remains
uniformly Lipschitz. By \cref{prop:finite-trial-detection-draft}, a
finite number $s$ of trials suffices to detect any quantized population gap
of at least $1/4$ with failure probability at most $1/3$.
These tests are also random-color WL tests, and both $q$ and $s$ are
independent of graph size.

Choose $H_{j,k}=G_{j,k}$ for the first graph class and
$H_{j,k}=\widehat G_{j,k}$ for the second. For sufficiently large $k$,
\cref{cor:cfi-trial-lower-draft} rules out statistical
$(q,s,1/3)$-distinguishability of $(H_{0,k},H_{1,k})$.
The quantized population gap must therefore be less than $1/4$.
Including the two quantization errors gives
$\gap_F(H_{0,k},H_{1,k})<1/2$.

For the chosen graph class, define a bounded graph invariant
$T_{\mathrm{CFI}}$ to equal $1$ on graphs isomorphic to $H_{1,k}$ for some
$k\geq2$ and $0$ otherwise. By \cref{lem:cfi-nonisomorphic-draft} and the
strictly increasing orders of each family, we have
$T_{\mathrm{CFI}}(H_{0,k})=0$ and $T_{\mathrm{CFI}}(H_{1,k})=1$ for every
$k$. Uniform approximation to
$T_{\mathrm{CFI}}$ with error less than $1/4$ would force
$\gap_F(H_{0,k},H_{1,k})>1/2$ for every $k$, contradicting the bound above.
Thus the model class fails to uniformly approximate a bounded graph
invariant on each of the two graph classes.
\end{proof}

\subsection{Colors and trials}

The trial lower bound also quantifies the tradeoff between the number of
colors and the number of trials.

\begin{corollary}[Colors and trials]
\label{cor:cfi-concentrated-randomness-draft}
Distinguishing the CFI pair $(G_{0,k},G_{1,k})$ by random-color WL with
success probability at least $2/3$, using $q\geq2$ colors and $s$ trials
per graph, requires
\begin{equation}
        q^{10}\log(3s)=\Omega(n_k).
\end{equation}
The same conclusion holds for $(\widehat G_{0,k},\widehat G_{1,k})$ with
$\widehat n_k$ in place of $n_k$.
\end{corollary}

\begin{proof}
By \cref{cor:cfi-trial-lower-draft} with $\alpha=1/3$,
\begin{equation}
        \log(3s)\geq-k\log(1-q^{-10})\geq kq^{-10}.
\end{equation}
Multiplying by $q^{10}$ and using $n_k=20k-5$ proves the claim.
For the virtual-node pair, the same trial lower bound and
$\widehat n_k=20k-4$ give the stated conclusion.
\end{proof}

The number of colors enters this bound polynomially, whereas the number of
trials enters only logarithmically. Thus, compensating for a smaller color
palette can require exponentially more trials.

Conversely, one trial per graph with $q=3n_k^2$ colors achieves success
probability at least $2/3$. By the union bound, both colorings are injective
with probability at least $1-2\binom{n_k}{2}/q\geq2/3$. On this event,
one WL refinement determines each graph up to isomorphism. Define the test to return zero if either coloring
is not injective and otherwise compare the reconstructed graphs. This test
satisfies the required guarantee for every $A,B\in\{G_{0,k},G_{1,k}\}$.
Thus polynomially many colors suffice in a
single trial, whereas any fixed number of colors requires exponentially
many trials by \cref{cor:cfi-trial-lower-draft}.
For the virtual-node pair, the same injective-coloring test uses
$q=3\widehat n_k^2$ colors. Therefore this exponential separation holds
both on graphs of maximum degree eight and, separately, on graphs of
diameter two.

\section{Separation from Uncolored 1-WL}

Now we show that random colors strictly increase expressive power over
standard uncolored $1$-WL.

\begin{proposition}[Random colors detect three-step returns]
\label{prop:positive-separation}
For arbitrarily large $n$, there are $n$-vertex graphs $G_n,H_n$
that are $1$-WL indistinguishable, but are separated by
a depth-$3$, uniformly Lipschitz random-color MPNN $F$ using two
initial colors, with $\gap_F(G_n,H_n)=1/4$.
\end{proposition}

\begin{proof}
For each $\ell\geq1$, set $n=6\ell$. Let $G_n$ and $H_n$ be disjoint unions
of $\ell$ copies of $2C_3$ and $C_6$, respectively. Both graphs are $2$-regular
on $n$ vertices, so uncolored $1$-WL yields one color class of size $n$ at
every round on both graphs. Thus, the graphs are $1$-WL indistinguishable.

Define the normalized neighbor-averaging operator $\vec{M}_G$ by
\begin{equation}
    (\vec{M}_G\vec{x})_v\coloneq\frac1{\deg_G(v)}\sum_{u\sim v}x_u
    \qquad\text{if }\deg_G(v)>0,
\end{equation}
and set $(\vec{M}_G\vec{x})_v=0$ at isolated vertices. Write $\vec{M}_G^3$ for three successive
applications of this operator. For initial features $\vec{x}\in[-1,1]^{V(G)}$, set
\begin{equation}
    F(G,\vec{x})\coloneq
    \frac{1}{|V(G)|}\sum_{v\in V(G)}
    \left|x_v+(\vec{M}_G^3\vec{x})_v\right|.
\end{equation}
The MPNN stores signed values through their positive and negative parts
and performs normalized neighbor averaging. In the third layer, it computes the absolute value using
$|z|=\operatorname{ReLU}(z)+\operatorname{ReLU}(-z)$. Normalized averaging is nonexpansive in the
max-over-nodes norm, and the final average does not introduce any dependence
on $|V(G)|$. Thus $F$ is uniformly $2$-Lipschitz.

Let each vertex receive an independent uniform color
$\xi_v\in\{-1,1\}$. The identity $|a+b|=1+ab$
for $a\in\{-1,1\}$ and $b\in[-1,1]$ gives
\begin{equation}
    F(G,\vec{\xi}^G)=1+
    \frac1{|V(G)|}\sum_{v\in V(G)}
    \xi_v(\vec{M}_G^3\vec{\xi}^G)_v.
\end{equation}
Let $r_G(v)$ be the probability that a length-$3$ random walk from $v$
returns to $v$. The centered colors are independent with zero mean, so all terms with
$u\neq v$ are zero in expectation, while $\E\xi_v^2=1$.
Therefore,
\begin{equation}
    \overline F(G)=1+
    \frac{1}{|V(G)|}\sum_{v\in V(G)}r_G(v).
\end{equation}
On $2C_3$, this probability is $1/4$ at every vertex,
as the walk can go around its triangle in one of the two directions. On
$C_6$, it is $0$, since $C_6$ is bipartite and a walk of odd length cannot
return to its starting vertex. Thus $\overline F(G_n)=5/4$ and
$\overline F(H_n)=1$, giving $\gap_F(G_n,H_n)=1/4$.
\end{proof}

By \cref{prop:finite-trial-detection-draft}, this constant population
gap yields statistical distinguishability using a constant number of trials.

\section{Incomparability with Higher-Order WL}
\label{app:kwl-incomparability}

We now compare random-color MPNNs with higher-order WL, using statistical
distinguishability as defined in \cref{def:rcwl-draft}.
Under the convention stated in the main text, $k$-WL colors ordered
$k$-tuples of vertices. Initially, a tuple's color records equality and
adjacency among its entries. For an update, each vertex $w$ gives a list
of $k$ colors: those of the tuples obtained by replacing the first entry
by $w$, then the second, and so on. The new color records the old color
and the multiset of these lists over all vertices $w$.

\begin{theorem}[Incomparability with higher-order WL]
\label{thm:kwl-incomparability}
For every $k\geq2$:
\begin{enumerate}
    \item some $k$-WL-indistinguishable pair is distinguished with success
    probability greater than $2/3$ by a uniformly $1$-Lipschitz two-color
    MPNN with order-normalized sum aggregation using one random-color
    trial per graph;
    \item for every fixed $q,s\geq1$, some pair distinguished by $2$-WL,
    and hence by $k$-WL, is not statistically
    $(q,s,1/3)$-distinguishable by random-color WL, and hence cannot be distinguished with the
    same guarantee by any $q$-color MPNN using $s$ trials.
\end{enumerate}

\end{theorem}

\begin{proof}
\noindent\textbf{Random colors beat $k$-WL.}
Choose connected, non-isomorphic CFI graphs $C^0,C^1$
on $n$ vertices, each of degree three, that are $k$-WL equivalent
\citep[Theorem~5.2 and Corollary~6.5]{cai1992optimal}.
Forget their colors; this preserves $k$-WL equivalence and non-isomorphism
\citep[Corollary~7.11]{chen2025uncolored}.
Attach a path $v=v_0,v_1,\ldots,v_t$ to each original vertex $v$, adding
$t=\lceil\log_2(3n^2)\rceil$ new vertices, and denote the resulting graphs by
$\widehat C^0,\widehat C^1$, of order $N=n(t+1)$.
Each original vertex now has degree four, while the added vertices have
degree one or two. Any isomorphism between the new graphs would therefore
map original vertices to original vertices, giving an isomorphism between
$C^0$ and $C^1$. Thus the new graphs remain non-isomorphic.

We verify directly that attaching these paths also preserves $k$-WL
equivalence. Consider a tuple
$((u_1)_{j_1},\ldots,(u_k)_{j_k})$, where $u_i$ is an original vertex
and $j_i\in\{0,\ldots,t\}$ is a position on its attached path.
We claim that after $r$ iterations its color is determined by the path
positions $(j_1,\ldots,j_k)$ and the color of $(u_1,\ldots,u_k)$ after
$r$ iterations on the original graph.
Initially, the original tuple color records which $u_i$ coincide and
which are adjacent. Together with the path positions, this determines
equality and adjacency in the new tuple: edges either join consecutive
positions on the same path or adjacent original vertices at position zero.

For the induction step, every replacement vertex has the form $w_\ell$.
Consider the $k$ tuples obtained by replacing each entry in turn with
$w_\ell$. By induction, their list of colors is determined by the
corresponding list using $w$ in the original graph, together with the
path positions. The next original WL color records its current color
and the multiset of these lists over all $w$. Repeating this for each
$\ell\in\{0,\ldots,t\}$ accounts for every replacement in the new graph,
proving the claim. Since the original graphs have identical tuple-color
counts at every iteration, the new graphs do too, for each choice of
path positions. Hence $\widehat C^0$ and $\widehat C^1$ remain $k$-WL equivalent.

Independent uniform binary features give the original vertices independent
uniform identifiers $\vec z_v=(\xi_{v_1},\ldots,\xi_{v_t})$, with
\begin{equation}
 \Pr\{\text{an identifier collision}\}
 \leq\binom n2 2^{-t}\leq\frac16.
\end{equation}
We compute these identifiers directly with an order-normalized sum MPNN.
Aggregating a constant-one channel at $u$ gives $\deg(u)/N$ on these
graphs; multiplying by the fixed constant $N$ recovers the degree.
We mark the degree-one path tips and propagate this indicator to
neighbors for $t$ rounds, retaining previously reached vertices.
The vertex $v_j$ is first reached by this propagated indicator at round
$t-j$, so the first arrival time identifies its path position.
In particular, the original vertices are exactly those first reached
at round $t$, so they can identify themselves.
Each vertex $u$ then stores a binary matrix $B_u\in\{0,1\}^{2\times t}$,
whose $(i,j)$ entry indicates that color $i\in\{0,1\}$ has been
observed at path position $j$.
Initialize all matrices to zero, then at each added vertex in position
$j$ with color $i$, set entry $(i,j)$ to one.
For $t-1$ rounds, each added vertex combines its own and its neighbors'
matrices entrywise by logical OR (computed by summing and clipping at
one), while original vertices keep their matrices zero.
Every added vertex then learns its path's complete matrix.
One further aggregation delivers this matrix to the original vertex
through its unique path neighbor. Its other neighbors are original
vertices and contribute zero. The matrix thus determines $\vec z_v$.

Next, original vertices exchange one-hot encodings of their identifiers,
with added vertices sending zero. Each original vertex encodes its
incident original edges in a matrix whose rows and columns are indexed
by the $2^t$ possible identifiers. In the row indexed by its own
identifier, it places a one in each column corresponding to a neighbor's
identifier; all other entries are zero. Added vertices contribute zero
matrices. When identifiers are distinct, mean pooling followed by
multiplication by $N$ gives the original graph's adjacency matrix in
these identifier coordinates. Pooling the one-hot identifiers alongside
the matrices also gives the number of vertices with each identifier.
On the finite set of possible pooled inputs, a ReLU MLP can represent
any function exactly. We choose it to accept precisely when $n$ distinct
identifiers occur, and the graph encoded on those identifiers is
isomorphic to $C^0$. A collision leaves fewer than $n$ occupied identifiers
and is rejected. The resulting Boolean output satisfies
$F(\widehat C^0)=1$ precisely on collision-free inputs, and
$F(\widehat C^1)=0$ always.

All the Boolean operations above have exact ReLU implementations.
The required neighbor sums are obtained by multiplying the
order-normalized sums by the fixed constant $N$.

By \cref{lem:standard-mpnn-lipschitzness} and Lipschitzness of the fixed
ReLU readout, this network has a finite size-uniform Lipschitz bound $L$.
Divide its output by $\max\{1,L\}$
to obtain a uniformly $1$-Lipschitz network, without changing equality
comparisons between outputs. The architecture depends on the chosen pair;
this is not a single size-uniform recognition algorithm.
Writing $p=\Pr\{F(\widehat C^0)=1\}\geq5/6$ for the unscaled output,
declare the two inputs the same exactly when their independent
single-trial outputs agree. The worst-case success probability is
$\min\{p^2+(1-p)^2,1,p\}\geq13/18>2/3$.

\medskip\noindent\textbf{$2$-WL beats every fixed random-color budget.}
Use $G_{0,m},G_{1,m}$ from \cref{sec:cfi-draft}, with $m=2^h$ input
pairs and a perfect binary parity tree. The marker is the unique
degree-one vertex, so $2$-WL recognizes it. Let $Q_m$ be the graph
obtained by deleting the marker. The pair colors reveal whether a
replacement vertex is the marker, so $2$-WL can recover its refinement
on $Q_m$ by omitting this vertex from the updates.

$2$-WL distinguishes the attachments by their triangle counts: in $Q_m$,
each vertex of a $2C_3$ attachment lies
in three triangles, whereas each vertex of a $C_6$ attachment lies in
two. This separates the input endpoints $x_i^0,x_i^1$ consistently in
both graphs. Once a gadget's inputs are distinguished, its internal
vertices count their superscript-$0$ input neighbors, so its output endpoints see the
distinct count multisets $\{2,0\}$ and $\{1,1\}$.
Since $2$-WL refines ordinary color refinement, these
distinctions propagate through successive parity gadgets up the tree,
distinguishing the root endpoints $r^0,r^1$. The marker's attachment to
$r^0$ or $r^1$ then distinguishes the two graphs.

For fixed $q\geq2$ and $s\geq1$, choose $m=2^h$ large enough that
$s<\tfrac13(1-q^{-10})^{-m}$. Then
\cref{cor:cfi-trial-lower-draft} rules out statistical
$(q,s,1/3)$-distinguishability. For $q=1$, every block pair is erased,
so \cref{lem:cfi-block-erasure-draft,lem:cfi-root-hiding-draft} give
identical $1$-WL observations for every trial count.
The same indistinguishability bound extends to MPNNs.
\end{proof}

\section{Subgraph Isomorphism for Trees}
\label{sec:tree-subgraph}

We now adapt the color-coding algorithm of~\citet{alon1995color} to a random-color MPNN. Specifically, we show that for any rooted tree $T$ there exists a random-color MPNN that can decide at each vertex $v$ of a given graph $G$ if a copy of $T$ is rooted at $v$, using a number of layers, hidden dimension size, and number of trials independent of the size of $G$.
In this section, we use the MPNN architecture with coordinatewise max
aggregation and binary vertex outputs.

First we show that even graph-level detection is impossible for standard MPNNs.
\begin{proposition}
\label{prop:uncolored-tree-obstruction}
No MPNN with identical initial vertex features can decide whether an arbitrary
graph contains a simple path of length three.
\end{proposition}

\begin{proof}
The graphs $C_6$ and $2C_3$ are both $2$-regular on six vertices, so they are
$1$-WL indistinguishable and yield identical outputs under any such MPNN.
However, $C_6$ contains a simple path of length three, whereas $2C_3$ does not.
\end{proof}

Let $G=(V,E,\chi)$ be a colored graph with $\chi:V\to[k]$. A set of vertices $U\subseteq V$ is \emph{colorful} if all vertices in $U$ have different colors. For a graph $H$, a \emph{colorful subgraph embedding} is an injective edge-preserving map $\phi \colon V(H)\to V(G)$ such that $\im\phi$ is colorful. The embedded subgraph need not be an induced subgraph.
\par
Let $T=(V(T),E(T))$ be a rooted tree. For a vertex $t\in V(T)$, let $T_t$ denote the subtree of $T$ with root $t$.

\begin{lemma}[Colorful tree subgraph isomorphism]
\label{lem:colorful-tree-mpnn}
    Let $T$ be a rooted tree with $k\geq1$ vertices. There is an MPNN that can determine, for each vertex $v$ of a colored graph $G=(V,E,\chi)$ with $\chi \colon V\to[k]$, whether there is a colorful subgraph embedding of $T$ mapping its root to $v$, with hidden dimension $\exp(O(k))$ and $O(k)$ message-passing layers.
\end{lemma}
\begin{proof}
    We use a dynamic program to find colorful subgraph embeddings of $T$. We first describe the dynamic program and prove its correctness. We then show that this dynamic program can be implemented on an MPNN.
    \par
    Let $r$ be the root of $T$.
    The dynamic program iteratively adds edges to the graph using a post-order traversal from the root $r$; whenever it visits a nonroot node, it adds the edge connecting that node to its parent. At iteration $i\in\{0,\ldots,k-1\}$, let $T^{(i)}$ denote the spanning subgraph of $T$ with all edges added so far, and let $T^{(i)}_t$ denote the subtree of $T^{(i)}$ at node $t\in V(T)$, following edges away from $r$.
    We use a dynamic programming table $D$ with entries corresponding to $V(G)\times V(T)\times\{0,\ldots,k-1\}\times\powerset([k])$, where $\powerset([k])$ denotes the power set of $[k]$. The dynamic programming table will store $\True$/$\False$ corresponding to
    \begin{equation}
    \begin{gathered}
        D[v,\,t,\,i,\,S] =\\
        \begin{aligned}
        &\text{There is a colorful subgraph embedding $\phi$}\\
        &\text{from the subtree $T^{(i)}_t$ to $G$}\\
        &\text{such that $\phi(t)=v$ and $\chi(\im\phi) = S$}.
        \end{aligned}
    \end{gathered}
    \end{equation}
    Clearly, there is a colorful embedding from $T$ to $G$ mapping $r$ to $v$ if and only if $D[v,r,k-1,[k]]$ is $\True$.
    \par
    The base case of the dynamic program consists of singleton trees. For all vertices $v\in V(G)$ and nodes $t\in V(T)$, we set the entry $D[v,t,0,\{\chi(v)\}]=\True$ as $t$ is the only node in the subtree $T^{(0)}_t$. All other entries at iteration zero are $\False$.
    \par
    Let $\{a,p\}$ be the edge added at iteration $i$, where $a$ is the child and $p$ is the parent. We only need to update the dynamic programming table entries corresponding to $p$, since the subtrees at all other nodes remain the same.
    \begin{equation}
    \begin{gathered}
    D[v,\,p,\,i,\,S] =\\
    \begin{cases}
        \True & \begin{aligned}[t]
            &\text{if }D[v,p,i-1,S']=\True\\
            &\text{and }D[u,a,i-1,S'']=\True \\
            &\text{for some neighbor }u\text{ of }v \\
            &\text{and some decomposition}\\
            &S=S'\sqcup S''.
        \end{aligned} \\
        \False & \text{otherwise,}
    \end{cases}
    \end{gathered}
    \end{equation}
    where $\sqcup$ denotes disjoint union. For $t\neq p$, set $D[v,t,i,S]=D[v,t,i-1,S]$.
    As a sketch of why this is correct, observe that any colorful embedding $\phi$ of $T^{(i)}_p$ with $\phi(p)=v$ will have to map $a$ to a neighbor $u$ of $v$, so it can be decomposed as embeddings of $T^{(i-1)}_p$ and $T^{(i-1)}_a$. If we can guarantee embeddings of these trees rooted at $v$ and $u$ exist, and their images have disjoint color sets, then this implies they can be combined to an embedding of $T^{(i)}_p$ rooted at $v$ (the images won't share nodes as they have different colors). Induction over the added edges proves correctness.
    \par
    We now argue that message passing can achieve this. We initialize node features as one-hot vectors in $\R^k$ corresponding to node colors $\chi(v)$. Our hidden vectors between iterations will be of size $k\times 2^k$; after iteration $ i$, the node feature $\vec{h}^{(i)}_v$ will correspond to the dynamic programming table entry $D[v,\cdot,i,\cdot]$. Accordingly, we index our node features by $\vec{h}_v^{(i)}[t,S]$ for $t\in V(T)$ and $S\subseteq[k]$. Our MPNN will then have $\vec{h}_v^{(i)}[t,S]=1$ if $D[v,t,i,S]=\True$ and $\vec{h}_v^{(i)}[t,S]=0$ otherwise. Here $i$ indexes dynamic-programming iterations, each implemented by a constant number of MPNN layers.
    \par
    Initially, we apply a linear encoding layer that sets $\vec{h}_v^{(0)}[t,\{\chi(v)\}]=1$ for all nodes $t\in V(T)$, with all other entries zero. Our message-passing updates will be of the form
    \begin{equation}
    \begin{aligned}
        \vec{m}_v^{(i)}&=\max_{u\in N_G(v)}\vec{h}_u^{(i-1)},\\
        \vec{h}_v^{(i)}&=f_i\left(\vec{h}_v^{(i-1)},\vec{m}_v^{(i)}\right),
    \end{aligned}
    \end{equation}
    where the maximum is coordinatewise, and $f_i$ is a ReLU MLP. At isolated vertices, set $\vec{m}_v^{(i)}=0$.
    \par
    We will repeatedly make use of the fact that a two-layer ReLU MLP can implement a threshold function $f(x)=\relu(x)-\relu(x-1)$ with the property that $f(x)=1$ for $x\geq 1$ and $f(x)=0$ for $x\leq 0$.
    \par
    Each iteration will then correspond to adding an edge to the tree. Max aggregation gives $\vec{m}_v^{(i)}[a,S'']=1$ iff $D[u,a,i-1,S'']=\True$ for some neighbor $u$ of $v$. For each $S\subseteq[k]$, we compute the update
    \begin{equation}
    \begin{multlined}
    f_i\left(\vec{h}_v^{(i-1)},\vec{m}_v^{(i)}\right)[p,S]\\
    =f\!\Biggl(\sum_{S=S'\sqcup S''}
        \relu\!\Bigl(\vec{h}_v^{(i-1)}[p,S']\\
        {}+\vec{m}_v^{(i)}[a,S'']-1\Bigr)\Biggr).
    \end{multlined}
    \end{equation}
    Thus, this entry is 1 iff there exists $S'\sqcup S''=S$ such that $D[v,p,i-1,S']$ and $D[u,a,i-1,S'']$ are true for some neighbor $u$ of $v$. This can be implemented by a three-layer ReLU MLP with hidden dimension $O(k4^k)$. The first layer takes the sum of $\vec{h}_v^{(i-1)}[p,S']$ and $\vec{m}_v^{(i)}[a,S'']$, subtracts 1, and applies ReLU, storing it in an intermediate feature $[p,S',S'']$. This feature is 1 iff both inputs are 1. The final two layers sum up over all choices of $S'\sqcup S''=S$ and take a threshold. These two layers set neighbor weights to zero, and copy all entries for $t\neq p$ unchanged through the self-update.
    \par
    After $k-1$ iterations, the vertex output $\vec{h}_v^{(k-1)}[r,[k]]$ is 1 iff a colorful copy of $T$ is rooted at $v$, and 0 otherwise. Including initialization, the construction uses $O(k)$ layers and hidden dimension $O(k4^k)=\exp(O(k))$, with parameters depending only on $T$.
\end{proof}

\begin{theorem}[Random-color MPNN for rooted tree detection]
\label{thm:random-color-tree-subiso}
Let $T$ be a rooted tree with $k\geq1$ vertices. There is a random-color MPNN using $k$
initial colors, hidden dimension $\exp(O(k))$, and $O(k)$ message-passing layers,
using max aggregation and binary vertex outputs, with the following guarantee.
For every graph $G$ and vertex $v\in V(G)$, one random-color
trial accepts at $v$ only if $G$ contains a copy of $T$ rooted at $v$. If such
a copy exists, then one trial accepts at $v$ with probability at least
\begin{equation}
        \frac{k!}{k^k}.
\end{equation}
Consequently, for any $\alpha\in(0,1)$, running
\begin{equation}
        s=\left\lceil \frac{k^k}{k!}\log(1/\alpha)\right\rceil
\end{equation}
independent trials and accepting at $v$ when its average binary output is positive
(equivalently, when any trial accepts at $v$) decides whether $G$
contains a copy of $T$ rooted at $v$ with false-positive probability $0$ and
false-negative probability at most $\alpha$ for each $v$.
\end{theorem}

\begin{proof}
Use the colored MPNN from Lemma~\ref{lem:colorful-tree-mpnn} with color set
$[k]$. In one random-color trial, each vertex of $G$ receives an independent
uniform color from $[k]$, and the colored MPNN checks at each $v$ for a colorful copy of
$T$ rooted at $v$. If the trial accepts at $v$, then the detected colorful copy is in particular an
uncolored copy of $T$ rooted at $v$, so there are no false positives.

Fix $v$ such that $G$ contains a copy of $T$ rooted at $v$, and fix one such copy. The
$k$ vertices of this copy receive distinct colors with probability $k!/k^k$.
On this event, the fixed copy is colorful, so the colored MPNN accepts at $v$. Thus one
trial accepts at $v$ with probability at least $k!/k^k$. For $s$ independent trials,
the false-negative probability at $v$ is at most
\begin{equation}
        \left(1-\frac{k!}{k^k}\right)^s
        \le
        \exp\left(-s\frac{k!}{k^k}\right)
        \le
        \alpha .\qedhere
\end{equation}
\end{proof}

Participation in any role follows by running all choices of the pattern root in parallel and taking the maximum at each vertex, with the same asymptotic resource bounds.

\begin{corollary}[Tree subgraph detection]
\label{cor:tree-subgraph-detection}
For every tree $T$ and $\alpha\in(0,1)$, max-over-nodes readout in
\cref{thm:random-color-tree-subiso} detects whether a graph contains a
subgraph isomorphic to $T$, with the same resource and trial bounds,
no false positives, and false-negative probability at most $\alpha$.
\end{corollary}

\begin{proof}
Root $T$ arbitrarily and take the maximum over vertex outputs in each trial.
A positive graph contains a fixed copy rooted at some vertex, whose miss
probability is at most $\alpha$; a negative graph has no positive vertex
outputs. For the empty graph, set the readout to zero.
\end{proof}

\begin{corollary}[Detecting simple paths]
\label{cor:simple-paths}
For every $k\ge 1$, there is a random-color MPNN that detects whether a graph
contains a simple path of length $k$, meaning a path with $k$ edges and $k+1$
vertices, using $k+1$ initial colors, hidden dimension $\exp(O(k))$, and $O(k)$
message-passing layers. One trial has no false positives and accepts with
probability at least
\begin{equation}
        \frac{(k+1)!}{(k+1)^{k+1}}
\end{equation}
whenever such a path exists. Hence, for any $\alpha\in(0,1)$,
\begin{equation}
        s=\left\lceil
        \frac{(k+1)^{k+1}}{(k+1)!}\log(1/\alpha)
        \right\rceil
\end{equation}
independent trials give false-positive probability $0$ and false-negative
probability at most $\alpha$.
\end{corollary}

\begin{proof}
Apply Corollary~\ref{cor:tree-subgraph-detection} to the path tree on $k+1$
vertices.
\end{proof}

\section{A Local MST Edge Task}

We next consider a local task for computing MST edges.  Let $G=(V,E)$ be a simple graph with edge weights
$w\colon E\to\Lambda$, where $\Lambda$ is a fixed finite ordered alphabet. We
aim to identify edges that uniquely minimize the weight across a cut $S\subset V$
with $|S|<k$. For $k\ge |V|$, this is
exactly the set of edges occurring in every MST (or every minimum spanning
forest if $G$ is disconnected).

An edge belongs to every MST if and only if it is the unique minimum-weight
edge crossing some cut. The cut, however, may have many vertices
on both sides. We therefore consider a local version of this task, identifying
only edges whose membership in every MST can be established by a cut with
fewer than $k$ vertices on one side.

Fix an integer $k\ge 2$, and let $\delta_G(S)$ denote the edges with exactly
one endpoint in $S$. This constrained edge set is
\begin{equation}
    \begin{gathered}
    A_k(G,w)\coloneq\\
    \left\{e\in E\colon\;
    \begin{aligned}
    &\exists S\subset V,\ |S|<k,\ e\in\delta_G(S),\\
    &w(e)<w(f)\ \text{for all }f\in\delta_G(S)\setminus\{e\}
    \end{aligned}
    \right\}.
    \end{gathered}
\end{equation}

To identify edges in $A_k$, we look for a small component separated from the
opposite endpoint after removing the candidate edge and all heavier edges.
For $e=uv$, define
\begin{equation}
    D_e\coloneq (V,\{f\in E\setminus\{e\}\colon w(f)\le w(e)\}),
\end{equation}
and let $C_e(v)$ denote the connected component of $v$ in $D_e$.
If the endpoints are separated, $e$ is uniquely lightest across the boundary
of either component. A component with fewer than $k$ vertices therefore gives
the required cut. Conversely, any qualifying set $S$ contains the component
of its endpoint. Hence,
\begin{equation}
    e\in A_k(G,w)
    \quad\Longleftrightarrow\quad
    C_e(u)\neq C_e(v)\ \text{and}\ \min\{|C_e(u)|,|C_e(v)|\}<k.
\end{equation}

We use the edge-feature extension of message passing: messages may depend
on the sender's state and a fixed encoding of the edge weight. Coordinatewise
max aggregation and updates implement the filters below using the initial
colors, which are retained in the node states. The edge decoder is a shared
ReLU MLP of the form $d(\vec{h}_u+\vec{h}_v,w(uv))$, where $\vec{h}_u,\vec{h}_v$ are the final
endpoint embeddings. An empty maximum is taken to be zero.

The next proposition shows that random colors allow a random color MPNN to compute this task for constant $k$. 

\begin{proposition}[Finite-color local Kruskal edge detection]
\label{prop:local-kruskal}
Fix an integer $k\ge 2$, a finite ordered alphabet $\Lambda$, and
$\varepsilon,\alpha\in(0,1)$. There is a finite-color edge-output MPNN, using a
constant number of independent random-color trials, such that for every
weighted graph $G$ with weights in $\Lambda$, the random output
$\widehat A\subseteq E(G)$ satisfies
\begin{equation}
        \Pr\left\{
        \begin{gathered}
        |\widehat A\cap A_k|\ge (1-\varepsilon)|A_k|
        \\
        \text{ and }\\
        |\widehat A\setminus A_k|
        \le \varepsilon |E\setminus A_k|
        \end{gathered}
        \right\}
        \ge 1-\alpha .
\end{equation}
The number of colors, number of trials, and depth depend only on
$k,\varepsilon,\alpha$, and $\Lambda$, not on $|V(G)|$.
\end{proposition}

\begin{proof}
Fix $q\ge 2k$. In one random trial, assign each vertex an independent uniform color
\begin{equation}
        \chi(v)\in [q]\coloneq \{1,\ldots,q\}.
\end{equation}
For each $\lambda\in\Lambda$ and $1\le a\le b\le q$, write
\begin{equation}
    \begin{gathered}
    D_{\lambda,a,b}\coloneq\\
    \left(V,\left\{xy\in E\colon\;
    \begin{aligned}
    &w(xy)<\lambda\ \text{or}\\
    &w(xy)=\lambda,\ \{\chi(x),\chi(y)\}\neq\{a,b\}
    \end{aligned}
    \right\}\right).
    \end{gathered}
\end{equation}
Thus, $D_{\lambda,a,b}$ retains edges of weight at most $\lambda$, except
weight-$\lambda$ edges with endpoint colors $\{a,b\}$. An edge $e=uv$ selects
the channel $\lambda=w(e)$ and
$\{a,b\}=\{\chi(u),\chi(v)\}$. This channel deletes $e$, but may also delete
other edges with the same weight and endpoint colors.

Run $k-1$ message-passing rounds in each channel, maintaining the finite state
\begin{equation}
        S_t^{\lambda,a,b}(v)\subseteq [q],
\end{equation}
the set of colors seen within distance at most $t$ from $v$ in
$D_{\lambda,a,b}$. Initialize
\begin{equation}
        S_0^{\lambda,a,b}(v)=\{\chi(v)\},
\end{equation}
and update by
\begin{equation}
        S_{t+1}^{\lambda,a,b}(v)
        =
        S_t^{\lambda,a,b}(v)\cup
        \bigcup_{x\in N_{D_{\lambda,a,b}}(v)}
        S_t^{\lambda,a,b}(x).
\end{equation}
We encode each color subset by its indicator vector. The union update is then
a coordinatewise OR, which we implement using max aggregation and fixed ReLU
updates. All channels are stored in the same node state. To implement the
filter, group incoming messages by the sender's initial color and the edge
weight, take coordinatewise maxima within each group, and select the groups
allowed by the receiver's initial color and the channel. These selections
are finite Boolean operations and can be implemented by fixed ReLU updates.
The number of channels depends only on $q$ and $\Lambda$, not on $|V(G)|$.

After $k-1$ rounds, set
\begin{equation}
        \begin{aligned}
        \Sigma_{\lambda,a,b}(v)&\coloneq S_{k-1}^{\lambda,a,b}(v),\\
        M_{\lambda,a,b}(v)&\coloneq
        \mathbf{1}\{|\Sigma_{\lambda,a,b}(v)|<k\}.
        \end{aligned}
\end{equation}
For an edge $e=uv$, the desired decision rule uses its selected channel
and accepts in one trial if
\begin{equation}
        \bigl(M_{\lambda,a,b}(u)\vee M_{\lambda,a,b}(v)\bigr)
        \quad\text{and}\quad
        \Sigma_{\lambda,a,b}(u)\neq\Sigma_{\lambda,a,b}(v).
\end{equation}
To implement this rule from the sum of endpoint embeddings, retain the
initial one-hot color and all $\Sigma$ indicators and $M$ bits in each
final node state. The summed color coordinates determine the unordered
endpoint colors, including the case of equal colors, and hence select
the channel together with $w(e)$. The summed $M$ coordinate is positive
exactly when at least one endpoint has $M=1$. The two $\Sigma$ sets differ
exactly when some summed indicator coordinate equals one. On inputs in
$\{0,1,2\}$, the ReLU function
\[
    r(x)=\relu(x)-2\relu(x-1)+\relu(x-2)
\]
is the indicator of this event. Taking the Boolean OR over these coordinates
and the conjunction with the $M$ condition gives the decision rule.
Channel selection and these Boolean operations are exactly implementable
by a fixed ReLU MLP applied to $\vec{h}_u+\vec{h}_v$ and the encoded edge weight.
Its width and depth depend only on $q,k,\Lambda$.
The final network accepts an edge when the average of its $s$ independent
binary trial outputs is positive, equivalently when any trial accepts it.

Fix an edge $e=uv$. For each endpoint $x\in\{u,v\}$, choose $W_x=C_e(x)$ if
$|C_e(x)|<k$; otherwise, choose a connected $k$-vertex set $W_x\subseteq C_e(x)$
containing $x$. Set $W=W_u\cup W_v$, so $|W|\le 2k$. These sets depend on
$G,w,e$, not on the random colors.

Suppose that the colors on $W$ are all distinct. The selected channel is a
subgraph of $D_e$, and it retains every edge of $D_e$ with both endpoints in
$W$: the only vertices of $W$ with colors $\chi(u)$ and $\chi(v)$ are $u$ and
$v$, and $e$ is already absent from $D_e$. Here we use that $G$ is simple.
If $|C_e(x)|<k$, the channel therefore preserves the whole component, so
$\Sigma_{\lambda,a,b}(x)$ is its color set and $M_{\lambda,a,b}(x)=1$.
If $|C_e(x)|\ge k$, the connected set $W_x$ lies within distance at most
$k-1$ from $x$ in the channel, so at least $k$ distinct colors are seen and
$M_{\lambda,a,b}(x)=0$.

Consequently, if both endpoint components are large, the decoder rejects.
If they are the same small component, their color sets agree, and the
decoder again rejects. If they are different and at least one is small,
their color sets differ: two small components have disjoint color sets,
while a large component contributes at least $k$ colors. Thus the decoder
is correct whenever the colors on $W$ are distinct. This event has probability
at least
\begin{equation}
        p_*(q)\coloneq \frac{(q)_{2k}}{q^{2k}},
\end{equation}
where $(q)_j\coloneq q(q-1)\cdots(q-j+1)$ denotes the falling factorial.
Therefore, after $s$ independent trials, each target edge is missed with
probability at most $(1-p_*(q))^s$, and each non-target edge is falsely accepted
with probability at most $s(1-p_*(q))$.

Set $\eta\coloneq \alpha\varepsilon/2$. Choose $q$ large enough that
$p_*(q)\ge 1/2$, then choose $s$ so that
$(1-p_*(q))^s\le \eta$, and finally increase $q$, if necessary, so that
\begin{equation}
        s(1-p_*(q))\le \eta .
\end{equation}
Then each target edge is missed with probability at most $\eta$, and each
non-target edge is falsely accepted with probability at most $\eta$. Markov's
inequality applied to the two error counts yields failure probability at most
$\alpha$. If either edge class is empty, its error count is identically zero.
\end{proof}

\begin{theorem}[Deterministic obstruction for the local Kruskal task]
\label{thm:deterministic-kruskal-obstruction}
Let $\Lambda=\{0,1\}$ and let $k>5$. There is a connected weighted graph $(H,w)$
on ten vertices, with weights in $\Lambda$, such that no deterministic
uncolored MPNN with identical initial vertex features and an edge-weight-aware
symmetric edge decoder computes $A_k(H,w)$.
\end{theorem}

\begin{proof}
Take two disjoint copies of $K_4$, subdivide one edge in each copy, and join
the two new subdivision vertices by an edge $e_*$
(\cref{fig:uncolored-bridge-obstruction}). Give every edge weight $1$.

Every vertex of $H$ has degree three. Since the initial vertex features are
identical and every edge has the same weight, induction on the message-passing
layers shows that all vertex states remain identical. Thus, any symmetric
edge decoder receives the same endpoint states and edge weight for every edge.

However, $e_*$ is a bridge whose removal leaves two components of size $5<k$,
so $e_*\in A_k(H,w)$. Every other edge lies on a cycle of weight-$1$ edges,
so none uniquely minimizes any cut. Consequently, no deterministic uncolored
MPNN with an edge-weight-aware symmetric decoder computes $A_k$ on this graph.
\end{proof}

\begin{center}
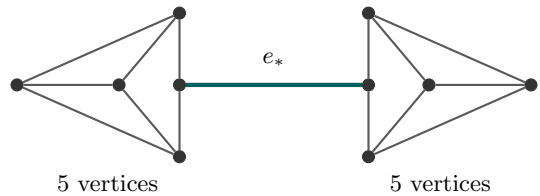

\begin{minipage}{\linewidth}
\centering
\begin{tikzpicture}[
    vertex/.style={circle, fill=black!80, draw=black!80,
        inner sep=0pt, minimum size=4.5pt},
    edge/.style={draw=black!65, line width=0.8pt},
    font=\small
]
    \foreach \side/\sign in {L/-1,R/1} {
        \begin{scope}[xscale=\sign]
            \node[vertex] (\side-a) at (3.4,0) {};
            \node[vertex] (\side-b) at (1.25,0.95) {};
            \node[vertex] (\side-c) at (1.25,-0.95) {};
            \node[vertex] (\side-d) at (2.05,0) {};
            \node[vertex] (\side-s) at (1.25,0) {};
            \draw[edge] (\side-a) -- (\side-b) -- (\side-s)
                -- (\side-c) -- (\side-a);
            \draw[edge] (\side-d) -- (\side-a);
            \draw[edge] (\side-d) -- (\side-b);
            \draw[edge] (\side-d) -- (\side-c);
            \node at (2.2,-1.3) {5 vertices};
        \end{scope}
    }
    \draw[draw=teal!75!black, line width=1.6pt]
        (L-s) -- node[above=3pt, text=black] {$e_*$} (R-s);
\end{tikzpicture}
\captionof{figure}{Counterexample in \cref{thm:deterministic-kruskal-obstruction}.
Every vertex has degree three and every edge has weight $1$.
The highlighted edge $e_*$ is the only bridge; deleting it leaves two
five-vertex components. Thus $A_k(H,w)=\{e_*\}$ for $k>5$.}
\label{fig:uncolored-bridge-obstruction}
\end{minipage}
\end{center}

\section{Expressivity Hierarchies for Colors and Trials}
\label{sec:random-color-hierarchies}

\Cref{cor:cfi-concentrated-randomness-draft} shows that
colors and trials can have very different quantitative effects. We now
show that increasing either resource strictly increases expressivity, i.e., the ability to distinguish non-isomorphic graphs with high probability. To do that, we use statistical $(q,s,\alpha)$-distinguishability by random-color
WL as defined in \cref{def:rcwl-draft}, i.e., a test receives $s$ independent
trials per graph with $q$ uniform colors and has failure probability at
most $\alpha$ for each of the four input pairs. 

Specifically, we show the two hierarchy results regarding the number of colors and trials, respectively:

\begin{theorem}[Strict hierarchy in the number of trials]
\label{thm:strict-hierarchy}
For every integer $s\geq2$, there are non-isomorphic finite graphs
$G_s,H_s$ that are $(2,s,1/3)$-distinguishable by
random-color WL, but are not statistically
$(2,s-1,1/3)$-distinguishable.
\end{theorem}

Ignoring the last trial gives the corresponding non-strict inclusion.
The theorem shows that two colors suffice, which is optimal. 
For colors, we compare budgets of at most $q$ colors, i.e., the test may
choose a uniform palette of size $r\leq q$, fixed across all its trials.

\begin{theorem}[Strict hierarchy in the number of colors]
\label{thm:color-hierarchy}
For every pair of integers $q,s\geq1$, there are non-isomorphic finite
graphs $G_{q,s},H_{q,s}$ that are statistically
$(q+1,s,1/3)$-distinguishable by random-color WL, but are not
statistically $(r,s,1/3)$-distinguishable for any $1\leq r\leq q$.
\end{theorem}

Hence, for every fixed number of trials, increasing the allowed palette
size strictly enlarges the set of distinguishable graph pairs. We prove
both theorems using one graph class. In each trial, the observation
either conceals the graph identity or determines it exactly. We then
carefully choose the construction parameters so that the optimal failure
probability crosses $1/3$ when one trial or one color is added.

\subsection{Construction of the graph class}

Let $\ell\geq7$ be prime, and create the following disjoint union of cycles
\begin{equation}
    A_\ell\coloneq C_\ell,
    \qquad B_\ell\coloneq C_3\sqcup C_{\ell-3}.
\end{equation}
Both graphs are $2$-regular on $\ell$ vertices. The following results show that $1$-WL distinguishes them unless both graphs are monochromatic.

\begin{lemma}
\label{lem:block}
Two colored copies of $A_\ell$ and $B_\ell$ are indistinguishable
by $1$-WL if and only if both are monochromatic in the same color.
If a root is added to each graph and joined to every block vertex,
the resulting colored graphs are indistinguishable if and only if
the blocks are monochromatic in the same color and the two roots
have the same initial color.
\end{lemma}

\begin{proof}
Suppose that the colored blocks are indistinguishable, and run color
refinement on their disjoint union. For each stable color $c$, let
$a_c$ be its multiplicity in $A_\ell$ and let $b_c$ be its
multiplicity in the triangular component of $B_\ell$. Every stable
color occurring in either graph occurs in both. Let $m_{cd}$ be the
number of neighbors of color $d$ of a vertex of color $c$.

The quotient on stable colors is connected because $A_\ell$ is
connected. A color occurring in the triangle brings all its neighboring
colors into the triangle, so $b_c>0$ for every stable color $c$.
For distinct adjacent colors, counting edges gives
\begin{equation}
\begin{aligned}
    a_c m_{cd}&=a_d m_{dc},\\
    b_c m_{cd}&=b_d m_{dc}.
\end{aligned}
\end{equation}
Hence, the fraction $b_c/a_c$ is constant on the connected quotient. Summing the
multiplicities gives $b_c=3a_c/\ell$. Since $\ell$ is prime and
$\ell\geq7$, every positive $a_c$ is divisible by $\ell$.
As $\sum_c a_c=\ell$, there is only one stable color. Both blocks
are therefore monochromatic in the same initial color. Conversely,
equally colored monochromatic $2$-regular graphs of the same order
have identical WL colors.

In the rooted graphs, the roots have degree $\ell$ and the block
vertices have degree three. Under indistinguishability, the two roots
therefore have the same stable color. Removing their common contribution
from the block vertices' neighbor multisets leaves the same argument.
The converse follows by a standard induction over the refinement rounds.
\end{proof}

We next use these blocks to distinguish the endpoints of a parity
gadget. Fix $k\geq3$ input pairs $X_i=(x_i^0,x_i^1)$, $i\in[k]$,
and a root pair $R=(r^0,r^1)$. Add an edge within every pair.
For each input $i$, attach $A_\ell$ to $x_i^0$ and $B_\ell$ to
$x_i^1$, joining each endpoint to every vertex of its block.
For $\vec a\in\{0,1\}^k$, write
\begin{equation*}
    \operatorname{par}(\vec a)
    \coloneq a_1\oplus\cdots\oplus a_k.
\end{equation*}
Create one internal vertex $g_{\vec a}$ for each such vector, with
neighborhood
\begin{equation*}
    N(g_{\vec a})
    =\{x_1^{a_1},\ldots,x_k^{a_k},
        r^{\operatorname{par}(\vec a)}\}.
\end{equation*}
Finally, define $H_{b,\ell,k}$, for $b\in\{0,1\}$, by attaching
one marker leaf to $r^b$. Note that each graph has
$2k\ell+2k+2^k+3$ vertices. For fixed $\ell,k$, set
$H_b \coloneq H_{b,\ell,k}$.

The following results show that $1$-WL can recover the structural role of $H_b$.

\begin{lemma}[Structural roles]
\label{lem:roles}
The complete WL observation of $H_b$ identifies the input endpoints,
block vertices, internal vertices, root endpoints, and marker leaf.
For every $t\geq0$, an input endpoint's structure after $t+2$ rounds
determines the $t$-round WL observation of its attached block in
isolation. The endpoint also identifies its pair partner through their
pair edge. These statements hold for every initial coloring.
\end{lemma}

\begin{proof}
We first look at the vertex degrees of the different structural roles in the graph $H_b$:
\begin{equation*}
\begin{array}{c|c}
\text{Role}&\text{Degree}\\ \hline
\text{marker leaf}&1\\
\text{block vertex}&3\\
\text{internal vertex}&k+1\\
\text{unmarked root}&2^{k-1}+1\\
\text{marked root}&2^{k-1}+2\\
\text{input endpoint}&\ell+2^{k-1}+1.
\end{array}
\end{equation*}
The degrees are distinct for $k\geq3$ and $\ell\geq7$, so one
1-$WL$ round can identify such structures. Each input endpoint has exactly one
input-endpoint neighbor (i.e., the partner), and, by construction, its block neighbors
are precisely its neighbors of degree $3$.

For a colored graph $(H_b,\chi)$, abbreviate its full structures
by $\pi_t(v)=\pi_t^{H_b,\chi}(v)$. Let $\beta_t(v)$ be the
structure obtained by running $1$-WL inside the block containing the vertex $v$.
We prove by induction that $\pi_{t+1}(v)$ determines $\beta_t(v)$
for every block vertex $v$. The base case holds because
$\beta_0(v)=\chi(v)$. The structure $\pi_{t+2}(v)$ records its own
$\pi_{t+1}$ structure and those of its neighbors. Because
$t+1\geq1$, each neighbor pattern determines the neighbor's degree.
Retaining exactly the neighbors of degree $3$ and applying the
induction hypothesis implies
\begin{equation*}
\begin{aligned}
    \beta_{t+1}(v)=\bigl(&\beta_t(v),\\
    &\{\!\{\beta_t(u) \colon u \in N(v),\ \deg(u)=3\}\!\}\bigr).
\end{aligned}
\end{equation*}
These are exactly the neighbors inside the block. An input endpoint's
$\pi_{t+2}$ pattern then gives the multiset of $\beta_t$ structures
over every vertex of its attached block.
\end{proof}

We now show that swapping the endpoints propagates through the gadget. 

\begin{lemma}[Parity extension]
\label{lem:parity-extension}
Consider the parity gadget before adding the blocks and marker.
Swapping the endpoints of $X_i$ when $s_i=1$ and of $R$ when
$s_R=1$ extends to an automorphism of the gadget if and only if
\begin{equation*}
    s_R=s_1\oplus\cdots\oplus s_k.
\end{equation*}
\end{lemma}

\begin{proof}
Let $\vec s=(s_1,\ldots,s_k)$; the input neighbors force
$g_{\vec a}$ to map to $g_{\vec a\oplus\vec s}$. Its root edge
is preserved exactly when
\begin{equation*}
\begin{aligned}
    \operatorname{par}(\vec a)\oplus s_R
    &=\operatorname{par}(\vec a\oplus\vec s)\\
    &=\operatorname{par}(\vec a)\oplus\operatorname{par}(\vec s).
\end{aligned}
\end{equation*}
This is equivalent to the stated condition. When it holds, the map
preserves every edge. Permuting input pairs also preserves the
gadget and does not change this condition.
\end{proof}

In addition, we can show that the  graphs $H_{0,\ell,k}$ and $H_{1,\ell,k}$ are non-isomorphic.

\begin{lemma}[Non-isomorphism]
\label{lem:noniso}
For every prime $\ell\geq7$ and $k\geq3$, the graphs
$H_{0,\ell,k}$ and $H_{1,\ell,k}$ are not isomorphic.
\end{lemma}

\begin{proof}
An isomorphism must preserve the roles in \cref{lem:roles} and
map the marked root $r^0$ in $H_0$ to $r^1$ in $H_1$.
Moreover, an isomorphism also maps input-pair edges to input-pair edges and preserves
the subgraph induced by each endpoint's degree-three neighbors.
Since $A_\ell$ is connected and $B_\ell$ is disconnected, it
cannot exchange the two endpoint types. It may therefore permute
input pairs but cannot swap their endpoints. By
\cref{lem:parity-extension}, it cannot swap the root endpoints
either, leading to a contradiction.
\end{proof}

\subsection{Erasure and the trial law}

In one trial, assign every vertex an independent uniform color from
$[q]$. Hence, a block is monochromatic with probability $q^{1-\ell}$.
For input $i$, let $E_i$ be the event that both attached blocks are
monochromatic while their colors need not agree. By the above, these events are
independent and have probability
\begin{equation}
    p_{\ell,q}\coloneq\Pr(E_i)=q^{2-2\ell}.
\end{equation}
Let
\begin{equation}
    E\coloneq\bigcup_{i=1}^k E_i,
    \qquad D\coloneq E^c,
\end{equation}
and
\begin{equation}
    \eta_{\ell,k,q}\coloneq\Pr(D)
    =(1-p_{\ell,q})^k.
\end{equation}
We call input $i$ erased when $E_i$ occurs. For fixed $\ell,k,q$,
write $P_b=P_{H_b}^{\mathrm{WL},q}$ for the probability law of
$Z_{\mathrm{WL},q}(H_b)$.

\begin{lemma}
\label{lem:decomposition}
For $q\geq2$, the conditional WL observation laws on $E$ agree,
whereas the conditional laws on $D$ have disjoint supports.
Moreover, the observation also determines whether $E$ or $D$ occurred.
Consequently, there are mutually singular probability laws
$Q,Q_0,Q_1$ such that, for $b\in\{0,1\}$,
\begin{equation}
    P_b=(1-\eta_{\ell,k,q})Q+\eta_{\ell,k,q}Q_b.
\end{equation}
For $q=1$, the two trial laws agree and $\eta_{\ell,k,1}=0$.
\end{lemma}

\begin{proof}
We first show that one erased input hides the graph identity.
Fix a coloring of $H_0$ in $E$, and choose the least erased index
$i$. On the same vertex sets, we replace the block $A_\ell$ at
$x_i^0$ by $B_\ell$ and the block $B_\ell$ at $x_i^1$ by
$A_\ell$. Therefore, each replaced block is monochromatic, has $\ell$ vertices,
and is $2$-regular internally. At every refinement round, all
vertices of such a block have the same structure, determined by two
copies of its preceding structure and the structure of its endpoint.
A straightforward structural induction therefore shows that these replacements preserve the full WL observation.

The resulting graph is isomorphic to $H_1$, i.e.,  swap the endpoints of
$X_i$ and $R$, send $g_{\vec a}$ to
$g_{\vec a\oplus\vec e_i}$, exchange the two blocks at input
$i$, and fix all other inputs and blocks. Here $\vec e_i$ is
the vector with its only nonzero coordinate at $i$. We now map
every vertex color along this isomorphism, including the colors
on the internal vertices, endpoints, and marker. We then choose the block
bijections in inverse pairs. This defines an involution between
the two copies of $E$, i.e., it preserves every event $E_j$ and hence
the least erased index. Also, all colorings have the same probability,
so the map preserves the conditional uniform law on $E$.
The coupled observations agree. The least erased index is used
only to define this map on labeled colorings, i.e., the observer need
not recover it.

Next, because the degree-three neighbor patterns retain
their initial colors, two WL rounds determine whether an endpoint's own block is
monochromatic. In the first round, the same information
about its unique input-endpoint neighbor. Thus the complete
observation determines the existence of an erased input and
hence the event $E$.

It remains to recover the graph identity on $D$. The color
partition on the disjoint union of two colored $\ell$-vertex
blocks stabilizes within $2\ell$ rounds. By \cref{lem:block},
the $2\ell$-round observation therefore identifies the type of
any nonmonochromatic block. By \cref{lem:roles}, an endpoint's
full structure after $2\ell+2$ rounds determines this observation.
On $D$, every input has at least one nonmonochromatic block.
Its type determines the endpoint's superscript. That is, if the other
block is monochromatic, the pair edge gives its endpoint the
complementary superscript. Thus, after $2\ell+3$ rounds, every
input endpoint's pattern determines its superscript, without
recovering the input index.

After one further round, $g_{\vec a}$ determines the number of
its input neighbors with superscript one and hence
$\operatorname{par}(\vec a)$. All internal neighbors of $r^b$
have parity $b$. Since degree identifies the marked root,
its pattern after $2\ell+5$ rounds determines $b$.
The two conditional supports on $D$ are therefore disjoint.
Since $E$ is observable, they are also disjoint from the common
conditional support on $E$. This gives the decomposition.
When $q=1$, every input is erased, and the same coupling proves
equality of the deterministic observations.
\end{proof}

\subsection{Optimal failure probability}

The decomposition makes the optimal error depend only on the
probability that both batches reveal their graph identities.

\begin{proposition}
\label{prop:risk}
Suppose that, for $b\in\{0,1\}$,
\begin{equation}
    P_b=(1-\eta)Q+\eta Q_b,
    \qquad \eta\in[0,1],
\end{equation}
where $Q,Q_0,Q_1$ are mutually singular probability laws.
Among randomized tests using two independent batches of $s$
i.i.d.\ observations, the smallest worst-case failure probability
for deciding whether their generating graph identities agree is
\begin{equation}
\label{eq:hierarchy-optimal-error}
    \alpha_s^*(\eta)
    =\frac12\left[1-\bigl(1-(1-\eta)^s\bigr)^2\right].
\end{equation}
The worst case ranges over all four ordered choices from
$\{P_0,P_1\}$, indexed by their graph identities.
\end{proposition}

\begin{proof}
Mutual singularity makes the three alternatives $Q,Q_0,Q_1$
observable. A batch reveals its identity if it contains an
observation from $Q_0$ or $Q_1$, an event of probability
$a_s=1-(1-\eta)^s$. If both batches reveal their identities,
compare them and answer correctly. Otherwise, flip an independent
fair coin. This test has failure probability $(1-a_s^2)/2$
in each of the four input cases.

For the lower bound, put the uniform prior on the four ordered
pairs of graph identities. If a batch contains only observations
from $Q$, its identity remains an independent uniform bit
conditional on both observed batches. Unless both identities
are revealed, their equality therefore has conditional probability
$1/2$. Every test has average failure probability at least
$(1-a_s^2)/2$ under this prior, and its worst-case failure
probability is at least its average. This proves optimality.
\end{proof}

Set $\delta\coloneq1-1/\sqrt3$. The proposition gives
\begin{equation}
\label{eq:hierarchy-threshold}
\begin{aligned}
    \alpha_s^*(\eta)\leq\frac13
    &\quad\Longleftrightarrow\quad(1-\eta)^s\leq\delta\\
    &\quad\Longleftrightarrow\quad\eta\geq1-\delta^{1/s}.
\end{aligned}
\end{equation}
For the constructed pair, use $\eta=\eta_{\ell,k,q}$ when
$q\geq2$. For $q=1$, the trial laws agree, so the optimal
failure probability is $1/2$, also given by the formula at
$\eta=0$.

\subsection{Strict hierarchies}

We are now ready to prove the two hierarchy results. We begin by proving the hierarchy results on the number of trials.  

\begin{proof}[Proof of \cref{thm:strict-hierarchy}]
Fix $s\geq2$, and set
\begin{equation*}
    L_s\coloneq1-\delta^{1/s},
    \qquad U_s\coloneq1-\delta^{1/(s-1)}.
\end{equation*}
Then $0<L_s<U_s<1$. Choose a prime $\ell\geq7$ sufficiently
large that, with $p=p_{\ell,2}=2^{2-2\ell}$,
\begin{equation*}
    p\leq1-\frac{L_s}{U_s},
    \qquad (1-p)^2\geq U_s.
\end{equation*}
Let $k$ be the smallest positive integer such that
$(1-p)^k<U_s$. The second inequality ensures $k\geq3$.
Minimality gives
\begin{equation*}
    L_s\leq(1-p)U_s\leq(1-p)^k<U_s.
\end{equation*}
Thus $L_s\leq\eta_{\ell,k,2}<U_s$. By
\cref{eq:hierarchy-threshold}, $s$ trials per graph achieve
failure probability at most $1/3$, whereas every test with
$s-1$ trials has worst-case failure probability greater than
$1/3$. Take $G_s=H_{0,\ell,k}$ and $H_s=H_{1,\ell,k}$.
By \cref{lem:noniso}, they are non-isomorphic, which gives the claim.
\end{proof}

We now give the hierarchy results for the number of colors.

\begin{proof}[Proof of \cref{thm:color-hierarchy}]

Fix $q,s\geq1$, and set
\begin{equation}
    \tau\coloneq1-\delta^{1/s},
    \qquad c\coloneq-\log\tau>0.
\end{equation}
For a prime $\ell\geq7$, put $a=2\ell-2$ and choose
\begin{equation}
    k\coloneq\max\{3,\lfloor c q^a\rfloor+1\}.
\end{equation}
When $q\geq2$, we have $kq^{-a}>c$, so
\begin{equation}
\begin{aligned}
    \eta_{\ell,k,q}
    &=(1-q^{-a})^k\\
    &\leq\exp(-kq^{-a})<\exp(-c)=\tau.
\end{aligned}
\end{equation}
For $q=1$, the same strict inequality holds because
$\eta_{\ell,k,1}=0$. Since $\eta_{\ell,k,r}$ is nondecreasing
in $r$, every palette size $r\leq q$ gives optimal failure
probability greater than $1/3$ with $s$ trials per graph.

On the other hand, $k\leq c q^a+3$, and Bernoulli's inequality
gives
\begin{equation}
\begin{aligned}
    \eta_{\ell,k,q+1}
    &\geq1-k(q+1)^{-a}\\
    &\geq1-c\left(\frac{q}{q+1}\right)^a-3(q+1)^{-a}.
\end{aligned}
\end{equation}
As $\ell$ tends to infinity through the primes, the last
expression tends to one. Choose $\ell$ sufficiently large that
it exceeds $\tau$. By \cref{eq:hierarchy-threshold}, $q+1$
colors and $s$ trials per graph then give failure probability
at most $1/3$. Taking $G_{q,s}=H_{0,\ell,k}$ and
$H_{q,s}=H_{1,\ell,k}$ proves the claim.
\end{proof}

\section{Representation Spaces for Message-Passing Networks}
\label{app:representation_spaces}

We fix an aggregation scheme
$\alpha\in\{\overline{\mathrm{sum}},\mathrm{mean},
\mathrm{max},\mathrm{min}\}$. All measures below are finite Borel measures. Graphs are finite and nonempty; for mean, max, and min aggregation, we assume that every vertex has a nonempty neighborhood.

\subsection{Metrics on measures and compact sets}

Let $(X,d)$ be a compact metric space.
Write $\mathcal{M}_{\le1}(X)$ for the nonnegative measures on $X$
with total mass at most one, and $\mathcal{P}(X)$ for the
probability measures on $X$.
For $\mu,\nu\in\mathcal{M}_{\le1}(X)$, define the bounded-Lipschitz
metric
\begin{equation}
\beta_d(\mu,\nu)
\coloneq
\sup_{\substack{\|f\|_\infty\le1\\\operatorname{Lip}_d(f)\le1}}
\left|\int_X f\,d\mu-\int_X f\,d\nu\right|.
\label{eq:bl_metric}
\end{equation}
The bound on $\|f\|_\infty$ makes this expression finite even when
the measures have different total masses.

Let $\mathcal{K}(X)$ denote the collection of nonempty compact
subsets of $X$. Its Hausdorff metric is
\begin{equation}
\operatorname{H}_d(A,B)
\coloneq
\max\left\{
\sup_{a\in A}\inf_{b\in B}d(a,b),
\sup_{b\in B}\inf_{a\in A}d(a,b)
\right\}.
\label{eq:hausdorff_metric}
\end{equation}

\subsection{Iterated representation spaces}

Set $\mathcal{X}_0=K$ and
$\mathsf{d}_0(\vec{x},\vec{y})=\|\vec{x}-\vec{y}\|_2$.
Given $(\mathcal{X}_\ell,\mathsf{d}_\ell)$, define the neighborhood
space
\[
\mathcal{Y}_\ell
=
\begin{cases}
\mathcal{M}_{\le1}(\mathcal{X}_\ell),
& \alpha=\overline{\mathrm{sum}},\\
\mathcal{P}(\mathcal{X}_\ell),
& \alpha=\mathrm{mean},\\
\mathcal{K}(\mathcal{X}_\ell),
& \alpha\in\{\mathrm{max},\mathrm{min}\},
\end{cases}
\]
equipped with
\[
\mathsf{e}_\ell
=
\begin{cases}
\beta_{\mathsf{d}_\ell},
& \alpha\in\{\overline{\mathrm{sum}},\mathrm{mean}\},\\
\operatorname{H}_{\mathsf{d}_\ell},
& \alpha\in\{\mathrm{max},\mathrm{min}\}.
\end{cases}
\]
Define
\begin{equation}
\mathcal{X}_{\ell+1}
=
\mathcal{X}_\ell\times\mathcal{Y}_\ell,
\qquad
\mathsf{d}_{\ell+1}\bigl((z,\nu),(z',\nu')\bigr)
=
\mathsf{d}_\ell(z,z')+\mathsf{e}_\ell(\nu,\nu').
\label{eq:recursive_representation_metric}
\end{equation}

For normalized-sum and mean aggregation, these are spaces of iterated degree measures: each representation records the initial feature together with measures on lower-depth representations. For max and min aggregation, we replace measures with compact neighborhood sets.

\subsection{Induced representations and pseudometrics}
\label{app:induced_representations}

For a graph $G$ with initialization $\vec{x}\in K^{V(G)}$, set
$\eta_{G,v}^{(0)}(\vec{x})=\vec{x}_v$ and recursively define
\begin{equation*}
\eta_{G,v}^{(\ell+1)}(\vec{x})
=
\bigl(\eta_{G,v}^{(\ell)}(\vec{x}),\nu_{G,v}^{(\ell)}(\vec{x})\bigr),
\end{equation*}
where
\[
\nu_{G,v}^{(\ell)}(\vec{x})
=
\begin{cases}
\displaystyle
\frac{1}{|V(G)|}
\sum_{u\in N_G(v)}
\delta_{\eta_{G,u}^{(\ell)}(\vec{x})},
& \alpha=\overline{\mathrm{sum}},\\[6pt]
\displaystyle
\frac{1}{|N_G(v)|}
\sum_{u\in N_G(v)}
\delta_{\eta_{G,u}^{(\ell)}(\vec{x})},
& \alpha=\mathrm{mean},\\[6pt]
\{\eta_{G,u}^{(\ell)}(\vec{x}):u\in N_G(v)\},
& \alpha\in\{\mathrm{max},\mathrm{min}\}.
\end{cases}
\]
Here $\delta_z$ is the Dirac probability measure at $z$. The normalized-sum measure has mass $|N_G(v)|/|V(G)|\le1$, the mean measure has mass one, and the neighborhood sets are finite and nonempty. Thus
$\eta_{G,v}^{(\ell)}(\vec{x})\in\mathcal{X}_\ell$.

The induced pseudometric on rooted attributed graphs is
\begin{equation*}
d_D^{\mathrm{node}}\bigl((G,\vec{x},v),(G',\vec{x}',v')\bigr)
\coloneq
\mathsf{d}_D\bigl(
\eta_{G,v}^{(D)}(\vec{x}),\eta_{G',v'}^{(D)}(\vec{x}')
\bigr).
\end{equation*}
It is a pseudometric because distinct rooted attributed graphs
may have identical representations.

For graph-level representations, define
\[
R_{G,\vec{x}}^{(D)}
\coloneq
\frac1{|V(G)|}\sum_{v\in V(G)}
\delta_{\eta_{G,v}^{(D)}(\vec{x})}
\in\mathcal{P}(\mathcal{X}_D).
\]
The induced graph-level pseudometric is
\begin{equation*}
d_D^{\mathrm{graph}}\bigl((G,\vec{x}),(G',\vec{x}')\bigr)
\coloneq
W_1\bigl(R_{G,\vec{x}}^{(D)},R_{G',\vec{x}'}^{(D)}\bigr),
\end{equation*}
where the Wasserstein distance $W_1$ uses the ground metric $\mathsf{d}_D$.

The induced representation maps are continuous, hence Borel measurable, by induction on depth using the continuity of the finite measure and neighborhood-set constructions. So their
laws under random initialization are well defined.

\section{Compactness of the Representation Spaces and Lipschitz Property}
\label{app:representation_compactness_lipschitzness}
This section collects the compactness and Lipschitz properties of the deterministic representation framework used in our analysis. We recall the relevant results and refer to \citet{wittig2026which} for the underlying arguments. We develop the extension to random initialization in the next appendix.

\subsection{Compactness of Representation Spaces}
\label{app:subsec_compactness}
Compactness of the representation spaces follows inductively from compactness of the initial feature space $K$; each successive space is a product of the preceding space and a space of neighborhood
measures or compact sets. The following lemma provides the compactness properties needed for this construction.

\begin{lemma}
\label{lem:measure_hyperspace_compactness}
If $(X,d)$ is a compact metric space, then
$\mathcal{M}_{\le1}(X)$ and $\mathcal{P}(X)$ are compact under
$\beta_d$, and $\mathcal{K}(X)$ is compact under
$\operatorname{H}_d$.
\end{lemma}

\begin{proof}
For complete proofs, see \Citep[Theorems~96 and~98]{wittig2026which}. For the spaces of measures, compactness follows from the compactness and metrizability of the weak$^*$ topology together with the fact that the Kantorovich metric $\beta_d$ metrizes this topology; see \cite{Kechris1995}, Theorem~17.22, and \cite{Bogachev2007}, Theorem~8.3.2. For $\mathcal{K}(X)$, compactness with respect to the Hausdorff metric $\operatorname{H}_d$ follows from the Blaschke selection theorem; see \Citep[Theorem~1.8.4]{Schneider1993}.
\end{proof}

\begin{corollary}
\label{cor:representation_compactness}
For every $\ell\ge0$, $(\mathcal{X}_\ell,\mathsf{d}_\ell)$ is a
compact metric space.
\end{corollary}

\subsection{Lipschitzness of MPNNs}
\label{app:subsec_lipschitzness}

Set $\phi_\theta^{(0)}(\vec{z})=\vec{z}$.
For $(z,\nu)\in\mathcal{X}_{\ell+1}$, define
\[
\phi_\theta^{(\ell+1)}(z,\nu)
=
\sigma\left(
\vec{W}_{\mathrm{self}}^{(\ell)}\phi_\theta^{(\ell)}(z)
+
\vec{W}_{\mathrm{agg}}^{(\ell)}
\mathcal{A}_\theta^{(\ell)}(\nu)
+\vec b^{(\ell)}
\right),
\]
where
\[
\mathcal{A}_\theta^{(\ell)}(\nu)
=
\begin{cases}
\displaystyle\int\phi_\theta^{(\ell)}(y)\,d\nu(y),
& \alpha\in\{\overline{\mathrm{sum}},\mathrm{mean}\},\\[5pt]
\displaystyle\max_{y\in\nu}\phi_\theta^{(\ell)}(y),
& \alpha=\mathrm{max},\\[5pt]
\displaystyle\min_{y\in\nu}\phi_\theta^{(\ell)}(y),
& \alpha=\mathrm{min}.
\end{cases}
\]
The maximum and minimum are taken coordinate-wise.

The map $\phi_\theta^{(\ell)}$ realizes the MPNN on the structural representation space. The following identity, obtained by induction on $\ell$, shows that the structural representation contains all
information required to compute the node output. Consequently, Lipschitz bounds for $\phi_\theta^{(\ell)}$ transfer directly to the MPNN under the induced pseudometric.

\begin{lemma}
\label{lem:representation_realization}
For every initialization $\vec{x}$, vertex $v$, and depth $\ell$,
\[
\phi_\theta^{(\ell)}\bigl(\eta_{G,v}^{(\ell)}(\vec{x})\bigr)
=
\vec{h}_{\theta,G}^{(\ell)}(v;\vec{x}).
\]
\end{lemma}

\begin{proof}
We prove the result by induction on $\ell$. For $\ell=0$, we have
$$
\phi_\theta^{(0)} \bigl(\eta_{G,v}^{(0)}(\vec{x})\bigr)= \phi_\theta^{(0)}(\vec{x}_v) = \vec{x}_v= \vec{h}_{\theta,G}^{(0)}(v;\vec{x}).
$$
Now suppose that $\phi_\theta^{(\ell)} \bigl(\eta_{G,u}^{(\ell)}(\vec{x})\bigr)= \vec{h}_{\theta,G}^{(\ell)}(u;\vec{x})$, for every vertex $u\in V(G)$. For $\ell+1$, we write
$\eta_{G,v}^{(\ell+1)}(\vec{x}) = \bigl( \eta_{G,v}^{(\ell)}(\vec{x}), \nu_{G,v}^{(\ell)}(\vec{x})\bigr).$ Then, by definition of $\phi_\theta^{(\ell+1)}$,
\begin{align*}
\phi_\theta^{(\ell+1)}
\bigl(\eta_{G,v}^{(\ell+1)}(\vec{x})\bigr)
&=
\sigma\Bigl(
\vec{W}_{\mathrm{self}}^{(\ell)}
\phi_\theta^{(\ell)}
\bigl(\eta_{G,v}^{(\ell)}(\vec{x})\bigr)
+
\vec{W}_{\mathrm{agg}}^{(\ell)}
\mathcal A_\theta^{(\ell)}
\bigl(\nu_{G,v}^{(\ell)}(\vec{x})\bigr)
\Bigr).
\end{align*}
By the induction hypothesis, the first term inside the activation is
$$
\phi_\theta^{(\ell)}
\bigl(\eta_{G,v}^{(\ell)}(\vec{x})\bigr)
=
\vec{h}_{\theta,G}^{(\ell)}(v;\vec{x}).
$$
It remains to identify the aggregation term. By construction of
$\nu_{G,v}^{(\ell)}(\vec{x})$, applying
$\mathcal A_\theta^{(\ell)}$ to this neighborhood representation
recovers the corresponding aggregation of the depth-$\ell$ hidden
states. Indeed, for the sum and mean cases, integration against the
induced atomic measure gives the required weighted sum of
$$
\phi_\theta^{(\ell)}
\bigl(\eta_{G,u}^{(\ell)}(\vec{x})\bigr),
\qquad u\in N_G(v),
$$
which, by the induction hypothesis, equals the corresponding sum or
mean of
$\vec{h}_{\theta,G}^{(\ell)}(u;\vec{x})$. For $\alpha=\mathrm{max}$,
\begin{align*}
\mathcal A_\theta^{(\ell)}
\bigl(\nu_{G,v}^{(\ell)}(\vec{x})\bigr)
&=
\max_{y\in\nu_{G,v}^{(\ell)}(\vec{x})}
\phi_\theta^{(\ell)}(y)\
&=
\max_{u\in N_G(v)}
\vec{h}_{\theta,G}^{(\ell)}(u;\vec{x}),
\end{align*}
and the same argument gives, for $\alpha=\mathrm{min}$,
$$
\mathcal A_\theta^{(\ell)}\bigl(\nu_{G,v}^{(\ell)}(\vec{x})\bigr)=\min_{u\in N_G(v)}\vec{h}_{\theta,G}^{(\ell)}(u;\vec{x}),
$$
with maxima and minima taken coordinate-wise. Thus, in every case, $\mathcal A_\theta^{(\ell)}
(\nu_{G,v}^{(\ell)}(\vec{x}))$ is exactly the neighborhood aggregate appearing in the definition of the message-passing update. Hence
$$
\phi_\theta^{(\ell+1)} \bigl(\eta_{G,v}^{(\ell+1)}(\vec{x})\bigr) =\vec{h}_{\theta,G}^{(\ell+1)}(v;\vec{x}).
$$
This completes the induction.
\end{proof}

\begin{lemma}
\label{lem:aggregation_stability}
Let $g\colon X\to\mathbb{R}^p$ satisfy
$\sup_x\|g(x)\|_2\le M$ and
$\|g(x)-g(y)\|_2\le C d(x,y)$.
For $\mu,\nu\in\mathcal{M}_{\le1}(X)$,
\[
\left\|\int g\,d\mu-\int g\,d\nu\right\|_2
\le\max\{M,C\}\beta_d(\mu,\nu),
\qquad
\left\|\int g\,d\mu\right\|_2\le M.
\]
For $A,B\in\mathcal{K}(X)$,
\[
\|\max_A g-\max_B g\|_2
\le\sqrt{p}\,C\operatorname{H}_d(A,B),
\qquad
\|\max_A g\|_2\le\sqrt{p}\,M.
\]
The same bounds hold for minimum aggregation.
\end{lemma}

\begin{proof}
For each unit vector $\vec{a}\in\mathbb{R}^p$, the scalar function
$\langle \vec{a},g\rangle$ has uniform norm at most $M$ and Lipschitz
constant at most $C$.
After scaling by $\max\{M,C\}$, it is an admissible test function
in \cref{eq:bl_metric}; the zero case is immediate.
Taking the supremum over unit vectors proves the first bound.
The second follows from $\mu(X)\le1$.

For each coordinate $j$, Hausdorff proximity gives
\[
|\max_A g_j-\max_B g_j|
\le C\operatorname{H}_d(A,B).
\]
Summing the squared coordinate bounds yields the factor
$\sqrt{p}$. Since $|g_j|\le M$, each coordinate of $\max_A g$
has absolute value at most $M$, proving its norm bound.
Replacing $g$ by $-g$ proves the claims for minimum aggregation.
\end{proof}

Combining these aggregation bounds with the $1$-Lipschitz property of ReLU yields recursive bounds on the magnitude and Lipschitz constant of each layer. We denote these bounds by $M_{\theta,\ell}$ and $C_{\theta,\ell}$, respectively; the resulting recursions provide an explicit Lipschitz certificate for the final readout. Write
\[
a_\ell=\|\vec{W}_{\mathrm{self}}^{(\ell)}\|_2,
\qquad
b_\ell=\|\vec{W}_{\mathrm{agg}}^{(\ell)}\|_2,
\qquad
R_K=\max_{\vec{x}\in K}\|\vec{x}\|_2.
\]
Initialize $M_{\theta,0}=R_K$ and $C_{\theta,0}=1$.
The bias contributes to the magnitude bound but cancels in differences of updates.
For normalized-sum and mean aggregation, define
\begin{align*}
M_{\theta,\ell+1}
&=(a_\ell+b_\ell)M_{\theta,\ell}+\|\vec b^{(\ell)}\|_2,\\
C_{\theta,\ell+1}
&=\max\left\{
a_\ell C_{\theta,\ell},
b_\ell\max\{M_{\theta,\ell},C_{\theta,\ell}\}
\right\}.
\end{align*}
For max and min aggregation, define
\begin{align*}
M_{\theta,\ell+1}
&=(a_\ell+b_\ell\sqrt{d_\ell})M_{\theta,\ell}+\|\vec b^{(\ell)}\|_2,\\
C_{\theta,\ell+1}
&=\max\{a_\ell,b_\ell\sqrt{d_\ell}\}C_{\theta,\ell}.
\end{align*}

\begin{proposition}[Boundedness and Lipschitz continuity]
\label{prop:mpnn_certificates}
For every $\ell$,
\[
\sup_{z\in\mathcal{X}_\ell}
\|\phi_\theta^{(\ell)}(z)\|_2\le M_{\theta,\ell},
\qquad
\|\phi_\theta^{(\ell)}(z)-\phi_\theta^{(\ell)}(z')\|_2
\le C_{\theta,\ell}\mathsf{d}_\ell(z,z').
\]
Consequently, $\psi_\theta=\vec{W}_{\mathrm{out}}\phi_\theta^{(D)}$
has Lipschitz certificate $B_\theta\coloneq \|\vec{W}_{\mathrm{out}}\|_2C_{\theta,D}$.
\end{proposition}

\begin{proof}
The statement trivially holds at depth zero. Suppose it holds at depth $\ell$. ReLU is $1$-Lipschitz and satisfies $\|\sigma(\vec{y})\|_2\le\|\vec{y}\|_2$. For $s=(z,\nu)$ and $s'=(z',\nu')$,
\begin{align*}
\|\phi_\theta^{(\ell+1)}(s)
-\phi_\theta^{(\ell+1)}(s')\|_2
&\le
a_\ell C_{\theta,\ell}\mathsf{d}_\ell(z,z')\\
&\quad+
b_\ell\|\mathcal{A}_\theta^{(\ell)}(\nu)
-\mathcal{A}_\theta^{(\ell)}(\nu')\|_2.
\end{align*}
Apply Lemma \ref{lem:aggregation_stability} to the second term.
The resulting coefficients of
$\mathsf{d}_\ell(z,z')$ and $\mathsf{e}_\ell(\nu,\nu')$
are bounded by $C_{\theta,\ell+1}$.
Their sum is therefore at most
$C_{\theta,\ell+1}\mathsf{d}_{\ell+1}(s,s')$.
The norm estimates in the same lemma give
$M_{\theta,\ell+1}$.
This completes the induction.
Composition with the linear readout gives the final assertions.
\end{proof}

Since $a_\ell,b_\ell\le B_W$,
$\|\vec b^{(\ell)}\|_2\le B_b$,
$\|\vec{W}_{\mathrm{out}}\|_2\le B_{\mathrm{out}}$, and the architecture
is fixed, these recursions yield uniform finite bounds
\begin{equation}
\sup_{\theta\in\Theta}B_\theta\le B,
\qquad
\sup_{\theta\in\Theta}\|\psi_\theta\|_\infty\le M,
\label{eq:uniform_certificate_bounds}
\end{equation}
where we may choose $M\ge1$.
The constants do not depend on graph order.

\section{Randomly Initialized Representations}
\label{app:randomized_representations}
We now extend the compactness and Lipschitz properties to representation laws induced by random node initialization. The resulting domain is totally bounded under the Wasserstein distance and therefore admits finite covers at every positive scale. Taking expectations of the MPNN outputs yields predictors
on this domain that retain the deterministic Lipschitz certificates established above.

\subsection{Compactness of the probability-measure domain}

\begin{lemma}
\label{lem:wasserstein_compactness}
If $(X,d)$ is compact, then $(\mathcal{P}(X),W_1)$ is compact.
\end{lemma}

\begin{proof}
We assume $X\ne\varnothing$, since otherwise the claim is
immediate. By \cref{lem:measure_hyperspace_compactness}, every
sequence in $\mathcal{P}(X)$ has a weakly convergent subsequence.
It remains to show that weak convergence implies convergence
in $W_1$. Fix $x_0\in X$ and let
\[
\mathcal{T}_0
=
\{f\colon X\to\mathbb{R}:
\operatorname{Lip}_d(f)\le1,\ f(x_0)=0\}.
\]
The functions in $\mathcal{T}_0$ are uniformly bounded by
$\operatorname{diam}(X)$ and equicontinuous. Since $\mathcal{T}_0$
is also closed under uniform convergence, it is compact in the
uniform norm by Arzelà--Ascoli
\cite[Theorem~7.25]{rudin1976principles}. Now suppose $\mu_n$ converges weakly to $\mu$. Given $a>0$,
choose a finite $a$-net $\{f_1,\ldots,f_N\}$ of $\mathcal{T}_0$.
Kantorovich--Rubinstein duality then gives
\[
W_1(\mu_n,\mu)
=
\sup_{f\in\mathcal{T}_0}
\left|\int_X f\,d\mu_n-\int_X f\,d\mu\right|
\le
\max_{1\le j\le N}
\left|\int_X f_j\,d\mu_n-\int_X f_j\,d\mu\right|+2a.
\]
The maximum tends to zero by weak convergence, and $a$ can be
chosen arbitrarily small. Hence $W_1(\mu_n,\mu)\to0$.
Every sequence therefore has a $W_1$-convergent subsequence,
which proves compactness.
\end{proof}

Applying this lemma to $\mathcal{X}_D$ shows that
$\mathcal{P}(\mathcal{X}_D)$ is compact.
Its subset $\mathcal{Q}_D$ need not be compact, but it is totally bounded (i.e., admits a finite $\varepsilon$-cover for every $\varepsilon>0$).

\subsection{Expected outputs and Lipschitzness}

For $\vec{\xi}^G\sim\rho_K^{\otimes V(G)}$, define
\[
Q_{G,v}^{(D)}
=
\operatorname{Law}(\eta_{G,v}^{(D)}(\vec{\xi}^G)),
\qquad
Q_G^{(D)}
=
\frac1{|V(G)|}\sum_v Q_{G,v}^{(D)}.
\]
These laws are well defined by the measurability established in
\cref{app:induced_representations}.
Moreover, $Q_G^{(D)}$ is the mean of the random measure
$R_{G,\vec{\xi}^G}^{(D)}$, in the sense that
\[
\int f\,dQ_G^{(D)}
=
\mathbb{E}\left[\int f\,dR_{G,\vec{\xi}^G}^{(D)}\right]
\]
for every bounded measurable $f$.

By \cref{lem:representation_realization},
\[
F_\theta(G,\vec{x})=\int\psi_\theta\,dR_{G,\vec{x}}^{(D)}.
\]
The boundedness of $\psi_\theta$ permits exchanging expectation
and the finite vertex sum, giving
\[
\overline{F}_\theta(G) = \mathbb{E}[F_\theta(G,\vec{\xi}^G)] =
\int\psi_\theta\,dQ_G^{(D)} = \mathcal{F}_\theta(Q_G^{(D)}).
\]

For $Q,Q'\in\mathcal{P}(\mathcal{X}_D)$ and any coupling
$\pi\in\Pi(Q,Q')$,
\begin{align*}
|\mathcal{F}_\theta(Q)-\mathcal{F}_\theta(Q')|
&=
\left|\int(\psi_\theta(z)-\psi_\theta(z'))\,d\pi(z,z')\right|\\
&\le
B_\theta\int\mathsf{d}_D(z,z')\,d\pi(z,z').
\end{align*}

Thus, averaging preserves the certificate $B_\theta$, with Lipschitz continuity now measured with respect to $W_1$ on probability measures and absolute value on scalar outputs. The same argument gives
\[
|F_\theta(G,\vec{x})-F_\theta(G',\vec{x}')|
\le
B_\theta d_D^{\mathrm{graph}}\bigl((G,\vec{x}),(G',\vec{x}')\bigr).
\]

On graphs without specified initializations, define
\[
d_D^{\mathrm{rand}}(G,G')
\coloneq
W_1(Q_G^{(D)},Q_{G'}^{(D)}).
\]
This is a pseudometric, and
\[
|\overline{F}_\theta(G)-\overline{F}_\theta(G')|
\le B_\theta d_D^{\mathrm{rand}}(G,G').
\]
Relabeling a graph permutes its node representations and preserves
the i.i.d.\ initialization law. Hence $Q_G^{(D)}$ and
$\overline{F}_\theta(G)$ are invariant under vertex relabeling.

\section{Uniform Monte Carlo Approximation}
\label{app:monte_carlo}

Fix training graphs $G_1,\ldots,G_n$ and write
$Q_i=Q_{G_i}^{(D)}$.
For each graph, the $S$ initializations are independent and have
law $\rho_K^{\otimes V(G_i)}$.
The following argument is conditional on the training graphs and
does not require them to form a cover. Dependence between node
representations within an initialization is allowed.

Let $\Psi=\{\psi_\theta\colon\theta\in\Theta\}$.
By \cref{eq:uniform_certificate_bounds},
\[
\sup_{\theta\in\Theta}\|\psi_\theta\|_\infty\le M,
\qquad
\sup_{\theta\in\Theta}\operatorname{Lip}_{\mathsf d_D}
(\psi_\theta)\le B,
\]
where $M\ge1$ and $B<\infty$ depend only on the fixed architecture,
feature space, and weight and bias bounds.
Thus, $\Psi$ is uniformly bounded and equicontinuous on the compact
space $\mathcal X_D$. Arzelà--Ascoli (\cite[Theorem~7.25]{rudin1976principles}) implies total boundedness in
the uniform norm. For $a>0$, let
\[
N_\Psi(a)
\coloneq
\mathcal N(\Psi,\|\cdot\|_\infty,a),
\]
with covering centers chosen in $\Psi$.
The class $\Psi$ is separable in the uniform norm, and the sampling
deviation is continuous in this norm, so its supremum over $\Psi$
is measurable.

\begin{proposition}[Uniform sampling bound]
\label{prop:uniform_monte_carlo}
For every $a,t>0$,
\[
\mathbb P\left(
\sup_{\theta\in\Theta}\max_{1\le i\le n}
|\widehat F_{\theta,S}(G_i)-\overline F_\theta(G_i)|
>2a+t
\right)
\le
2nN_\Psi(a)\exp\left(-\frac{St^2}{2M^2}\right).
\]
\end{proposition}

\begin{proof}
Choose an $a$-net $\psi_1,\ldots,\psi_N$ of $\Psi$, with
$N=N_\Psi(a)$, and define
\[
Y_{i,s,j}
=
\frac1{|V(G_i)|}\sum_{v\in V(G_i)}
\psi_j\bigl(\eta_{G_i,v}^{(D)}(\vec{\xi}_{i,s}^{G_i})\bigr).
\]
For fixed $i,j$, these variables are independent over $s$,
belong to $[-M,M]$, and satisfy
\[
\mathbb E[Y_{i,s,j}]=\int\psi_j\,dQ_i.
\]
Hoeffding's inequality gives
\[
\mathbb P\left(
\left|\frac1S\sum_{s=1}^S Y_{i,s,j}
-\int\psi_j\,dQ_i\right|>t
\right)
\le
2\exp\left(-\frac{St^2}{2M^2}\right).
\]
A union bound makes this estimate simultaneous over all $i,j$,
with failure probability at most
$2nN_\Psi(a)\exp(-St^2/(2M^2))$.

For any $\theta$, choose $j$ such that
$\|\psi_\theta-\psi_j\|_\infty\le a$.
Replacing $\psi_\theta$ by $\psi_j$ changes its sampled graph-level
average by at most $a$ and its expectation by at most $a$.
Hence, on the simultaneous event,
\[
|\widehat F_{\theta,S}(G_i)-\overline F_\theta(G_i)|
\le 2a+t
\]
for every $\theta$ and $i$.
\end{proof}

Taking $a=\varepsilon_{\mathrm{MC}}/4$ and
$t=\varepsilon_{\mathrm{MC}}/2$, a sufficient sample size is
\begin{equation}
S\ge
\frac{8M^2}{\varepsilon_{\mathrm{MC}}^2}
\log\left(
\frac{2nN_\Psi(\varepsilon_{\mathrm{MC}}/4)}{\delta}
\right).
\label{eq:mc_sample_complexity}
\end{equation}
With probability at least $1-\delta$, this ensures
\[
\sup_{\theta\in\Theta}\max_{1\le i\le n}
|\widehat F_{\theta,S}(G_i)-\overline F_\theta(G_i)|
\le\varepsilon_{\mathrm{MC}}.
\]
Since $\bigl||a-y|-|b-y|\bigr|\le|a-b|$, averaging over the
training graphs yields
\begin{equation}
\mathbb P\left(
\sup_{\theta\in\Theta}
|\widehat{\mathcal L}_{\mathrm{emp}}(\theta)
-\mathcal L_{\mathrm{emp}}(\theta)|
\le\varepsilon_{\mathrm{MC}}
\right)\ge1-\delta.
\label{eq:uniform_mc_main}
\end{equation}
The same bound holds for the regularized losses because their
regularizers coincide.

\section{Proof of Learnability Guarantee}
\label{app:learning_proofs}
Below we prove \Cref{thm:randomized_generalization} from the main paper. For fixed training graphs $G_1,\ldots,G_n$, we recall the definitions
\[
\mathcal L_{\mathrm{emp}}(\theta)
\coloneq
\frac1n\sum_{i=1}^n
|\overline F_\theta(G_i)-f^\ast(G_i)|,
\qquad
\mathcal L_{\mathrm{reg}}(\theta)
\coloneq
\mathcal L_{\mathrm{emp}}(\theta)
+\lambda(B_\theta-B_\ast)_+.
\]
We also write
\[
\widehat{\mathcal L}_{\mathrm{emp}}(\theta)
\coloneq
\frac1n\sum_{i=1}^n
|\widehat F_{\theta,S}(G_i)-f^\ast(G_i)|.
\]
Note that $\widehat{\mathcal L}_{\mathrm{reg}}$ and $\mathcal L_{\mathrm{reg}}$ use the same regularizer since the Lipschitz certificate only depends on the parameters of the MPNN and not on the initialization of node features.

\begin{proof}
Fix $\varepsilon>0$ and $\delta\in(0,1)$, and set
\[
r=\frac{\varepsilon}{2(2B_\ast+1)}.
\]
By total boundedness of $(\mathcal G,d_D^{\mathrm{rand}})$, choose a finite $r$-net $\{G_1,\ldots,G_n\}\subseteq\mathcal G$. For every $G\in\mathcal G$, there is an index $i$ such that $d_D^{\mathrm{rand}}(G,G_i)\le r$. The triangle inequality
and the Lipschitz bounds give
\begin{align*}
|\overline F_\theta(G)-f^\ast(G)|
&\le
|\overline F_\theta(G)-\overline F_\theta(G_i)|
+|\overline F_\theta(G_i)-f^\ast(G_i)|
+|f^\ast(G_i)-f^\ast(G)|\\
&\le
(B_\theta+B_\ast)r+n\mathcal L_{\mathrm{emp}}(\theta).
\end{align*}

Choose
\[
\eta=\min\left\{\lambda,\frac{\varepsilon}{4n}\right\},
\qquad
\varepsilon_{\mathrm{MC}}=\frac{\eta}{2}.
\]
By the uniform Monte Carlo bound proved in \cref{app:monte_carlo}, every $S$ satisfying
\cref{eq:mc_sample_complexity} gives an event of probability at least $1-\delta$ on which
\[
\sup_{\theta\in\Theta}
|\widehat{\mathcal L}_{\mathrm{emp}}(\theta)
-\mathcal L_{\mathrm{emp}}(\theta)|
\le\frac{\eta}{2}.
\]
On this event, any $\theta$ with $\widehat{\mathcal L}_{\mathrm{reg}}(\theta)<\eta$
satisfies
\[
\mathcal L_{\mathrm{emp}}(\theta)
\le
\widehat{\mathcal L}_{\mathrm{emp}}(\theta)+\frac{\eta}{2}
<
\frac{3\eta}{2},
\qquad
B_\theta\le B_\ast+\frac{\eta}{\lambda}.
\]
Consequently,
\[
\sup_{G\in\mathcal G}
|\overline F_\theta(G)-f^\ast(G)|
\le
\frac{3n\eta}{2}
+\left(2B_\ast+\frac{\eta}{\lambda}\right)r
\le
\frac{3\varepsilon}{8}+\frac{\varepsilon}{2}
<\varepsilon.
\]
\end{proof}

\clearpage
\section{Experimental details}
\label{app:experimental-details}

\paragraph{Computing infrastructure.}
Experiments ran on a shared server using three NVIDIA RTX A6000 GPUs
(48\,GB each) and two AMD EPYC 7452 32-core CPUs. We used PyTorch 2.9.1
with CUDA 12.8, running each training fit on a single GPU and independent
fits in parallel.

\paragraph{Asset licenses.}
The \href{https://github.com/ralphabb/GNN-RNI/blob/43c69c06462d3a165ba2b532d76cbc045fc767d7/LICENSE}{GNN-RNI repository}
providing EXP includes a GNU GPLv3 license, with no separate dataset
license stated. The pinned
\href{https://github.com/chrsmrrs/k-gnn/tree/cde1f577fce0f3abb9ee9e5b0f8bc4ab43956e2e}{k-gnn}
and \href{https://github.com/joisino/random-features/tree/180e926edd601e10d4b115a0a1722fa7183f7f13}{random-features (TRIANGLE)}
releases contain no explicit license statement. Our reproduction scripts
fetch third-party code and data directly from these repositories.

\subsection{Bridge prediction}
\label{app:er-bridges}

\paragraph{Data and targets.}
We sample graphs from $G(n,1.5/(n-1))$, allowing disconnected graphs. Training and validation each
contain equal numbers of graphs at $n=16,32$; independent test splits
use $n\in\{32,128,512\}$. For each target $b\in\{4,5,6\}$, an edge
is positive exactly when it is a bridge whose removal leaves a smaller
endpoint component of size $b$; every other edge is negative, including
bridges of other sizes.  We retain every sampled graph and
extend splits to obtain at least 1,000 training positives and 200 per
evaluation split (\cref{tab:bridge-data-coverage}). All palettes and
training seeds share the data for their target.

\begin{table}[t]
\centering
\small
\begin{tabular}{@{}lrrr@{}}
\toprule
Split & $b=4$ & $b=5$ & $b=6$ \\
\midrule
Train ($n=16,32$) & 2,944 / 1,516 & 2,944 / 1,001 & 4,608 / 1,004 \\
Validation ($n=16,32$) & 640 / 367 & 640 / 208 & 1,056 / 205 \\
Test ($n=32$) & 448 / 305 & 448 / 200 & 672 / 207 \\
Test ($n=128$) & 128 / 372 & 128 / 223 & 160 / 220 \\
Test ($n=512$) & 64 / 705 & 64 / 447 & 64 / 324 \\
\bottomrule
\end{tabular}
\caption{Bridge datasets: graph count / positive-edge count.}
\label{tab:bridge-data-coverage}
\end{table}

\paragraph{Model and training.}
Each trial assigns independent uniform colors from $q\in\{2,4,16,32\}$,
encoded as one-hot vectors padded to 32 coordinates for every palette.
The six-layer, width-32 ReLU MPNN uses max aggregation and a symmetric
two-layer edge head, with 22,305 trainable parameters in every condition.
Training averages 16 sigmoid-probability trials before class-balanced
binary cross-entropy over pooled edges. We sample graphs uniformly with
replacement in batches of eight and use Adam at $10^{-3}$ without weight
decay for 100,000 updates, without early stopping. We select checkpoints
by lowest validation balanced error. Each target and palette uses five
training seeds, varying initialization, minibatches, and training colors.

\paragraph{Evaluation and reporting.}
Test inference averages 16 fresh trials and thresholds at $1/2$;
validation and test color streams are fixed and shared across fits.
FNR pools positive edges across test graphs, while FPR pools all negative
edges, including wrong-size bridges. \Cref{fig:er-bridge-generalization}
reports balanced accuracy, $100[1-(\mathrm{FNR}+\mathrm{FPR})/2]$,
and \cref{tab:bridge-exact-side-errors} reports its complementary error.
Standard errors (SEs) equal the sample standard deviation across five
training seeds divided by $\sqrt{5}$, conditional on the fixed dataset.
\Cref{tab:bridge-neighbor-errors} exposes confusion hidden by pooled
balanced error: at $b=6,q=32,n=512$, FPRs on side-5 and side-7 bridges
are $70.5\%$ and $92.8\%$, and precision is $13.5\%$, despite $6.41\%$
balanced error. The learned detectors are motivated by the local MST
construction but do not implement its algorithm.

\begin{table}[t]
\centering
\small
\setlength{\tabcolsep}{4pt}
\caption{Exact-side bridge detection: balanced test error (\%), mean $\pm$ SE across five training seeds.}
\label{tab:bridge-exact-side-errors}
\begin{tabular}{rrccc}
\toprule
$b$ & Colors & $n=32$ & $n=128$ & $n=512$ \\
\midrule
4 & 2 & $3.77 \pm 0.25$ & $2.97 \pm 0.16$ & $2.47 \pm 0.08$ \\
4 & 4 & $1.54 \pm 0.11$ & $1.09 \pm 0.04$ & $1.14 \pm 0.05$ \\
4 & 16 & $0.47 \pm 0.12$ & $0.28 \pm 0.08$ & $0.25 \pm 0.02$ \\
4 & 32 & $0.28 \pm 0.10$ & $0.43 \pm 0.07$ & $0.20 \pm 0.02$ \\
\midrule
5 & 2 & $8.00 \pm 0.71$ & $9.40 \pm 0.66$ & $8.64 \pm 0.89$ \\
5 & 4 & $4.32 \pm 0.15$ & $3.97 \pm 0.15$ & $2.62 \pm 0.07$ \\
5 & 16 & $2.97 \pm 0.61$ & $2.57 \pm 0.62$ & $2.04 \pm 0.32$ \\
5 & 32 & $2.54 \pm 0.27$ & $2.19 \pm 0.52$ & $1.40 \pm 0.41$ \\
\midrule
6 & 2 & $7.53 \pm 0.54$ & $8.45 \pm 0.86$ & $9.01 \pm 0.88$ \\
6 & 4 & $6.08 \pm 0.14$ & $7.48 \pm 0.54$ & $8.65 \pm 0.89$ \\
6 & 16 & $6.48 \pm 0.49$ & $7.33 \pm 0.99$ & $7.67 \pm 1.63$ \\
6 & 32 & $5.67 \pm 0.62$ & $6.22 \pm 1.04$ & $6.41 \pm 1.21$ \\
\bottomrule
\end{tabular}
\end{table}

\begin{table}[t]
\centering
\small
\setlength{\tabcolsep}{5pt}
\caption{Neighboring-side false-positive rates and overall precision (\%) at $n=512$, averaged over five training seeds. Bridges of sizes $b-1$ and $b+1$ are negatives for target $b$; precision is the fraction of predicted positives with the correct label.}
\label{tab:bridge-neighbor-errors}
\begin{tabular}{rrccc}
\toprule
$b$ & Colors & Side $b-1$ FPR & Side $b+1$ FPR & Precision \\
\midrule
4 & 2 & 54.08 & 38.79 & 42.29 \\
4 & 4 & 27.96 & 10.87 & 61.67 \\
4 & 16 & 0.97 & 4.65 & 90.60 \\
4 & 32 & 0.39 & 1.79 & 93.18 \\
\midrule
5 & 2 & 87.38 & 93.02 & 11.91 \\
5 & 4 & 67.23 & 68.02 & 29.65 \\
5 & 16 & 23.15 & 43.40 & 43.61 \\
5 & 32 & 6.78 & 26.17 & 56.69 \\
\midrule
6 & 2 & 93.47 & 94.72 & 8.18 \\
6 & 4 & 87.83 & 98.18 & 8.75 \\
6 & 16 & 83.62 & 95.15 & 10.24 \\
6 & 32 & 70.47 & 92.81 & 13.50 \\
\bottomrule
\end{tabular}
\end{table}

\subsection{Tree substructure detection}
\label{app:tree-detection}

\paragraph{Data and exact labels.}
The fixed ten-vertex pattern has a root with three children and two
leaves per child. In a simple 3-regular graph, it embeds as a non-induced
subgraph rooted at $v$ exactly when the closed two-hop ball around $v$
has ten vertices; edges between the six leaves are allowed. We sample
connected random 3-regular graphs on 24 vertices, rejecting isomorphic
duplicates within and across splits. We use
256 training, 64 validation, and 1,000 test graphs, with 4,096, 1,024,
and 16,000 selected root vertices, respectively, balanced exactly by
label within each split. Graphs are split before roots are selected;
models receive full graphs, and all fits share the splits. Evaluation
tests unseen graphs of the same size.

\paragraph{Model and training.}
We use two width-32 residual ReLU message-passing layers with max
aggregation and a vertex MLP, without graph pooling. Independent uniform
one-hot colors are projected to the hidden space; hidden widths are
fixed, but the input projection grows with the palette. We train all
combinations of $q\in\{2,8,64\}$ colors and $s\in\{1,16,64\}$ trials
with five training seeds. The models have 5,345, 5,537, and 7,329 trainable
parameters for $q=2,8,64$, respectively, independent of trial count.
Training uses binary cross-entropy on each
root's maximum trial logit, fresh colors, and Adam at $10^{-3}$ without
weight decay. Each update samples 32 roots per class with replacement.
We allow up to 200,000 updates, stopping after 25,000 updates without
improved validation max-trial cross-entropy, assessed every 250 updates.
We select the checkpoint minimizing this loss.

\paragraph{Inference and reporting.}
Let $z_v^{(r)}$ be the logit at vertex $v$ in trial $r$, and let
$\sigma$ denote the sigmoid function. At inference, we sum trial
probabilities and predict
\[
    a_v=\sum_{r=1}^{s}\sigma(z_v^{(r)}),
    \qquad \widehat y_v=\mathbf{1}\{a_v>\tau\}.
\]
Thus, a vertex is predicted to have the maximum two-hop neighborhood
size of ten exactly when its summed score exceeds the threshold $\tau$;
otherwise it is predicted negative. We choose $\tau$ separately for each
trained model to minimize validation classification error and hold it
fixed on the test set. Evaluation color streams are fixed across training seeds. \Cref{fig:tree-detection-heatmap}
reports mean node accuracy $\pm$ SE across five seeds on the fixed,
balanced test set; \cref{tab:tree-detection-seeds} retains individual
errors, including the variation at $(q,s)=(8,64)$. These learned
detectors do not implement the exact color-coding dynamic program or
require colorful witnesses.

\begin{table}[t]
\centering
\small
\setlength{\tabcolsep}{3pt}
\begin{tabular}{@{}rrrrrrr@{}}
\toprule
$q$ & $s$ & Seed 0 & Seed 1 & Seed 2 & Seed 3 & Seed 4 \\
\midrule
2 & 1 & 46.45 & 46.52 & 46.44 & 46.52 & 46.52 \\
2 & 16 & 38.17 & 39.46 & 41.93 & 41.02 & 42.89 \\
2 & 64 & 36.00 & 36.77 & 37.01 & 36.94 & 36.19 \\
\addlinespace
8 & 1 & 36.86 & 37.64 & 37.10 & 37.89 & 37.44 \\
8 & 16 & 10.71 & 14.26 & 11.60 & 12.11 & 11.02 \\
8 & 64 & 2.96 & 25.66 & 1.30 & 2.31 & 15.18 \\
\addlinespace
64 & 1 & 30.64 & 28.21 & 31.44 & 28.51 & 28.04 \\
64 & 16 & 10.01 & 7.89 & 3.30 & 18.70 & 8.47 \\
64 & 64 & 0.09 & 0.86 & 0.32 & 0.48 & 1.41 \\
\bottomrule
\end{tabular}
\caption{Tree detection test error (\%) for each training seed, with
$q$ colors and $s$ trials. All models use the same 16,000 labeled roots
on 1,000 unseen graphs.}
\label{tab:tree-detection-seeds}
\end{table}

\subsection{Inference-time quantization}
\label{app:exp-inference-quantization}

\paragraph{EXP data and model.}
We use the released EXP dataset~\cite{abboud2021surprising}:
1,200 graphs on 32--64 vertices, in 600 oppositely labeled,
$1$-WL-indistinguishable pairs. The released ten-fold splits contain
972 training, 108 validation, and 120 test graphs per fold, with balanced
labels and pairs kept together. This is held-out classification, not
size extrapolation. Our GCN-RNI reimplementation uses eight width-64
mean-neighbor layers, separate root transforms, ELU, max graph pooling,
and an MLP classifier. Inputs are 64 independent
$\operatorname{Unif}[-1,1)$ coordinates, or one followed by 63 zeros,
keeping the architecture fixed at 72,354 trainable parameters. Original
vertex labels are not inputs.

\paragraph{EXP training.}
We use Adam at $5\cdot10^{-4}$, batches of 20, negative log-likelihood,
and 500 epochs, with one fresh feature draw per prediction and final
checkpoints. We follow the paper's ELU rather than the released script's
tanh default, and use its Gaussian-input learning rate because no
separate uniform-input rate is reported.

\paragraph{TRIANGLE data and model.}
The released TRIANGLE data~\cite{sato2021random} contain 1,000 20-vertex
3-regular training graphs and 1,000 test graphs at each of 20 and 100 vertices.
Labels indicate triangle membership, with original class frequencies
retained. Following rGIN, we use four width-64 layers with sum aggregation
including the root, MLP updates, batch normalization, ReLU, and a vertex
classifier, totaling 30,930 trainable parameters. We replace its 100-value
random scalar with a fresh independent
$\operatorname{Unif}[0,1)$ scalar per vertex; no other vertex features
are supplied.

\paragraph{TRIANGLE training.}
We train one model for 17,500 Adam updates (350 epochs), with batches
of 32 graphs and learning rate $0.01$, halved every 50 epochs. We retain
the upstream cross-entropy on softmax outputs, weighted inversely by
class frequency. Training uses one trial per prediction; evaluation
uses the final checkpoint without test-based selection.

\paragraph{Frozen-model evaluation.}
Continuous and quantized conditions share feature draws and test graphs,
with weights and batch-normalization statistics fixed. We average metrics
from single-trial predictions, not predictions across trials. EXP quantizes
each active coordinate to $-1/2$ or $1/2$ according to its sign; each fold
averages argmax classification accuracy over 50 draws. We report mean
accuracy and descriptive SE across ten fold means, whose training sets
overlap. TRIANGLE uses $Q_q(u)=(\lfloor qu\rfloor+1/2)/q$ and reports
mean one-versus-rest ROC-AUC over both classes. Its mean and SE use ten
paired feature draws of one trained model, not multiple training seeds.
Both SEs are sample SD divided by $\sqrt{10}$.
The exploratory palette sweep in \cref{fig:triangle-inference-quantization}
shows close agreement with continuous inputs for $q\geq4$, with
nonmonotonic behavior at smaller palettes.

\begin{figure}[t]
\centering
\includegraphics[width=\linewidth]{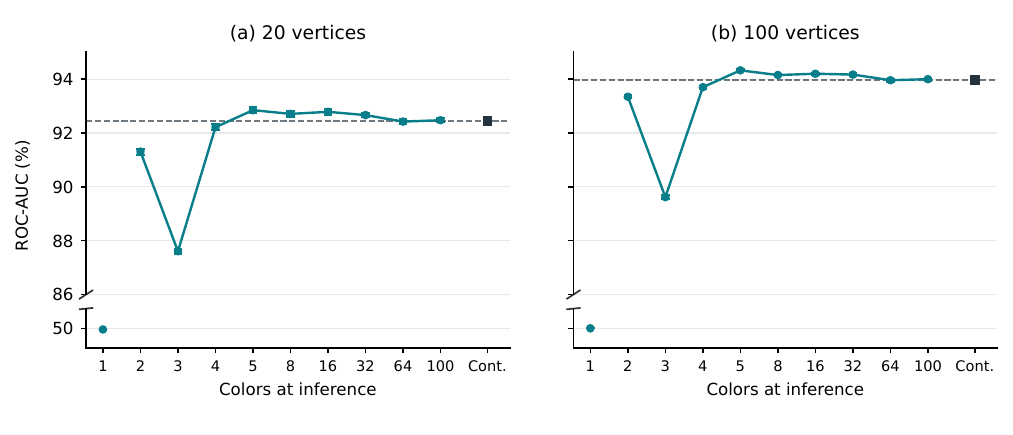}
\caption{TRIANGLE inference-time quantization of one frozen rGIN trained
with continuous features on 20-vertex cubic graphs. Test sets contain
1,000 graphs at each size. Points and error bars show mean ROC-AUC
$\pm$ SE across ten paired single-trial feature draws; squares
and dashed lines mark the continuous baseline. The axis break omits
51--86\% ROC-AUC. These error bars do not measure training-seed variability.}
\label{fig:triangle-inference-quantization}
\end{figure}

\clearpage
\tableofcontents

\end{document}